\documentclass[11pt]{article}

\usepackage[margin=1in]{geometry}

\usepackage[utf8]{inputenc}
\usepackage[T1]{fontenc}
\usepackage{lmodern}
\usepackage{microtype}

\usepackage{amsmath,amssymb,amsthm,bm,mathtools}

\DeclareMathOperator{\logit}{logit}

\usepackage{graphicx}
\usepackage{booktabs}
\usepackage{longtable}
\usepackage{subcaption}
\usepackage{algorithm}
\usepackage{algpseudocode}
\usepackage{enumitem}

\usepackage{tikz}
\usetikzlibrary{arrows.meta,positioning,calc,matrix,backgrounds}

\usepackage{xcolor}
\usepackage[numbers]{natbib}

\usepackage[
    colorlinks=true,
    linkcolor=blue,
    citecolor=blue,
    urlcolor=blue
]{hyperref}

\newtheorem{theorem}{Theorem}
\newtheorem{proposition}{Proposition}
\newtheorem{lemma}{Lemma}

\newtheorem{assumption}{Assumption}

\theoremstyle{definition}

\theoremstyle{remark}

\newcommand{\R}{\mathbb{R}}
\newcommand{\E}{\mathbb{E}}
\newcommand{\1}{\mathbf{1}}
\newcommand{\cN}{\mathcal{N}}
\newcommand{\cD}{\mathcal{D}}

\newcommand{\yt}{y_t,\dots,y_m}
\newcommand{\smin}{\operatorname{smin}}

\newcommand{\hardmcmc}{\textsc{Hard-MCMC}}
\newcommand{\relaxedmcmc}{\textsc{Relaxed-MCMC}}
\newcommand{\fullrankvi}{\textsc{FullRank-VI}}
\newcommand{\flowvi}{\textsc{Flow-VI}}

\title{A Differentiable Bayesian Relaxation for Latent Partial-Order Inference}

\author{
Dongqing Li\\
University of Oxford\\
\texttt{dongqing.li@kellogg.ox.ac.uk}
\and
Geoff K. Nicholls\\
University of Oxford\\
\texttt{geoffrey.nicholls@stats.ox.ac.uk}
\and
Shiyi Sun\\
University of Oxford\\
\texttt{shiyi.sun@spc.ox.ac.uk}
\and
You Luo\\
University College London\\
\texttt{you.luo.25@ucl.ac.uk}
}
\begin{document}
\maketitle
\begin{abstract}
Many ranking and agent trace datasets are recorded as linear orders even though their latent structure is only partially ordered. This is especially common in agent and workflow traces, where observed order may reflect arbitrary linearization rather than true prerequisites. We introduce a differentiable relaxation for latent partial-order inference from such traces. Starting from a hard frontier-constrained model of noisy linear extensions, we replace discontinuous product-order precedence and binary frontier feasibility with smooth surrogates, yielding a continuous posterior that preserves closure-level partial-order semantics and supports gradient-based MCMC and variational inference. We prove soft transitivity, sharp-limit frontier recovery, and convergence to the hard likelihood. Experiments on synthetic data, records of social dominance relations, and cloud-agent traces show close posterior fidelity to hard MCMC on small instances and improved runtime--accuracy trade-offs on larger problems.
\end{abstract}

% ============================================================================
% Suggested NeurIPS main-paper page budget:
% Intro 1.0 / Notation 0.3 / Hard model 1.3 / Relaxation 2.0 / Bayesian 1.4 /
% Theory 1.0 / Experiments 1.5 / Discussion+Conclusion 0.5.
% Move detailed proofs and posterior derivations to the appendix.
% ============================================================================

\section{Introduction}
\label{sec:intro}
Rank-data and action-trace datasets are typically recorded as linear sequences,
although the constraints governing valid outcomes are often only partially
ordered. These constraints may be temporal or process-based \citep{mannila00,mannila08,leemans2023partial}, causal
\citep{beerenwinkel2007conjunctive}, or dominance-based
\citep{nicholls2025royalacta}, and are usually not observed directly. Inferring
them is important because they encode interpretable structure and support
feasibility evaluation on new sequences. In these settings, however, the
underlying relation is often incomplete: the latent structure is a partial order,
or \emph{poset}, in which pairs of items that can occur in either order have no
precedence relation. Consequently, an observed order need not imply a true
prerequisite relation; it may reflect scheduling, logging, or a single valid
linearization of the latent partial order. Treating all observed precedences as
real can therefore produce overly sequential and unrealistic structures,
especially in workflow or LLM-agent settings where unnecessary ordering induces
extra execution steps and compute. A latent partial order separates stable
prerequisites from reorderable actions.

% In recent work \cite{li2026delinearizingagenttracesbayesian} has shown that posets are helpful for workflow and agent trace analysis: an outcome in which \(a\) comes before \(b\) may reflect scheduling, logging, or simply one valid linearization among many, rather than a true prerequisite relation. However, \(a\) before \(b\) may be a logical requirement of the task and so \cite{li2026delinearizingagenttracesbayesian} learn the constraining poset and pass it to the LLM, saving tokens. If traces are linearisations of an incomplete order then treating their precedence relations as real will reconstruct unrealistic and unnecessarily sequential execution structures. A latent partial order separates stable prerequisite constraints from interchangeable actions.

Bayesian inference is natural because sparse or noisy traces support multiple compatible partial orders. A posterior over precedence relations allows us to distinguish high-confidence prerequisites from unresolved relations. %\citep{jiang2023bayesian}. 
The main obstacle is computational: current MCMC-based Bayesian Partial Order inference or ``hard-PO'' \citep{nicholls11} is combinatorial, and counting linear extensions of a poset is \#P-complete \citep{brightwell1991linearextensions}. These difficulties become acute as traces grow longer: inference with hard-PO has the right semantics but does not scale; its discontinuous precedence constraints prevent use of efficient gradient-based methods.

We address this bottleneck with a \emph{differentiable Bayesian relaxation for latent partial-order inference from noisy linear extensions}. Our key contribution is a \emph{differentiable posterior over closure-level partial-order structure} that preserves \emph{frontier semantics} while replacing discrete poset exploration with gradient-based inference. This is most useful when observations are linearizations of a latent workflow and non-observed pair reversals should be treated as uncertainty rather than evidence of precedence. Building on the hard frontier-softmax likelihood \citep{li2026delinearizingagenttracesbayesian}, we introduce a product-order relaxation by smoothing precedence, frontier feasibility, and successor utility, with provable soft transitivity, sharp-limit frontier recovery, and convergence to the hard likelihood; Figure~\ref{fig:method_overview} summarizes the pipeline.

\begin{figure}[h]
    \centering
    \includegraphics[width=0.89\textwidth, trim=0mm 0mm 0mm 0mm, clip]{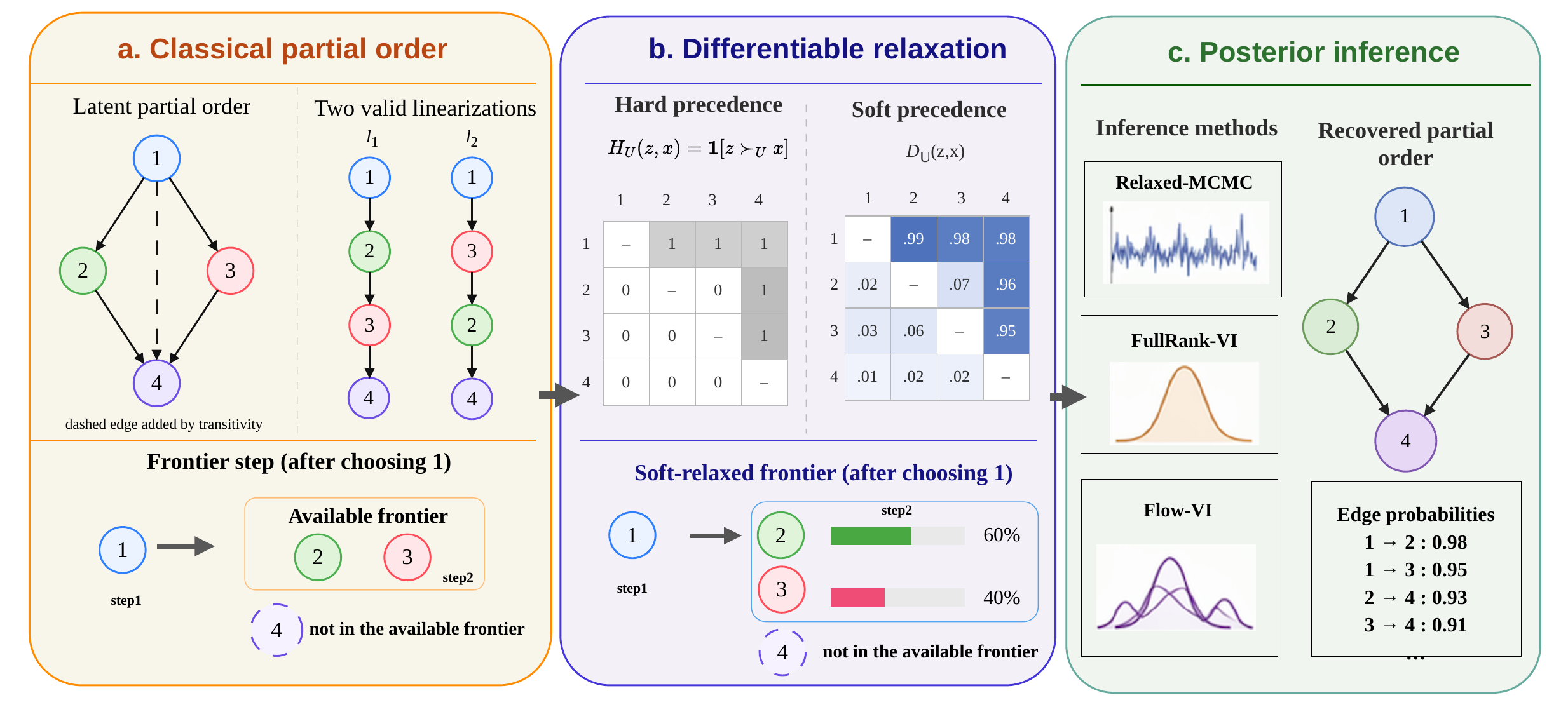}
\caption{
Overview of the proposed framework.
(a) A latent partial order induces multiple valid linearizations; after selecting item \(1\), items \(2\) and \(3\) lie on the hard frontier, while item \(4\) is not yet feasible.
(b) The hard precedence matrix \(H_U\) is replaced by a differentiable soft precedence matrix \(D_U\), which induces a soft-relaxed frontier and the relaxed frontier-softmax likelihood.
(c) The resulting continuous posterior supports \relaxedmcmc{}, \fullrankvi{}, and \flowvi{}, yielding posterior edge probabilities and a recovered partial order.
}

\label{fig:method_overview}
\end{figure}

% We address this challenge with a differentiable relaxation of the model in \cite{li2026delinearizingagenttracesbayesian}. The relaxation preserves poset-constrained sequential choice, replacing hard product-order precedence and binary frontier membership with smooth surrogates. This yields a continuous target for \relaxedmcmc and variational inference, and recovers hard-PO as a limit. Figure~\ref{fig:method_overview} gives the pipeline: data are linear extensions of a partial order with a hard closure representation, which is relaxed to a differentiable soft precedence matrix for inference.
% The key design principle is to smooth the inference problem without changing the object of inference. We therefore retain the sequential frontier likelihood, rather than replacing partial-order learning with an unconstrained differentiable loss. The frontier is the mechanism that connects observed traces to latent prerequisite structure. 

\noindent\textbf{Contributions.}
(i) We introduce a differentiable Bayesian relaxation for latent partial-order inference from noisy linear extensions.
(ii) We develop a smooth frontier-constrained likelihood that preserves hard-model semantics while enabling gradient MCMC and variational inference.
(iii) We prove sharp-limit recovery of the hard model and show strong posterior fidelity and runtime--accuracy trade-offs on synthetic, ranking, and agent-trace data.

\section{Notation and Preliminaries}
\label{sec:preliminaries}

We provide an overview of the related concepts and terminology in Figure~\ref{fig:po_old} in Appendix~\ref{app:preliminaries}.

\noindent\textbf{Partial orders, Choice sets and Observation lists:}
\label{subsec:choice_sets_posets} Let $\mathcal{M} = [M] := \{1,\dots,M\}$  denote the universe of items. The data will be sequences of items chosen from $\mathcal{M}$. Let \(\mathcal{B}_{\mathcal{M}}\) be the set of all nonempty subsets of \(\mathcal{M}\). For a generic \emph{choice set} $S\in\mathcal{B}_{\mathcal{M}}$ write $m = |S|$ for its size. A (strict) \emph{partial order} $h = (S,\succ_h)$ on \(S\) is a binary relation \(\succ_h\) satisfying
 \emph{asymmetry} (if $x_1\succ_h x_2$ for some $x_1,x_2\in S$ then not $x_2\succ_h x_1$) and \emph{transitivity} (if $x_1\succ_h x_2$ and $x_2\succ_h x_3$ then $x_1\succ_h x_3$). Relations may be incomplete so for $x,z\in S$ it is possible that neither $x\succ_h z$ nor $z\succ_h x$ hold. Let $\mathcal H_S$ be the class of all partial orders on items in $S$. 

If $h\in \mathcal H_S$ and for every $x\ne z$ in \(S\) either $x\succ_h z$ or $z\succ_h x$ then \(h\) is a \emph{complete order}. Complete orders are in one-to-one correspondence with permutations of \(S\). A sequence
$
y=(y_1,\dots,y_m)
$
respects $h$ if $y_i\succ_h y_j$ implies $i<j$, so higher precedence items come first. A \emph{linear extension} of \(h\) respects $h$ and contains each element of \(S\) exactly once. Figure~\ref{fig:method_overview}(a) illustrates a latent partial order in transitive closure, and two valid linear extensions.

\noindent\textbf{Closure, Reduction, and max-set:}
\label{subsec:closure_cover_frontier}
The \emph{transitive reduction}  $
\mathrm{TR}(h)
$ 
of a poset \(h\in \mathcal H_S\) drops all relations implied by transitivity. The \emph{transitive closure} 
$
\mathrm{TC}(h)
$
includes all relations. We recover closure-level relations; the reduction is used for its Hasse-diagram visualization. For any choice set \(S\) and partial order \(h=(S,\succ_h)\), define the maximal set
\[
\max(h)
=
\{x \in S : \nexists\ z \in S \text{ such that } z \succ_h x\}.
\]
When constructing a linear extension sequentially, the next item must be chosen from the max-set of the remaining suborder \(h\). The frontier at each choice step is this maximal set.

\subsection{Frontier-softmax likelihood}
\label{subsec:frontier_softmax_likelihood}
The data will be $N$ ranking lists of items. In our motivating application (Section~\ref{subsec:aliyun_experiment}) a generic list $y_{1:m}=(y_1,\dots,y_m)$ is a noisy sequential execution of a latent partial order. The basic modeling approach \citep{li2026delinearizingagenttracesbayesian} is that at each step the next item in the sequence is chosen from the current \emph{frontier}, i.e., from the maximal elements of the remaining suborder.
Let \(h=(\mathcal{M},\succ_h)\) be a partial order on $\mathcal{M}$, let \(S\in\mathcal{B}_\mathcal{M}\) be a choice set, and let
$
y_{1:m}%=(y_1,\dots,y_m)
$
be a linear extension of the suborder $h[S]=(S,\succ_h)$. This includes all constraints on the order of items in $S$ imposed by $h$. 
At step \(t=1,\dots,m\) the remaining items are
$\
\yt=(y_t,\dots,y_m),
$
and \(h[\yt]=(\yt,\succ_h)\) is the suborder for unplaced items. The \emph{feasible hard frontier},
\begin{equation}
F(h;\yt)
:=
\max(h[\yt]),
%=\left\{j\in y_{\scriptscriptstyle \,\ge\, i}:\nexists j'\in y_{\scriptscriptstyle \,\ge\, i}\setminus\{j\}\text{ such that } j'\succ_h j\right\}.
\label{eq:frontier_definition}
\end{equation}
is just the max-set of $h[\yt]$. These are the ``next'' items available to grow the sequence. 

In order to weight selection over feasible items, we assign each \(x \in \yt\) a successor count
\begin{equation}\textstyle
S_t(x;h)
=\sum_{z\in \yt}\mathbf{1}[x \succ_h z],
\label{eq:frontier_softmax_descendants}
\end{equation}
which counts the remaining items that must follow \(x\). This encodes a ``hardest-path-first'' heuristic: items that dominate larger suborders are preferred.
Following \citet{li2026delinearizingagenttracesbayesian}, we define the frontier-softmax utility, where \(F(h;\yt)\) is the feasible frontier defined in \eqref{eq:frontier_definition}.
\begin{equation}
Q_t(x;h)=\log\!\bigl(1+S_t(x;h)\bigr)\quad\text{for }x\in F(h;y_{\ge t}),
\label{eq:frontier_softmax_utility}
\end{equation}

Given inverse temperature \(\beta\ge 0\), the corresponding noise-free selection probability is
\begin{equation}
q(x \mid h[\yt],\beta)
\propto
\exp\{\beta Q_t(x;h)\}, \qquad x\in F(h;\yt)
\label{eq:frontier_softmax_choice}
\end{equation}
As \(\beta\) increases, the distribution concentrates more strongly on high-utility frontier elements. The likelihood for $h\in\mathcal{H}_S$ given one order \(y_{1:m}\) is
\begin{equation}\textstyle
p(y_{1:m}\mid h[S],\beta)
=
\prod_{t=1}^{m}
q(y_t \mid h[\yt],\beta)
\label{eq:frontier_softmax_full}.
\end{equation}
which defines a generative ranking model for linear extensions. It is a sequential random utility model, see \citet{seshadri21} for discussion of this class of model.

\section{Differentiable Relaxation of the Latent Partial Order}
\label{sec:relaxation}
The hard latent partial-order sequence-model $p(y_{1:m}\mid h,\beta)$ is semantically natural but computationally difficult. We therefore introduce a continuous latent variable parameterisation of partial orders and show how to smooth the model's dependence on the new parameters. 

\textbf{Latent product-order parameterisation.} \label{subsec:latent_product_order}Give each item \(x\in \mathcal{M}\) a latent embedding $u_x\in\R^d$ and let \(U=(u_x)_{x\in \mathcal{M}}\) so that $U$ is an $\R^{M\times d}$ embedding matrix. 
The hard latent precedence relation $h_U=(\mathcal{M},\succ_U),\ h_U\in\mathcal{H}_{\mathcal{M}}$ parameterised by $U$ is the product order
\begin{equation}
z \succ_U x
\iff
\mathbf 1[u_{z,k}>u_{x,k}]=1
\quad\forall k=1,\dots,d.
\label{eq:hard_product_order}
\end{equation}
The dimension \(d\) controls how expressive the embedding is (Appendix~\ref{app:poset_dimension}): when \(d=1\), the induced order is complete; if \(d>1\), pairs can be incomparable (for example \(u_{z,1}>u_{x,1}\) but \(u_{x,2}>u_{z,2}\)). See \citep{nicholls2025royalacta} for discussion crediting \cite{winkler1985random}.

When we parameterise $h$ via $h_U$ the likelihood $p(y_{1:m}\mid h_U,\beta)$ depends on $U$ through $h$. However, $U$ is processed through the non-differentiable precedence indicator in \eqref{eq:hard_product_order}, so we replace precedence with a smooth product-order score, lift it to differentiable surrogates for feasibility \eqref{eq:frontier_definition} and utility \eqref{eq:frontier_softmax_choice} and use these to define a relaxed frontier-softmax likelihood.

\textbf{Soft pairwise precedence.}
%Recall that the hard pairwise margin $m_U(z,x)$ defined below \eqref{eq:hard_product_order} is the worst coordinate-wise gap, and its sign determines the hard precedence relation in \eqref{eq:hard_product_order_mU}.
If for \(z,x \in \mathcal M\) and $k=1,\dots,d$ we define coordinate-wise gaps
$
\Delta_k(z,x) = u_{z,k} - u_{x,k}
$ and let 
\begin{equation}
m_U(z,x)=\min_{k=1,\dots,d}\Delta_k(z,x)\label{eq:hard_min_mU}    
\end{equation}
(the \emph{pairwise margin}) then \eqref{eq:hard_product_order} can be written
\begin{equation}
z \succ_U x
\iff
\mathbf 1[m_U(z,x)>0]=1.\label{eq:hard_product_order_mU}
\end{equation}

Motivated by product-order embedding models for transitive asymmetric structure
\citep{vendrov2016order}, we replace the discontinuous minimum \(m_U\) in
\eqref{eq:hard_product_order_mU} by the standard log-sum-exp soft minimum
\citep{nesterov2005smooth} and the indicator by a sigmoid. For \(\tau>0\), let
\(\smin_\tau\) be defined by \eqref{eq:softmin_definition}, and define the
relaxed pairwise margin \(M_U\) by \eqref{eq:soft_margin}.
\begin{equation}
\smin_\tau(a_1,\dots,a_d)
=
-\tau\log\!\left(\sum_{k=1}^d e^{-a_k/\tau}\right).
\label{eq:softmin_definition}
\end{equation}
\begin{equation}
M_U(z,x)
=
\smin_\tau\!\bigl(\Delta_1(z,x),\dots,\Delta_d(z,x)\bigr),
\label{eq:soft_margin}
\end{equation}
\(M_U(z,x)\) is a smooth approximation to \(m_U(z,x)\) (See Lemma~\ref{lem:margin_approx_appendix}). Let \(\gamma>0\)  be a sharpness parameter and \(\sigma(t)=(1+e^{-t})^{-1}\) the sigmoid function. The soft precedence score
\begin{equation}
D_U(z,x)
=
\sigma\!\bigl(\gamma M_U(z,x)\bigr),
\label{eq:soft_pairwise_precedence}
\end{equation}
is a differentiable surrogate for the hard indicator \(\1[z\succ_U x]\). The diamond partial order in Figure~\ref{fig:method_overview}a can be represented either by the binary closure matrix in Figure~\ref{fig:method_overview}b (at left) or by the soft precedence matrix in Figure~\ref{fig:method_overview}b (at right). The relaxed matrix replaces the binary precedence event by a score in \([0,1]\) while preserving the same latent geometric structure.

\textbf{Soft frontier.}
The pairwise score \(D_U(z,x)\) must be lifted to a set-level frontier weight.
For the remaining set \(\yt\), the hard frontier indicator and its relaxation are
\begin{equation}
\begin{aligned}
F_t(x;h)
&:=
\mathbf 1[x\in F(h;\yt)]
=
\prod\nolimits_{z\in \yt\setminus\{x\}}
\bigl(1-\mathbf 1[z\succ_h x]\bigr),\\
\widetilde F_t(x;U)
&:=
\prod\nolimits_{z\in \yt\setminus\{x\}}
\bigl(1-D_U(z,x)\bigr).
\end{aligned}
\label{eq:frontier_indicator_and_relaxation}
\end{equation}
The product form says that \(x\) is feasible iff no remaining item precedes it in $h$.
The relaxation replaces each hard precedence indicator by the soft score
\(D_U(z,x)\), so \(\widetilde F_t(x;U)\) is small when some remaining item
strongly precedes \(x\), and close to one when \(x\) is frontier-feasible.

\textbf{Soft successor utility and relaxed likelihood.}
The descendant count in \eqref{eq:frontier_softmax_descendants} and successor utility in \eqref{eq:frontier_softmax_utility} are relaxed to give
\begin{equation}
\widetilde S_t(x;U)
=
\sum_{z\in \yt\setminus\{x\}} D_U(x,z),
\qquad
\widetilde Q_t(x;U)
=
\log\bigl(1+\widetilde S_t(x;U)\bigr)
\label{eq:soft_successor_terms}
\end{equation}
respectively.
Substituting \(\widetilde F_t\) and \(\widetilde Q_t\) into \eqref{eq:frontier_softmax_choice} yields the relaxed one-step likelihood
\begin{equation}
\widetilde p(y_t\mid \yt,U,\beta,\gamma)
\propto
%\frac
{
\widetilde F_t(x;U)\exp\{\beta \, \widetilde Q_t(x;U)\}
},\qquad x\in \yt 
%{
%\sum_{x\in \yt}
%\widetilde F_t(x;U)\exp\{\beta \, \widetilde Q_t(x;U)\}
%}
\label{eq:relaxed_step_likelihood}
\end{equation}
and the relaxed full likelihood is
\begin{equation}\textstyle
\widetilde p(y_{1:m}\mid U,\beta,\gamma)
=
\prod_{t=1}^{m}
\widetilde p_{\yt}(y_t\mid \yt,U,\beta,\gamma).
\label{eq:relaxed_trace_likelihood}
\end{equation}

Figure~\ref{fig:frontier_likelihood} in Appendix~\ref{sec:relaxed_po_details}
contrasts the hard and relaxed frontier likelihoods. Both place mass on frontier-feasible choices, but they differ off the frontier: the hard
likelihood assigns zero probability, whereas the relaxed likelihood remains
small but nonzero via \(\widetilde F_t\). This yields a differentiable
objective while retaining a strong bias toward feasible actions.
Figure~\ref{fig:soft_product_order}  visualizes the overall relaxation, from hard
product-order precedence to a soft precedence matrix \(D_U\) used in the
relaxed frontier and utility terms, making the likelihood smooth in \(U\).

\subsection{Structural properties of the relaxation}
\label{subsec:relaxation_theory}
We now state the main structural guarantees of the relaxation: uniform approximation of the hard margin, soft transitivity, frontier recovery, and sharp-limit convergence back to the hard frontier-softmax model. We begin with uniform approximation of the hard margin.
\begin{proposition}[Soft-min approximation]
\label{prop:softmin_properties}
For any \(a=(a_1,\dots,a_d)\in\R^d\) and any \(\tau>0\),
\begin{equation}
\min_k a_k-\tau\log d
\;\le\;
\smin_\tau(a_1,\dots,a_d)
\;\le\;
\min_k a_k,
\label{eq:softmin_uniform_bounds}
\end{equation}
and \(\smin_\tau(a)\to \min_k a_k\) as \(\tau\downarrow 0\).
\end{proposition}The relaxed model exhibits a \emph{soft-transitivity} which converges to hard transitivity as $\tau\downarrow 0$.
\begin{theorem}[Soft transitivity]
\label{thm:soft_transitivity}
For any \(x,y,z\in S\),
\begin{equation}
M_U(z,x)\ge M_U(z,y)+M_U(y,x).
\label{eq:soft_transitivity_margin_main}
\end{equation}
Consequently,
\begin{equation}
\logit D_U(z,x)\ge \logit D_U(z,y)+\logit D_U(y,x).
\label{eq:soft_transitivity_logit_main}
\end{equation}
\end{theorem}
Figure~\ref{fig:hard_soft_transitivity} gives the geometric intuition behind soft transitivity.
The hard product order is transitive because the product-order embedding in \eqref{eq:hard_product_order} composes: if for $k=1,\dots,d$ we have \(u_{z,k}>u_{y,k}\) and \(u_{y,k}>u_{x,k}\) then \(u_{z,k}>u_{x,k}\). The relaxed score
inherits this geometry by replacing the hard minimum with a soft-min margin and sigmoid
precedence score.
\begin{figure}[h]
    \centering
    \includegraphics[width=0.7\textwidth]{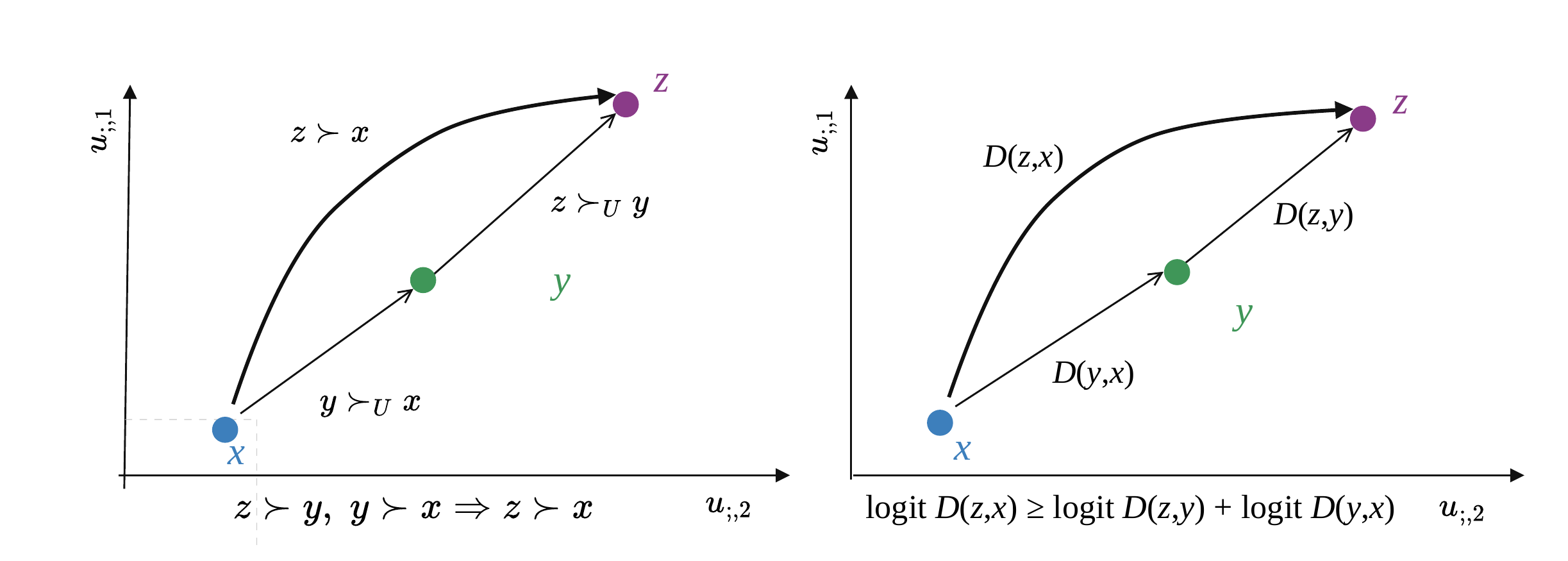}
\caption{
Hard and soft transitivity under product-order precedence.
(a) Hard product-order transitivity;
(b) Relaxed transitivity: strong \(D_U(z,y)\) and \(D_U(y,x)\) induce strong \(D_U(z,x)\)}
    \label{fig:hard_soft_transitivity}
\end{figure}
\begin{assumption}[Margin separation]
\label{ass:margin_separation}
The margin \(m_U(z,x)\!\neq\! 0\) in \eqref{eq:hard_min_mU} for all distinct \(x,z\in S\).
\end{assumption} 
This means the hard decision is off the boundary. A boundary case means $u_{z,k}=u_{x,k}$ for some $k=1,\dots d$, and since the prior for $U$ is continuous, Assumption~\ref{ass:margin_separation} holds almost surely.

\begin{theorem}[Sharp-limit frontier recovery]
\label{thm:frontier_quantitative}
Under Assumption~\ref{ass:margin_separation}, for all \(x\in \yt\),
\begin{equation}
\lim_{\gamma\to\infty}\lim_{\tau\downarrow 0}
\widetilde F_t(x;U)
=
\1\!\left[x\in F(h_U;\yt)\right],\qquad t=1,\dots,m.
\label{eq:frontier_limit_main}
\end{equation}
\end{theorem}

Theorem~\ref{thm:frontier_quantitative} supports the weak convergence of the data distributions $p$ and $\widetilde p$.
\begin{theorem}[Convergence of the relaxed frontier-softmax likelihood]
\label{thm:likelihood_convergence}
Under Assumption~\ref{ass:margin_separation}, for every step \(t\),
\begin{equation}
\lim_{\gamma\to\infty}\lim_{\tau\downarrow 0}
\widetilde p(y_t\mid \yt,U,\beta,\gamma)
=
p(y_t\mid \yt,h_U,\beta),
\label{eq:step_likelihood_limit_main}
\end{equation}
and consequently
\begin{equation}
\lim_{\gamma\to\infty}\lim_{\tau\downarrow 0}
\widetilde p(y\mid U,\beta,\gamma)
=
p(y\mid h_U,\beta).
\label{eq:trace_likelihood_limit_main}
\end{equation}
\end{theorem}

For fixed \(\tau>0\), \(\gamma>0\), \(\beta \ge 0\), the relaxed likelihood $\widetilde p(y_t\mid \yt,U,\beta,\gamma)$ is clearly \(C^\infty\) in \(U\) so in summary we have a continuously differentiable likelihood that converges to the likelihood we want to approximate as the hyperprameters $\gamma$ and $\tau$ go to zero.
%In Appendix \ref{app:proof_finite_error} we show not only that the relaxed model converges to the hard model, but that it is already close at finite $\tau$ and $\gamma$, with an explicit error bound.
However, for positive $\tau$ and $\gamma$, the relaxed posterior is still approximate: the experiments assess fidelity at $\gamma,\ \tau>0$.

\section{Bayesian Learning and Inference}
\label{sec:bayesian_learning}
We parameterise the partial order using coordinatewise dominance,
relaxing hard precedence and frontier constraints using the surrogates in
Section~\ref{sec:relaxation}. The posterior is based on the frontier-softmax likelihood
and that gives a continuous target for \relaxedmcmc{} and VI. 

\noindent\textbf{Prior and non-centred parameterization.}
\label{subsec:prior_geometry}
Following
\citep{nicholls2025royalacta}, the hard-PO partial order $h=h_U$ in \eqref{eq:hard_product_order} and our relaxation $D_U$ are given by the item embedding matrix $U=(u_x)_{x\in\mathcal{M}}$ of Section~\ref{sec:relaxation}. These have an exchangeable Gaussian prior,
\[
u_x \mid \rho \sim \mathcal N(0,\Sigma_\rho),\ x\in \mathcal{M},
\qquad
\Sigma_\rho=(1-\rho)I_d+\rho\,\mathbf 1\mathbf 1^\top,
\qquad
\rho\in(0,1),
\]
We write $U$ in term of standard normal variables
\(z_x\sim\mathcal N(0,I_d),\ x\in \mathcal{M}\), taking \(u_x=L_\rho z_x\) with
\(\Sigma_\rho=L_\rho L_\rho^\top\) so that \(U=ZL_\rho^\top\). 
Further priors are
\(\rho\sim\mathrm{Beta}(a_\rho,b_\rho)\),
\(\beta\sim\mathrm{Gamma}(a_\beta,b_\beta)\), and
\(\gamma\sim\mathrm{Gamma}(a_\gamma,b_\gamma)\). The soft-min
temperature \(\tau\) is fixed and checked by ablation. %Inference is carried out over \(Z\) and \((\rho,\beta,\gamma)\), with the likelihood in \eqref{eq:relaxed_trace_likelihood}.

%\iffalse
\citep{nicholls2025royalacta} show the family of marginal priors $\pi_S(h),\ S\in\mathcal{B}_{\mathcal{M}}$ is marginally consistent. This means that if $h\sim \pi_{\mathcal{M}}(\cdot)$ then the suborder $h[S]\sim \pi_S(\cdot)$. Here $D_U=[D_U(z,x)]_{z,x\in\mathcal{M}},\ D_U\in (0,1)^{M\times M}$ is the continuous surrogate for $h=(\mathcal{M},\succ_h)$ and the equivalent result holds. Let $\pi_S(D_U)$ be the marginal prior given by our $Z,\rho$ and $\gamma$ priors, constructed using $z_x,\ x\in S$ only.
If $D_U\sim \pi_\mathcal{M}(\cdot)$ then $D_U[S,S]$ is the matrix we get by deleting rows and columns of $D_U$ in $\mathcal{M}\setminus S$.
\begin{theorem}[Marginal Consistency, Appendix~\ref{app:marginal_consistency}]
\label{thm:marginal_consistency}
If $D_U\sim \pi_{\mathcal{M}}(\cdot)$ then $D_U[S,S]\sim \pi_S(\cdot)$.
\end{theorem}
%\fi

\noindent\textbf{Inference methods.}
\label{subsec:inference_methods}
In our data \(\cD=\{y^{(n)}\}_{n=1}^N\), \(y^{(n)}\) are lists ordering the elements of choice sets $S^{(n)}$ which are given and fixed. We compare
four inference routes.  \hardmcmc{} targets the non-differentiable reference posterior
\[\textstyle
p(Z,\rho,\beta|\cD)\propto p(Z,\rho,\beta)\,\prod_{n=1}^Np(y^{(n)}|h_U[S^{(n)}],\beta)
\] 
and uses the likelihood in \eqref{eq:frontier_softmax_full}. The continuous methods target the relaxed
posterior
%\[\textstyle
$p(Z,\rho,\beta,\gamma|\cD)$ with prior $p(Z,\rho,\beta,\gamma)$ and list-likelihood $\widetilde p(y^{(n)}|U,\beta,\gamma),\ n=1,\dots,N$ given in \eqref{eq:relaxed_trace_likelihood}. \relaxedmcmc{} uses NUTS \citep{hoffman2014nuts}, implemented in
Stan \citep{carpenter2017stan}; \fullrankvi{} follows automatic differentiation
variational inference \citep{kucukelbir2017advi}; and \flowvi{} uses a
normalizing-flow variational family \citep{rezende2015flows}. 
% Let \(\Theta=(Z,\rho,\beta,\gamma)\), and let \(w\) be the unconstrained parameter
% vector. Table~\ref{tab:inference_methods} summarizes the posterior target and
% approximation family of each method. For the continuous methods, the unconstrained target is
% \[
% \bar p(w\mid \mathcal D)\propto p(\mathcal D,T(w))|\det J_T(w)|.
% \]
Full derivations are given in Appendix~\ref{app:relaxed_posterior_target}.

\begin{table}[h]
\centering
\small
\caption{Inference routes used in the paper.}
\label{tab:inference_methods}
\begin{tabular}{lcccc}
\toprule
Method & \hardmcmc{} & \relaxedmcmc{} & \fullrankvi{} & \flowvi{} \\
\midrule
Target & hard posterior & relaxed posterior & relaxed posterior & relaxed posterior \\
 Approx.& discrete-state MCMC & gradient MCMC & full-rank VI & normalizing-flow VI \\
\bottomrule
\end{tabular}
\end{table}

\textbf{Posterior decoding.}
\label{subsec:posterior_decoding}
For closure-level evaluation, we decode each posterior draw:
\[
H^{(s)}_{ij}
=
\mathbf 1\!\left[u^{(s)}_{i,\ell}>u^{(s)}_{j,\ell}\ \forall \ell=1,\ldots,d\right],
\quad
H^{(s)}_{ii}=0,
\quad
\widehat P_{ij}=S^{-1}\sum_{s=1}^S H^{(s)}_{ij}.
\]
We estimate \(\widehat h\) by thresholding \(\widehat P\) at \(\zeta=1/2\) (\(\zeta=1/3\) for \(n=100\)), breaking threshold-induced cycles, and taking the
transitive closure.

\noindent\textbf{Scalability.}
A poset on \(m\) items admits up to \(m!\) linear extensions, and counting them is \(\#\)P-complete \citep{brightwell1991linearextensions}. \hardmcmc{} avoids this because it is based on the Frontier-Softmax likelihood in Section~\ref{subsec:frontier_softmax_likelihood}, but exploring the space of partial orders $\mathcal{H}_\mathcal{M}$ still slows dramatically with increasing trace length $M=|\mathcal{M}|$. By contrast, after precomputing the soft-precedence matrix, one relaxed likelihood evaluation costs \(O(M^2 d + \sum_n T_n^2)\), so for fixed \(d\) the dominant cost is quadratic in trace length. Full derivations are in Appendix~\ref{app:scaling_analysis}.

\section{Experiments}
\label{sec:experiments}

We ask whether the differentiable relaxation preserves the hard Bayesian partial-order posterior
while making inference practical beyond the regime of discrete posterior exploration. On synthetic
data, where long-run \hardmcmc{} is feasible, we measure posterior fidelity using Mean Absolute
Error (MAE) between pairwise precedence posteriors and compare runtime; at larger scales, we
evaluate closure recovery and predictive fit. We study two real data settings: the Bishops
witness-list corpus \citep{jiang2023bayesian}, using WAIC, and the Cloud agent-trace benchmark
\citep{li2026delinearizingagenttracesbayesian}, using closure recovery and next-action prediction.

Our baselines isolate distinct alternatives. \hardmcmc{} serves as the discrete posterior
reference, testing fidelity to the original frontier-constrained Bayesian target. Majority
(Appendix~\ref{alg:majority_baseline_instruct}) tests whether simple pairwise precedence counts
suffice without posterior inference or closure constraints. SoftDAG-Frontier
(\cite{zheng2018dags,lorch2021dibs}, Appendix~\ref{app:softdag}) uses the same frontier-softmax likelihood but replaces the
product-order representation with a generic differentiable DAG, inspired by continuous DAG
learning and testing whether product-order
geometry is an essential inductive bias rather than a convenience. When ground truth is available,
we report closure precision, recall, and F1; predictive fit is measured by held-out NLL and WAIC
(Appendix~\ref{app:evaluation_metrics}).

\subsection{Synthetic Experiments}
\label{subsec:synthetic_experiments}
Synthetic experiments test whether the relaxation preserves the hard partial-order posterior while improving scalability. We generate datasets with \(n\in\{5,10,20,30,50,100\}\), \(\rho\in\{0.5,0.9\}\), and three seeds per configuration. Training traces are selected to maximize incomparable-pair coverage under a budget of \(n\) to \(2n\) lists; test traces are drawn from the same ground-truth partial order. Full generation details and results are in Appendix~\ref{app:synthetic_full_grid}.

Figure~\ref{fig:synthetic_scaling} summarizes posterior fidelity, runtime, and closure recovery.
Where the \hardmcmc{} reference is feasible, the relaxed methods achieve substantially lower
posterior MAE than Majority: they approximate the hard Bayesian posterior rather
than merely exploiting pairwise frequencies. As \(n\) grows, \hardmcmc{} becomes rapidly more
expensive, while continuous inference maintains strong closure recovery at much lower runtime.
\relaxedmcmc{} is closest to the hard reference, \fullrankvi{} gives a strong
accuracy--speed trade-off up to \(n=50\), and \flowvi{} scales to \(n=100\) (Figure~\ref{fig:flow_n100_closure}). The same figure also includes SoftDAG-Frontier, a generic differentiable
DAG with the frontier-softmax likelihood; its weaker closure recovery suggests that the
product-order geometry is an important inductive bias, not just an implementation choice; see details in Appendix \ref{app:softdag}.

Appendix diagnostics confirm two sensitivities: IP-Cov shows that missing reversals between
incomparable pairs can induce false precedences, while the \(\tau\) ablation shows that
\(\tau\in\{0.1,0.3\}\) tracks the hard reference and \(\tau=1.0\) over-smooths
(Appendices~\ref{sec:ipcov-ablation} and~\ref{sec:tau-ablation}).

\begin{figure}[h]
    \centering
    \includegraphics[width=\textwidth]{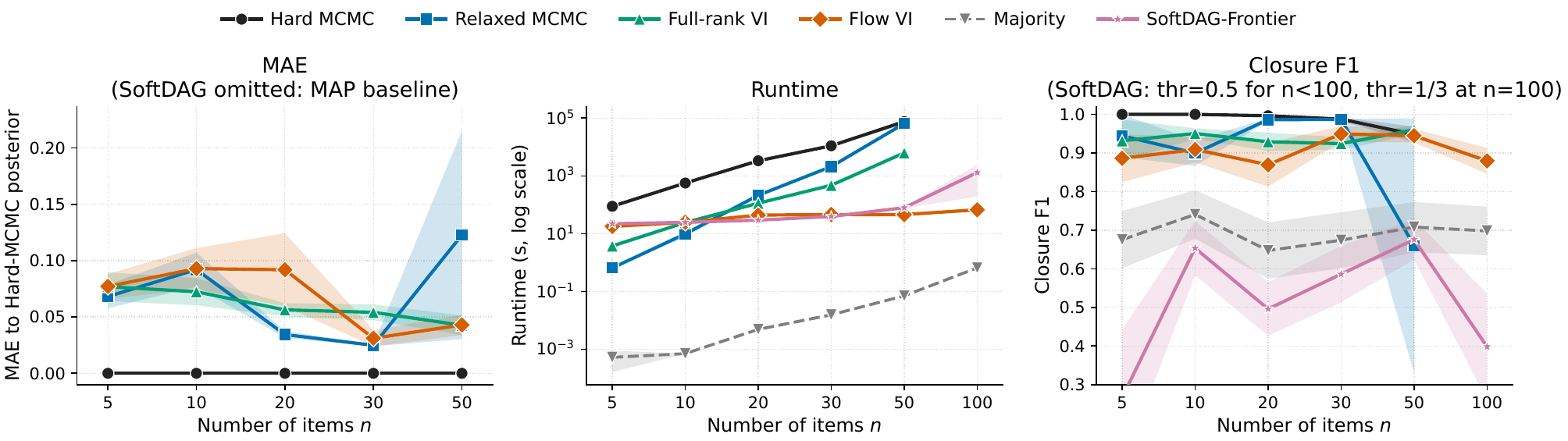}
\caption{Synthetic scaling at relaxation temperature \(\tau=0.3\) with full IP-Cov. Panels show
posterior MAE to \hardmcmc{}, runtime on a log scale, and closure F1 as \(n\) grows.
SoftDAG-Frontier is a generic differentiable-DAG baseline using the same frontier-softmax
likelihood. Points average completed seeds over \(\rho\in\{0.5,0.9\}\); bands show
\(\pm\) one SEM. The wide \relaxedmcmc{} band at \(n=50\) is driven by one high-variance seed. SoftDAG-Frontier closure F1 uses \(\theta_{\mathrm{DAG}}=0.5\) for \(n<100\)
and \(\theta_{\mathrm{DAG}}=1/3\) for \(n=100\).}
    \label{fig:synthetic_scaling}
\end{figure}
\subsection{Bishops witness-list corpus}
\label{app:bishop}

We evaluate on a set of data drawn from the Royal-Acta corpus published by the 
\emph{The Charters of William II and Henry I} project
\citep{sharpe2014charters,nicholls2025royalacta}. Our data contains \(68\)
incomplete ordered witness lists over \(n=45\) actors; list lengths and choice
sets vary, so observations inform overlapping suborders of the underlying partial order. With no ground-truth partial order, we
report list-level WAIC \citep{watanabe2010,gelman2014waic}, runtime, and
threshold-free posterior comparisons. As external total-order baselines, we fit
PLMIX models with \(G\in\{1,2,3\}\) mixture components
\citep{luce1959,mollica2017bayesian}, with WAIC computed on PLMIX likelihood. These baselines test whether finite mixtures of total orders can close the predictive gap (Appendix~\ref{app:pl_base}).

The relaxed methods use \(d=3\) and \(\tau=0.10\); all partial-order methods fix \(\beta=0\), corresponding to uniform frontier selection.
\relaxedmcmc{} gives the best WAIC (\(956.5\)), followed by \fullrankvi{}
(\(999.4\)), \hardmcmc{} (\(1010.4\)), and \flowvi{} (\(1034.6\)).
All Plackett--Luce baselines are substantially worse under the same list-level
WAIC criterion: \(G=1\) gives \(1588.6\), \(G=2\) gives \(1820.5\), and
\(G=3\) gives \(1733.5\). In runtime, \fullrankvi{} is the fastest
partial-order method at 5.5 minutes, followed by \flowvi{} at around 9 minutes,
\relaxedmcmc{} at 27 minutes, and \hardmcmc{} at 70 minutes. Thus,
\relaxedmcmc{} gives the best predictive fit, while \fullrankvi{} offers the
best practical trade-off, improving over \hardmcmc{} in both WAIC and runtime
for short lists over a large item universe.

Because thresholded graphs depend on the cutoff, Figure~\ref{fig:bishop_pairwise_agreement}
compares posterior pairwise precedence probabilities directly against \hardmcmc{}.
\relaxedmcmc{} is closest to the hard posterior (MAE \(=0.049\)), followed by
\fullrankvi{} (MAE \(=0.068\)); \flowvi{} is more diffuse (MAE \(=0.123\)).
Disagreement is concentrated on intermediate-probability pairs. Additional
heatmaps, Hasse diagrams, and scalar summaries are in
Appendix~\ref{app:bishops_details}. The recovered Hasse diagrams
(Figure~\ref{fig:bishop_posterior_hasse}) broadly align with recorded titles,
which are not used in fitting, supporting the partial order as an interpretable
social-precedence summary rather than a forced total ranking.
\begin{figure}[h]
    \centering
    \includegraphics[width=0.98\textwidth]{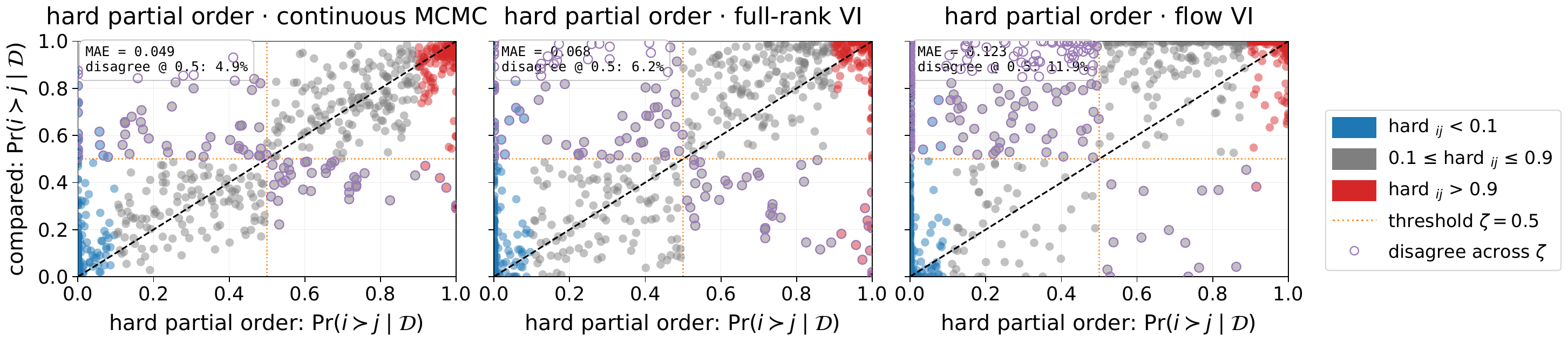}
\caption{%
\textbf{Pairwise closure agreement on the Bishops corpus.}
For each ordered pair $(i,j)$, panels compare $\Pr(i\succ j\mid\mathcal D)$ under hard RJ-MCMC
($x$-axis) with relaxed MCMC, full-rank VI, and flow VI ($y$-axis, left to right).
The dashed diagonal marks agreement; rings indicate pairs crossing the threshold $\zeta$,
and insets show mean absolute deviation.}
\label{fig:bishop_pairwise_agreement}
\end{figure}
\subsection{Cloud agent-trace benchmark}
\label{subsec:aliyun_experiment}
The Cloud agent-trace benchmark, adapted from
\citet{li2026delinearizingagenttracesbayesian}, contains 54 successful workflow
traces generated by multiple LLM agents across six infrastructure-as-code scenarios.
The 12 evaluation cases correspond to six scenarios and two train/test seeds, with
workflow sizes ranging from 5 to 12 actions. We use this benchmark to test whether
the relaxation preserves the hard partial-order semantics of the Li et al. hard-MCMC
baseline while improving efficiency, treating each trace as a linearization of an
underlying workflow (Figure~\ref{fig:aliyun-gt-covers}).

\begin{table}[h]
\centering
\small
\caption{Cloud agent-trace benchmark: mean structural, predictive, and runtime metrics over 12 cases. Trace NLL, Step NLL are posterior-predictive quantities over retained draws. SoftDAG-Frontier is a MAP diagnostic baseline with the same frontier-softmax.}
\label{tab:aliyun_metrics_updated}
\begin{tabular}{lrrrrrrrr}
\toprule
Method & Runs & Time (s) & F1 & Prec. & Rec. & Trace NLL & Step NLL \\
\midrule
\hardmcmc{}    & 12 & 1455 & 0.99 & 1.00 & 0.97 & 5.61 & 0.70 \\
\relaxedmcmc{} & 12 &  291 & 0.98 & 0.97 & 1.00 & 3.70 & 0.49  \\
\fullrankvi{}  & 12 &   54 & 0.96 & 0.93 & 0.99 & 3.85 & 0.51  \\
\flowvi{}      & 12 &   99 & 0.88 & 0.81 & 0.99 & 3.87 & 0.54 \\
\midrule
\textsc{SoftDAG-Frontier} & 12 & 40 & 0.56 & 0.62 & 0.70 & 5.23 & 0.72  \\
\bottomrule
\end{tabular}
\end{table}

Table~\ref{tab:aliyun_metrics_updated} shows a trade-off among structural recovery,
predictive performance, and runtime. \hardmcmc{} gives the strongest closure
recovery but is by far the slowest. \relaxedmcmc{} nearly matches its closure
structure and gives the best posterior-predictive Trace NLL and Step NLL.
\fullrankvi{} is slightly weaker structurally but is the fastest posterior
method, giving the strongest practical speed--accuracy trade-off. \flowvi{}
improves Step NLL relative to \hardmcmc{}, but has weaker closure recovery.
Appendix~\ref{app:aliyun_prefix_anomaly} provides additional diagnostics,
including Hasse diagrams (Figure~\ref{fig:aliyun_qualitative_example}),
convergence traces (Figure~\ref{fig:aliyun_convergence_traces}),
per-scenario next-action results (Table~\ref{tab:aliyun-prefix-nll-mae}), and
scalar posterior variation (Figure~\ref{fig:aliyun_scalar_posteriors_grid}).
\section{Related work}
\label{subsec:related_work}

Our work connects repeated-selection ranking models, Bayesian inference for latent partial orders, and differentiable relaxations of discrete order structures.

\noindent\textbf{Ranking models and repeated selection.}
Many classical ranking models treat an observed order as a sequence of choices from a shrinking set, including the Plackett model for permutations~\citep{luce1959,plackett1975} and some Mallows models \citep{mallows1957non}. This repeated-selection view is developed more generally in modern ranking models~\citep{ragain2019choosing} and has been extended to top-\(k\) partial orders with stopping decisions~\citep{awadelkarim2024topk}. Our likelihood follows the same sequential-choice principle, but the feasible set at each step is the frontier of a partial order rather than the full remaining set. Fast approximate likelihoods for models for rank data based on discrete latent parameters have been given \citep{liu2022pseudomallows} but ours is the first differentiable relaxation of a ranking model with a discrete parameter or any model for Bayesian inference of a partial order.

\noindent\textbf{Bayesian partial-order inference.}
Bayesian partial-order models are natural when rank-data can be viewed as random linear extensions and uncertainty over precedence relations matters~\citep{nicholls2025royalacta}. The main challenge is combinatorial: partial orders form a discrete space, and counting linear extensions is \(\#\mathsf{P}\)-complete~\citep{brightwell1991linearextensions}. Prior Bayesian work has imposed tractable structure~\citep{jiang2023bayesian} %, extended the models to ties and Mallows noise~\citep{jiang2024nonparametric}, 
or sampled partial orders as an auxiliary search space for Bayesian-network structure discovery \citep{niinimaki2016structure}. Related combinatorial work finds posets that contain a given set of linear orders \citep{fernandez2013mining} (without uncertainty) and samples linear extensions of a fixed poset \citep{talvitie2017mixing} (without poset estimation). Closest to our setting, \citet{li2026delinearizingagenttracesbayesian} infer latent prerequisite structure from agent traces using a hard frontier-constrained likelihood and long-run MCMC. We retain the same frontier semantics but relax hard precedence and frontier indicators to allow scalable gradient-based MCMC and variational inference.

\noindent\textbf{Geometric order representations and differentiable graph learning.}
Geometric order representations encode transitive asymmetric structure using
ordered point embeddings, density embeddings, or box/lattice measures
\citep{vendrov2016order,athiwaratkun2018hierarchical,vilnis2018probabilistic}. These works motivate our product-order geometry but do not define generative models for rank data and traces. Our relaxation is also related to smooth surrogates for discrete choices~\citep{maddison2017concrete,jang2017categorical} and to differentiable graph learning~\citep{zheng2018dags,lorch2021dibs,annadani2023bayesdag}. Unlike these methods, we observe linearized traces and infer the latent partial order itself, and in particular, SoftDAG-Frontier, which we developed as a diagnostic baseline inspired by \cite{ban2024differentiable}, was not set out by its authors as a model for rank-data or intended to express the kind of transitive relations we get for free in a product order embedding.

\section{Discussion and Limitations}
\label{sec:discussion}
We introduced a differentiable Bayesian relaxation for latent partial-order inference from noisy linear extensions. The approximate posterior preserves the frontier-constrained semantics of hard partial-order inference while enabling gradient-based MCMC and variational inference, and is actually a good model in its own right. Across synthetic, ranking, and cloud-trace data, the relaxed methods closely match \hardmcmc{} where feasible and give better runtime--accuracy trade-offs at larger scale, with \relaxedmcmc{} closest to the hard posterior and \fullrankvi{} often giving the best speed--accuracy compromise.

The main limitations are scope and inductive bias. Hard-posterior fidelity can only be checked on small and medium instances; at larger \(n\), we rely on structural recovery, predictive fit, and finite-temperature theory. Scaling stops at $M=100$. Memory demands make $M=200$ infeasible, though careful coding may deal with that. We treat \(d\) and \(\tau\) as fixed hyperparameters and infer \(\gamma\) under the posterior, reflecting the bias--smoothness trade-off of the relaxation. The product-order representation needs \(d\ge \lfloor M/2\rfloor\) to express all high-dimension posets \citep{Hiraguchi51}, while \(\beta\), \(\gamma\), and the decoding threshold $\zeta$ remain modeling choices. We therefore report posterior probabilities and sensitivity diagnostics rather than relying on a single decoded graph.

\bibliographystyle{plainnat}
\bibliography{dpo_nips.bib}

@inproceedings{ragain2019choosing,
  title     = {Learning Rich Rankings},
  author    = {Seshadri, Arjun and Ragain, Stephen and Ugander, Johan},
  booktitle = {Advances in Neural Information Processing Systems},
  volume    = {33},
  pages     = {9435--9446},
  year      = {2020},
  publisher = {Curran Associates, Inc.}
}

@inproceedings{awadelkarim2024topk,
  title={Statistical Models of Top-k Partial Orders},
  author={Awadelkarim, Amel and Ugander, Johan},
  booktitle={Proceedings of the 30th ACM SIGKDD Conference on Knowledge Discovery and Data Mining},
  pages={39--48},
  year={2024}
}

@article{nicholls2025royalacta,
  title={Bayesian Inference for Partial Orders from Random Linear Extensions: Power Relations from 12th Century Royal Acta},
  author={Nicholls, Geoff K. and Lee, Jeong Eun and Karn, Nicholas and Johnson, David and Huang, Rukuang and Muir-Watt, Alexis},
  journal={The Annals of Applied Statistics},
  volume={19},
  number={2},
  pages={1663--1690},
  year={2025},
  doi={10.1214/24-AOAS2002}
}

@article{brightwell1991linearextensions,
  title   = {Counting linear extensions},
  author  = {Brightwell, Graham and Winkler, Peter},
  journal = {Order},
  volume  = {8},
  number  = {3},
  pages   = {225--242},
  year    = {1991},
  doi     = {10.1007/BF00383444}
}

@inproceedings{ban2024differentiable,
  title     = {Differentiable Structure Learning with Partial Orders},
  author    = {Ban, Taiyu and Chen, Lyuzhou and Wang, Xiangyu and Wang, Xin and Lyu, Derui and Chen, Huanhuan},
  booktitle = {Advances in Neural Information Processing Systems},
  volume    = {37},
  year      = {2024},
  doi       = {10.52202/079017-3728}
}

@misc{li2026delinearizingagenttracesbayesian,
      title={De-Linearizing Agent Traces: Bayesian Inference of Latent Partial Orders for Efficient Execution}, 
      author={Dongqing Li and Zheqiao Cheng and Geoff K. Nicholls and Quyu Kong},
      year={2026},
      eprint={2602.02806},
      archivePrefix={arXiv},
      primaryClass={stat.AP},
      url={https://arxiv.org/abs/2602.02806}, 
}

@phdthesis{seshadri21,
    author = {Seshadri, A.},
    title = {Learning preferences from choices and rankings},
    school = {Stanford University Department of Electrical Engineering},
    year = {2021},
    url={https://purl.stanford.edu/zv003vd9732}
}

@article{nesterov2005smooth,
  title={Smooth Minimization of Non-smooth Functions},
  author={Nesterov, Yurii},
  journal={Mathematical Programming},
  volume={103},
  number={1},
  pages={127--152},
  year={2005}
}

@inproceedings{vendrov2016order,
  author    = {Vendrov, Ivan and Kiros, Ryan and Fidler, Sanja and Urtasun, Raquel},
  title     = {Order-Embeddings of Images and Language},
  booktitle = {International Conference on Learning Representations},
  year      = {2016},
  note      = {ICLR 2016 Oral},
  url       = {https://arxiv.org/abs/1511.06361}
}

@inproceedings{maddison2017concrete,
  title={The Concrete Distribution: A Continuous Relaxation of Discrete Random Variables},
  author={Maddison, Chris J and Mnih, Andriy and Teh, Yee Whye},
  booktitle={International Conference on Learning Representations (ICLR)},
  year={2017}
}

@inproceedings{jang2017categorical,
  title     = {Categorical Reparameterization with {Gumbel-Softmax}},
  author    = {Jang, Eric and Gu, Shixiang and Poole, Ben},
  booktitle = {International Conference on Learning Representations},
  year      = {2017},
  url       = {https://openreview.net/forum?id=rkE3y85ee}
}

@inproceedings{jiang2023bayesian,
  author    = {Chuxuan Jiang and Geoff K. Nicholls and Jeong Eun Lee},
  title     = {Bayesian Inference for Vertex-Series-Parallel Partial Orders},
  booktitle = {Proceedings of the Thirty-Ninth Conference on Uncertainty in Artificial Intelligence},
  series    = {Proceedings of Machine Learning Research},
  volume    = {216},
  pages     = {995--1004},
  year      = {2023}
}

@inproceedings{athiwaratkun2018hierarchical,
  author    = {Ben Athiwaratkun and Andrew Gordon Wilson},
  title     = {Hierarchical Density Order Embeddings},
  booktitle = {International Conference on Learning Representations (ICLR)},
  year      = {2018}
}

@misc{sharpe2014charters,
  author = {Sharpe, Richard and Carpenter, David X. and Doherty, Hugh and Hagger, Mark and Karn, Nicholas},
  title  = {The Charters of {William II} and {Henry I}},
  year   = {2014},
  note   = {Online database; accessed 1 October 2022},
  url    = {https://actswilliam2henry1.wordpress.com/the-charters/}
}

@inproceedings{vilnis2018probabilistic,
  author    = {Luke Vilnis and Xiang Li and Shikhar Murty and Andrew McCallum},
  title     = {Probabilistic Embedding of Knowledge Graphs with Box Lattice Measures},
  booktitle = {Proceedings of the 56th Annual Meeting of the Association for Computational Linguistics (Volume 1: Long Papers)},
  pages     = {263--272},
  year      = {2018},
  address   = {Melbourne, Australia},
  publisher = {Association for Computational Linguistics},
  doi       = {10.18653/v1/P18-1025}
}

@article{watanabe2010,
  author  = {Watanabe, Sumio},
  title   = {Asymptotic Equivalence of {B}ayes Cross Validation and Widely
             Applicable Information Criterion in Singular Learning Theory},
  journal = {Journal of Machine Learning Research},
  year    = {2010},
  volume  = {11},
  number  = {116},
  pages   = {3571--3594},
  url     = {https://www.jmlr.org/papers/v11/watanabe10a.html}
}

@article{gelman2014waic,
  author  = {Gelman, Andrew and Hwang, Jessica and Vehtari, Aki},
  title   = {Understanding Predictive Information Criteria for {B}ayesian Models},
  journal = {Statistics and Computing},
  year    = {2014},
  volume  = {24},
  number  = {6},
  pages   = {997--1016},
  doi     = {10.1007/s11222-013-9416-2}
}

@book{luce1959,
  author    = {Luce, R. Duncan},
  title     = {Individual Choice Behavior: A Theoretical Analysis},
  publisher = {Wiley},
  address   = {New York},
  year      = {1959}
}

@article{plackett1975,
  author  = {Robin L. Plackett},
  title   = {The Analysis of Permutations},
  journal = {Journal of the Royal Statistical Society. Series C (Applied Statistics)},
  volume  = {24},
  number  = {2},
  pages   = {193--202},
  year    = {1975},
  doi     = {10.2307/2346567}
}

@article{mollica2017bayesian,
  author  = {Mollica, Cristina and Tardella, Luca},
  title   = {{Bayesian} {Plackett--Luce} Mixture Models for Partially Ranked Data},
  journal = {Psychometrika},
  volume  = {82},
  number  = {2},
  pages   = {442--458},
  year    = {2017},
  doi     = {10.1007/s11336-016-9530-0}
}

@article{leemans2023partial,
  title={Partial-order-based process mining: A survey and outlook},
  author={Leemans, Sander J. J. and van Zelst, Sebastiaan J. and Lu, Xixi},
  journal={Knowledge and Information Systems},
  volume={65},
  pages={1--29},
  year={2023},
  doi={10.1007/s10115-022-01777-3}
}

@article{beerenwinkel2007conjunctive,
  title   = {Conjunctive {Bayesian} networks},
  author  = {Beerenwinkel, Niko and Eriksson, Nicholas and Sturmfels, Bernd},
  journal = {Bernoulli},
  volume  = {13},
  number  = {4},
  pages   = {893--909},
  year    = {2007},
  doi     = {10.3150/07-BEJ6133}
}

@inproceedings{lorch2021dibs,
  title     = {{DiBS}: Differentiable Bayesian Structure Learning},
  author    = {Lorch, Lars and Rothfuss, Jonas and Sch{\"o}lkopf, Bernhard and Krause, Andreas},
  booktitle = {Advances in Neural Information Processing Systems},
  volume    = {34},
  pages     = {24111--24123},
  year      = {2021}
}

@article{niinimaki2016structure,
  title   = {Structure Discovery in Bayesian Networks by Sampling Partial Orders},
  author  = {Niinim{\"a}ki, Teppo and Parviainen, Pekka and Koivisto, Mikko},
  journal = {Journal of Machine Learning Research},
  volume  = {17},
  number  = {57},
  pages   = {1--47},
  year    = {2016}
}

@inproceedings{mannila00,
author = {Mannila, Heikki and Meek, Christopher},
title = {Global Partial Orders from Sequential Data},
year = {2000},
isbn = {1581132336},
publisher = {Association for Computing Machinery},
address = {New York, NY, USA},
url = {https://doi.org/10.1145/347090.347122},
doi = {10.1145/347090.347122},
booktitle = {Proceedings of the Sixth ACM SIGKDD International Conference on Knowledge Discovery and Data Mining},
pages = {161–168},
numpages = {8},
location = {Boston, Massachusetts, USA},
series = {KDD '00}
}

@inproceedings{mannila08,
  author    = {Mannila, Heikki},
  title     = {Finding Total and Partial Orders from Data for Seriation},
  booktitle = {Discovery Science, 11th International Conference, DS 2008, Budapest, Hungary, October 13--16, 2008, Proceedings},
  editor    = {Boulicaut, Jean-Fran{\c{c}}ois and Berthold, Michael R. and Horv{\'a}th, Tam{\'a}s},
  series    = {Lecture Notes in Computer Science},
  volume    = {5255},
  pages     = {16--25},
  publisher = {Springer},
  address   = {Berlin, Heidelberg},
  year      = {2008},
  doi       = {10.1007/978-3-540-88411-8_4}
}

@article{nicholls11,
  title={Partial Order Models for Episcopal Social Status in 12th Century {E}ngland},
  author={Nicholls, Geoff K and Muir Watt, Alexis },
  journal={IWSM 2011},
  pages={437},
  year={2011}
}

@article{winkler1985random,
  title={Random orders},
  author={Winkler, Peter},
  journal={Order},
  volume={1},
  number={4},
  pages={317--331},
  year={1985},
  publisher={Springer}
}

@inproceedings{zheng2018dags,
  title     = {{DAGs with NO TEARS}: Continuous Optimization for Structure Learning},
  author    = {Zheng, Xun and Aragam, Bryon and Ravikumar, Pradeep K. and Xing, Eric P.},
  booktitle = {Advances in Neural Information Processing Systems},
  volume    = {31},
  year      = {2018}
}

@inproceedings{annadani2023bayesdag,
  title     = {{BayesDAG}: Gradient-Based Posterior Inference for Causal Discovery},
  author    = {Annadani, Yashas and Pawlowski, Nick and Jennings, Joel and Bauer, Stefan and Zhang, Cheng and Gong, Wenbo},
  booktitle = {Advances in Neural Information Processing Systems},
  volume    = {36},
  year      = {2023}
}

@article{fernandez2013mining,
  title   = {Mining Posets from Linear Orders},
  author  = {Fernandez, Proceso L. and Heath, Lenwood S. and Ramakrishnan, Naren and Tan, Michael and Vergara, John Paul C.},
  journal = {Discrete Mathematics, Algorithms and Applications},
  volume  = {5},
  number  = {4},
  pages   = {1350030},
  year    = {2013}
}

@inproceedings{talvitie2017mixing,
  title     = {The Mixing of Markov Chains on Linear Extensions in Practice},
  author    = {Talvitie, Topi and Niinim{\"a}ki, Teppo and Koivisto, Mikko},
  booktitle = {Proceedings of the Twenty-Sixth International Joint Conference on Artificial Intelligence},
  pages     = {524--530},
  year      = {2017},
  publisher = {IJCAI},
  doi       = {10.24963/ijcai.2017/74}
}

@article{Hiraguchi51,
  author  = {Hiraguchi, Toshio},
  title   = {On the Dimension of Partially Ordered Sets},
  journal = {The Science Reports of the Kanazawa University},
  volume  = {1},
  number  = {2},
  pages   = {77--94},
  year    = {1951},
  mrnumber = {0070681},
  zbl      = {0200.00013}
}

@article{bogart1973maximal,
   author = {Kenneth P. Bogart},
   doi = {10.1016/0012-365X(73)90024-1},
   issn = {0012-365X},
   issue = {1},
   journal = {Discrete Mathematics},
   month = {5},
   pages = {21-31},
   publisher = {North-Holland},
   title = {Maximal dimensional partially ordered sets {I}. {H}iraguchi's theorem},
   volume = {5},
   year = {1973},
}

@article{yannakakis1982complexity,
  title={The complexity of the partial order dimension problem},
  author={Yannakakis, Mihalis},
  journal={SIAM Journal on Algebraic Discrete Methods},
  volume={3},
  number={3},
  pages={351--358},
  year={1982},
  publisher={SIAM}
}

@article{hoffman2014nuts,
  title = {The No-U-Turn Sampler: Adaptively Setting Path Lengths in Hamiltonian Monte Carlo},
  author = {Hoffman, Matthew D. and Gelman, Andrew},
  journal = {Journal of Machine Learning Research},
  volume = {15},
  number = {47},
  pages = {1593--1623},
  year = {2014},
  url = {https://jmlr.org/papers/v15/hoffman14a.html}
}

@article{carpenter2017stan,
  title = {Stan: A Probabilistic Programming Language},
  author = {Carpenter, Bob and Gelman, Andrew and Hoffman, Matthew D. and Lee, Daniel
            and Goodrich, Ben and Betancourt, Michael and Brubaker, Marcus
            and Guo, Jiqiang and Li, Peter and Riddell, Allen},
  journal = {Journal of Statistical Software},
  volume = {76},
  number = {1},
  pages = {1--32},
  year = {2017},
  doi = {10.18637/jss.v076.i01},
  url = {https://www.jstatsoft.org/article/view/v076i01}
}

@article{kucukelbir2017advi,
  title = {Automatic Differentiation Variational Inference},
  author = {Kucukelbir, Alp and Tran, Dustin and Ranganath, Rajesh
            and Gelman, Andrew and Blei, David M.},
  journal = {Journal of Machine Learning Research},
  volume = {18},
  number = {14},
  pages = {1--45},
  year = {2017},
  url = {https://jmlr.org/papers/v18/16-107.html}
}

@inproceedings{rezende2015flows,
  title = {Variational Inference with Normalizing Flows},
  author = {Rezende, Danilo Jimenez and Mohamed, Shakir},
  booktitle = {Proceedings of the 32nd International Conference on Machine Learning},
  series = {Proceedings of Machine Learning Research},
  volume = {37},
  pages = {1530--1538},
  year = {2015},
  publisher = {PMLR},
  url = {https://proceedings.mlr.press/v37/rezende15.html}
}

@article{mallows1957non,
  title={Non-null ranking models. I},
  author={Mallows, Colin L},
  journal={Biometrika},
  volume={44},
  number={1/2},
  pages={114--130},
  year={1957},
  publisher={JSTOR}
}

@misc{liu2022pseudomallows,
      title={Pseudo-Mallows for Efficient Probabilistic Preference Learning}, 
      author={Qinghua Liu and Valeria Vitelli and Carlo Mannino and Arnoldo Frigessi and Ida Scheel},
      year={2022},
      eprint={2205.13911},
      archivePrefix={arXiv},
      primaryClass={stat.ME},
      url={https://arxiv.org/abs/2205.13911}, 
}

\appendix

\section{Partial Order Basics}
\subsection{Notation reference}
\label{app:notation}

\renewcommand{\arraystretch}{1.18}
\small

\begin{longtable}{p{0.22\linewidth}p{0.70\linewidth}}
\toprule
\multicolumn{2}{l}{\textbf{Sets, traces, and hard partial orders}}\\
\midrule
\(M=[M]\) & Universe of items. \\
\(B_M\) & Set of all nonempty subsets of \(M\). \\
\(H_S\) & Class of all partial orders on items in \(S\). \\
\(S\) & Local item set. \\
\(m=|S|\) & Number of items in the local item set. \\
\(y=(y_1,\dots,y_T)\) & Observed local trace or ranking. In complete rankings, \(T=m\). \\
\(y_{<t}\) & Prefix \((y_1,\dots,y_{t-1})\) observed before step \(t\). \\
\(\yt\) & Remainder set \((y_t,\dots,y_{T})\) at step \(t\), i.e., items not yet selected. \\
\(h=(S,\succ_h)\) & Latent strict partial order on \(S\). \\
\(z\succ_h x\) & Hard precedence relation: item \(z\) must occur before item \(x\). \\
\(F(h;\yt)\) & Hard frontier at step \(t\): remaining items with no remaining predecessor. \\
\(F_t(x;h)\) & Hard frontier indicator, \(F_t(x;h):=\mathbf{1}\{x\in F(h;y_{\ge t})\}\). \\\\
\(S_t(x;h)\) & Hard descendant count of item \(x\) in the current remainder set. \\
\(Q_t(x;h)\) & Hard successor utility, \(Q_t(x;h)=\log\!\bigl(1+S_t(x;h)\bigr)\) for \(x\in\mathcal{F}(h;y_{\ge t})\). \\

\midrule
\multicolumn{2}{l}{\textbf{Latent product-order embedding}}\\
\midrule
\(d\) & Latent embedding dimension. \\
\(u_x=(u_{x,1},\dots,u_{x,d})^\top\in\mathbb R^d\) & Latent embedding vector for item \(x\). \\
\(U\in\mathbb R^{m\times d}\) & Matrix of latent embeddings; the \(x\)-th row is \(u_x^\top\). \\
\(z_x\in\mathbb R^d\) & Non-centred latent coordinate associated with item \(x\). \\
\(Z\in\mathbb R^{m\times d}\) & Matrix of non-centred latent coordinates; the \(x\)-th row is \(z_x^\top\). \\
\(\rho\in(0,1)\) & Correlation parameter for the latent Gaussian prior. \\
\(\Sigma_\rho=(1-\rho)I_d+\rho\mathbf 1\mathbf 1^\top\) & Prior covariance for centred embedding coordinates. \\
\(L_\rho\) & Cholesky factor of \(\Sigma_\rho\), so that \(\Sigma_\rho=L_\rho L_\rho^\top\). \\
\(u_x=L_\rho z_x\) & Non-centred parameterization of the latent embedding. Equivalently, \(U=ZL_\rho^\top\). \\
\(\Delta_k(z,x)=u_{z,k}-u_{x,k}\) & Coordinatewise gap between items \(z\) and \(x\). \\
\(m_U(z,x)=\min_k\Delta_k(z,x)\) & Hard pairwise margin; its sign determines hard precedence. \\
\(\succ_U\) & Product-order relation induced by \(U\). \\
\(z\succ_U x\) & Hard latent precedence relation induced by coordinatewise dominance: \(u_{z,k}>u_{x,k}\) for all \(k\). \\
\(h_U=(S,\succ_U)\) & Hard latent partial order induced by the embedding \(U\). \\

\midrule
\multicolumn{2}{l}{\textbf{Differentiable relaxation}}\\
\midrule
\(\smin_\tau(a_1,\dots,a_d)\) & Soft minimum with temperature \(\tau>0\). \\
\(\tau>0\) & Soft-min temperature controlling the sharpness of the smooth minimum. Treated as fixed in the main experiments. \\
\(M_U(z,x)\) & Soft pairwise margin,
\[
M_U(z,x)=\smin_\tau(\Delta_1(z,x),\dots,\Delta_d(z,x)).
\]
It is a smooth approximation to \(m_U(z,x)\). \\
\(\gamma>0\) & Steepness parameter for the soft precedence score. \\
\(D_U(z,x)=\sigma(\gamma M_U(z,x))\) & Soft precedence score; a differentiable surrogate for \(\mathbf 1[z\succ_U x]\). \\
\(\widetilde F_t(x;U)\) & Soft frontier weight; a differentiable surrogate for \(\mathbf 1[x\in F(h_U;\yt)]\). \\
\(\widetilde S_t(x;U)\) & Soft descendant count, obtained by summing \(D_U(x,z)\) over \(z\in \yt\setminus\{x\}\). \\
\(\widetilde Q_t(x;U)\) & Relaxed successor utility,
\[
\widetilde Q_t(x;U)=\log(1+\widetilde S_t(x;U)).
\]
\\
\(\beta>0\) & Inverse temperature in the frontier-softmax selection rule. \\
\(\widetilde p(y_t\mid \yt,U,\beta,\gamma)\) & Relaxed stepwise repeated-selection probability. \\
\(\widetilde p(y\mid U,\beta,\gamma)\) & Relaxed local-trace likelihood. \\

\midrule
\multicolumn{2}{l}{\textbf{Theory quantities}}\\
\midrule
\(\delta_t(U)\) & Local separation margin at step \(t\),
\[
\delta_t(U)=
\min_{\substack{x,z\in \yt\\x\neq z}}
|m_U(z,x)|.
\]
\\
\(\kappa_t(U,\tau,\gamma)\) & Finite-temperature approximation scale,
\[
\kappa_t(U,\tau,\gamma)
=
\tau\log d+
\exp\!\bigl(-\gamma(\delta_t(U)-\tau\log d)\bigr).
\]
\\

\midrule
\multicolumn{2}{l}{\textbf{Bayesian inference}}\\
\midrule
\(\mathcal D\) & Observed dataset of local traces. \\
\(\Theta=(Z,\rho,\beta,\gamma)\) & Full parameter vector under the non-centred relaxed Bayesian model. \\
\(p(\mathcal D,Z,\rho,\beta,\gamma)\) & Joint density of the data and latent variables under the relaxed Bayesian model. \\
\(w=(\operatorname{vec}(Z),\eta_\rho,\eta_\beta,\eta_\gamma)\) & Unconstrained parameter vector used for variational inference. \\
\(T(w)\) & Smooth map from unconstrained coordinates \(w\) to constrained parameters \(\Theta\). \\
\(\bar p(\mathcal D,w)\) & Transformed joint density in unconstrained coordinates, including the Jacobian correction. \\
\(q_\phi(w)\) & Normalizing-flow variational posterior on unconstrained coordinates. \\
\(\phi\) & Variational parameters of the normalizing flow. \\
\(\mathcal L_{\mathrm{NF}}(\phi)\) & Normalizing-flow evidence lower bound. \\
\(q_\lambda(w)\) & Full-rank Gaussian variational posterior on unconstrained coordinates. \\
\(\mathcal L_{\mathrm{FR}}(\lambda)\) & Full-rank Gaussian variational evidence lower bound. \\
\(\mathrm{KL}(q\|p)\) & Kullback--Leibler divergence from \(q\) to \(p\). \\
\(\mathbb E_q[\cdot]\) & Expectation with respect to distribution \(q\). \\

\bottomrule
\caption{Notation for the hard partial-order model, the differentiable relaxation, and Bayesian posterior approximation.}
\label{tab:notation_relaxed_bayes}
\end{longtable}

\paragraph{Conventions.}
The hard quantities \(F_t\), \(S_t\), and \(Q_t\) are defined with respect to the hard latent
partial order \(h\) or \(h_U\). Their relaxed counterparts are marked with tildes. We use
\(\succ\) throughout so that the notation matches the semantics ``must occur before.''
In particular, the frontier set \(F(h; y_{\ge t})\) consists of maximal remaining items.

\subsection{Poset dimension and the role of the embedding dimension}
\label{app:poset_dimension}

The embedding dimension \(d\) in the product-order representation is closely related to the classical order dimension of a finite poset. This connection is standard in the partial-order literature and is discussed in detail in prior work on Bayesian partial-order inference from Royal Acta data~\citep{nicholls2025royalacta}. We therefore do not treat order-dimension estimation as a main contribution of this paper. Instead, we use small fixed values of \(d\) as a structured inductive bias: \(d=4\) for the synthetic generator and \(d=3\) for the main real-data relaxed fits, with exact choices reported in Appendix~\ref{app:repro}.
\begin{figure}[h]
    \centering
    \begin{subfigure}[t]{0.46\textwidth}
        \centering
        \begin{tikzpicture}[x=1.cm,y=0.8cm,>=Latex,font=\small]
            % coordinate guides
            \foreach \x/\lab in {0.8/1,1.8/2,2.8/3,3.8/4}{
                \draw[gray!20] (\x,0) -- (\x,3.5);
                \node[below,gray!70] at (\x,0) {\lab};
            }
            \node[below,gray!70] at (2.3,-0.45) {coordinate \(d\)};

            % comparable pair z > x
            \coordinate (z1) at (0.8,3.0);
            \coordinate (z2) at (1.8,2.7);
            \coordinate (z3) at (2.8,2.3);
            \coordinate (z4) at (3.8,2.0);

            \coordinate (x1) at (0.8,1.4);
            \coordinate (x2) at (1.8,1.6);
            \coordinate (x3) at (2.8,1.2);
            \coordinate (x4) at (3.8,0.9);

            % incomparable pair w
            \coordinate (w1) at (0.8,2.2);
            \coordinate (w2) at (1.8,1.2);
            \coordinate (w3) at (2.8,2.1);
            \coordinate (w4) at (3.8,1.0);

            % paths
            \draw[blue!70!black, line width=1.2pt] (z1)--(z2)--(z3)--(z4);
            \draw[red!75!black, line width=1.2pt]  (x1)--(x2)--(x3)--(x4);
            \draw[black!55, dashed, line width=1.0pt] (w1)--(w2)--(w3)--(w4);

            % points
            \foreach \p in {z1,z2,z3,z4}{\fill[blue!70!black] (\p) circle (1.7pt);}
            \foreach \p in {x1,x2,x3,x4}{\fill[red!75!black] (\p) circle (1.7pt);}
            \foreach \p in {w1,w2,w3,w4}{\fill[black!55] (\p) circle (1.4pt);}

            % labels
            \node[blue!70!black, above left=1pt] at (z1) {\(z\)};
            \node[red!75!black, below left=1pt] at (x1) {\(x\)};
            \node[black!60, above right=1pt] at (w3) {\(w\)};

            % short annotations
            \node[font=\scriptsize, blue!70!black] at (4.35,2.25) {\(z \succ_U x\)};
            \node[font=\scriptsize, black!60] at (4.25,1.45) {incomparable};
        \end{tikzpicture}
        \caption{Embedding view.}
        \label{fig:latent_paths_partial_order_clean}
    \end{subfigure}%
    \hfill
    \begin{subfigure}[t]{0.46\textwidth}
        \centering
        \begin{tikzpicture}[x=1.15cm,y=0.85cm,>=Latex,font=\small]
            % coordinate guides
            \foreach \x/\lab in {0.8/1,1.8/2,2.8/3,3.8/4}{
                \draw[gray!20] (\x,0) -- (\x,3.5);
                \node[below,gray!70] at (\x,0) {\lab};
            }
            \node[below,gray!70] at (2.3,-0.45) {coordinate \(d\)};

            % paths
            \coordinate (z1) at (0.8,3.0);
            \coordinate (z2) at (1.8,2.7);
            \coordinate (z3) at (2.8,2.3);
            \coordinate (z4) at (3.8,2.0);

            \coordinate (x1) at (0.8,1.4);
            \coordinate (x2) at (1.8,1.6);
            \coordinate (x3) at (2.8,1.2);
            \coordinate (x4) at (3.8,0.9);

            \draw[blue!70!black, line width=1.2pt] (z1)--(z2)--(z3)--(z4);
            \draw[red!75!black, line width=1.2pt]  (x1)--(x2)--(x3)--(x4);

            \foreach \p in {z1,z2,z3,z4}{\fill[blue!70!black] (\p) circle (1.7pt);}
            \foreach \p in {x1,x2,x3,x4}{\fill[red!75!black] (\p) circle (1.7pt);}

            \node[blue!70!black, above left=1pt] at (z1) {\(z\)};
            \node[red!75!black, below left=1pt] at (x1) {\(x\)};

            % gap arrows
            \draw[gray!70, <->, line width=0.8pt] (0.8,1.4) -- (0.8,3.0);
            \draw[gray!70, <->, line width=0.8pt] (1.8,1.6) -- (1.8,2.7);
            \draw[orange!90!black, <->, line width=1.1pt] (2.8,1.2) -- (2.8,2.3);
            \draw[gray!70, <->, line width=0.8pt] (3.8,0.9) -- (3.8,2.0);

            \node[font=\scriptsize, gray!80] at (0.55,2.2) {\(\Delta_1\)};
            \node[font=\scriptsize, gray!80] at (1.55,2.15) {\(\Delta_2\)};
            \node[font=\scriptsize, orange!90!black] at (2.45,1.8) {\(\min_d \Delta_d\)};
            \node[font=\scriptsize, gray!80] at (4.05,1.65) {\(\Delta_4\)};
        \end{tikzpicture}
        \caption{Gap view.}
        \label{fig:gap_view_clean}
    \end{subfigure}
    \caption{Latent product-order embedding. Left: comparable and incomparable pairs. Right: \(z \succ_U x\) is determined by the minimum coordinatewise gap.}
    \label{fig:latent_embedding_and_gap}
\end{figure}
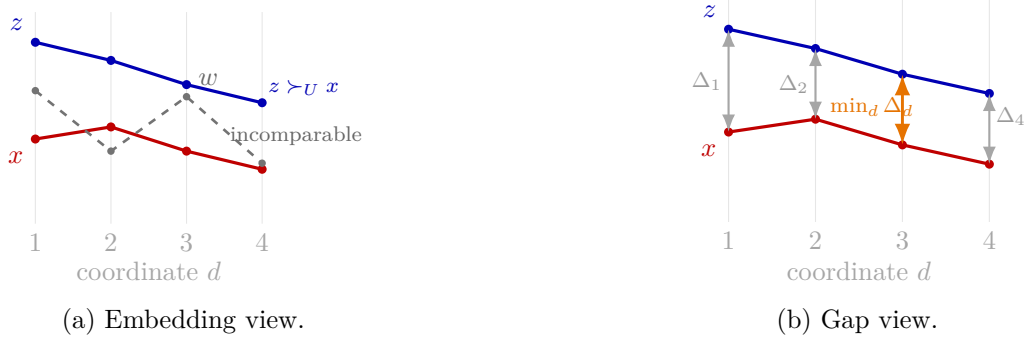

The product-order model used in this paper gives a geometric version of the same idea. If each item \(x\) is assigned an embedding \(u_x\in\mathbb R^d\), then
\[
x\succ_U y
\quad\Longleftrightarrow\quad
u_{x,k}>u_{y,k}\quad\text{for all }k=1,\ldots,d.
\]
With no coordinate ties, each coordinate induces a linear order by sorting the items along that coordinate. The product order is therefore the intersection of these \(d\) coordinate-wise linear orders. Thus, fixing \(d\) restricts the model to posets that can be represented, or well approximated, by a realizer of size at most \(d\). See Figure~\ref{fig:latent_embedding_and_gap}

The dimension of an arbitrary poset can be large. For \(m\ge 4\), Hiraguchi's bound gives
\[
\dim(h)\le \left\lfloor \frac{m}{2}\right\rfloor,
\]
and this worst-case order is attained by the standard-example, or crown-like, family~\citep{Hiraguchi51,bogart1973maximal}. Computing \(\dim(h)\) exactly is NP-hard in general~\citep{yannakakis1982complexity}. Consequently, choosing \(d=\lfloor m/2\rfloor\) would give a conservative worst-case representation, but it would remove the intended low-dimensional regularization and substantially increase the cost of continuous inference.

\paragraph{Use in this paper.}
Our goal is not to recover the exact order dimension of the unknown latent poset. Rather, \(d\) is a modeling choice controlling the expressiveness of the product-order prior. Small \(d\) gives a parsimonious low-dimensional hierarchy and can underfit high-dimension or highly entangled posets; large \(d\) gives a richer class but weakens regularization and increases computational cost. We therefore use moderate fixed values of \(d\), following the product-order modeling strategy developed in prior work, and focus on whether the differentiable relaxation makes Bayesian partial-order inference practical under this structured geometry.

\subsection{Illustration of notation and preliminaries}
\label{app:preliminaries}

\begin{figure}[h!]
    \centering
    \includegraphics[width=0.92\textwidth]{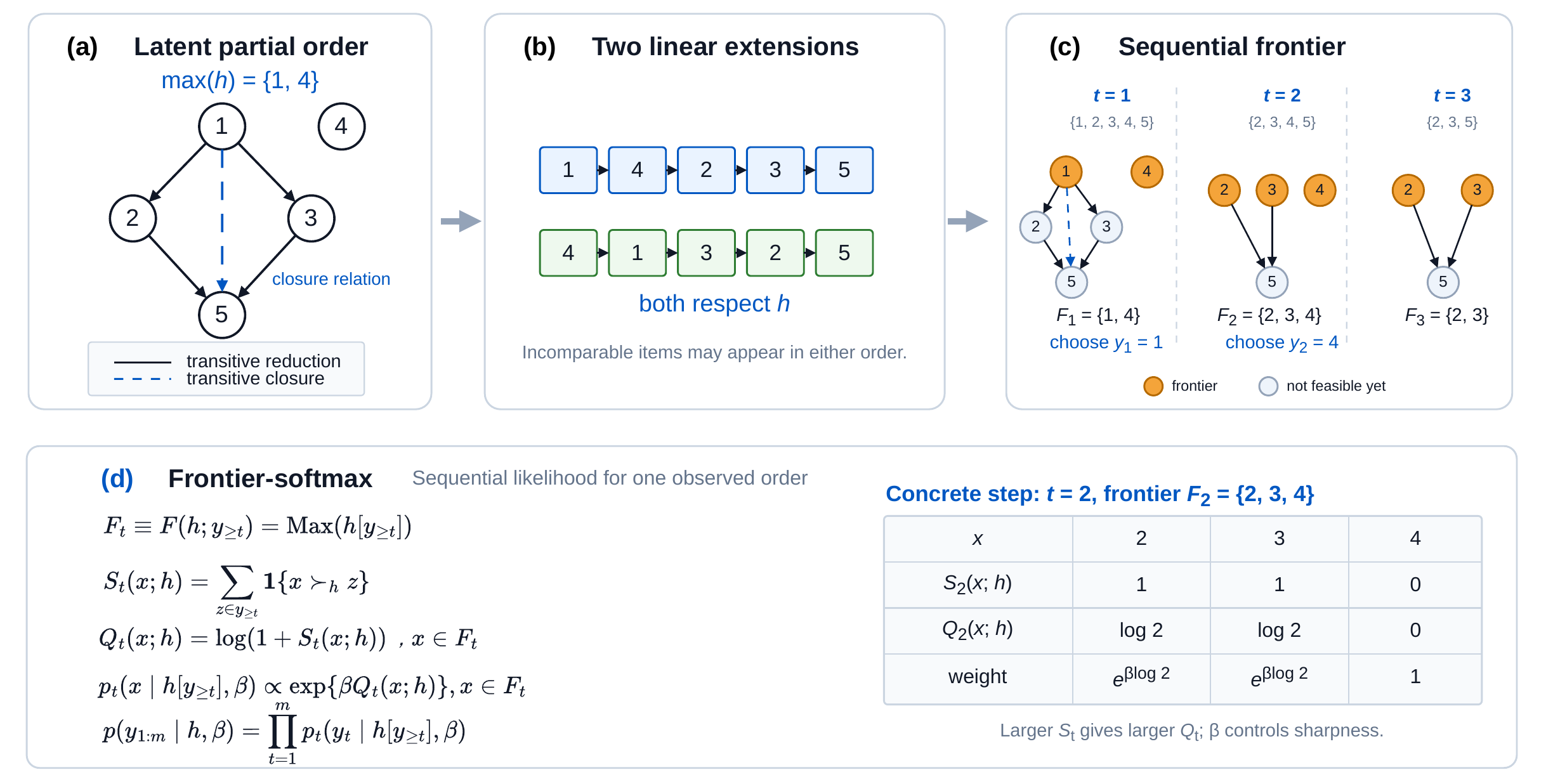}
\caption{
\textbf{From a latent partial order to the frontier-softmax likelihood.}
(a) The latent partial order and its current max-set. 
(b) Two valid linear extensions consistent with the same partial order. 
(c) Sequential construction of a ranking, where each item is chosen from the current frontier. 
(d) The frontier-softmax likelihood, illustrated by the step-\(t=2\) utility.
}
\label{fig:po_old}
\end{figure}

\subsection{Relaxed Partial Order Model Details}

\label{sec:relaxed_po_details}
\begin{figure}[h]
    \centering
    \includegraphics[width=0.9\textwidth]{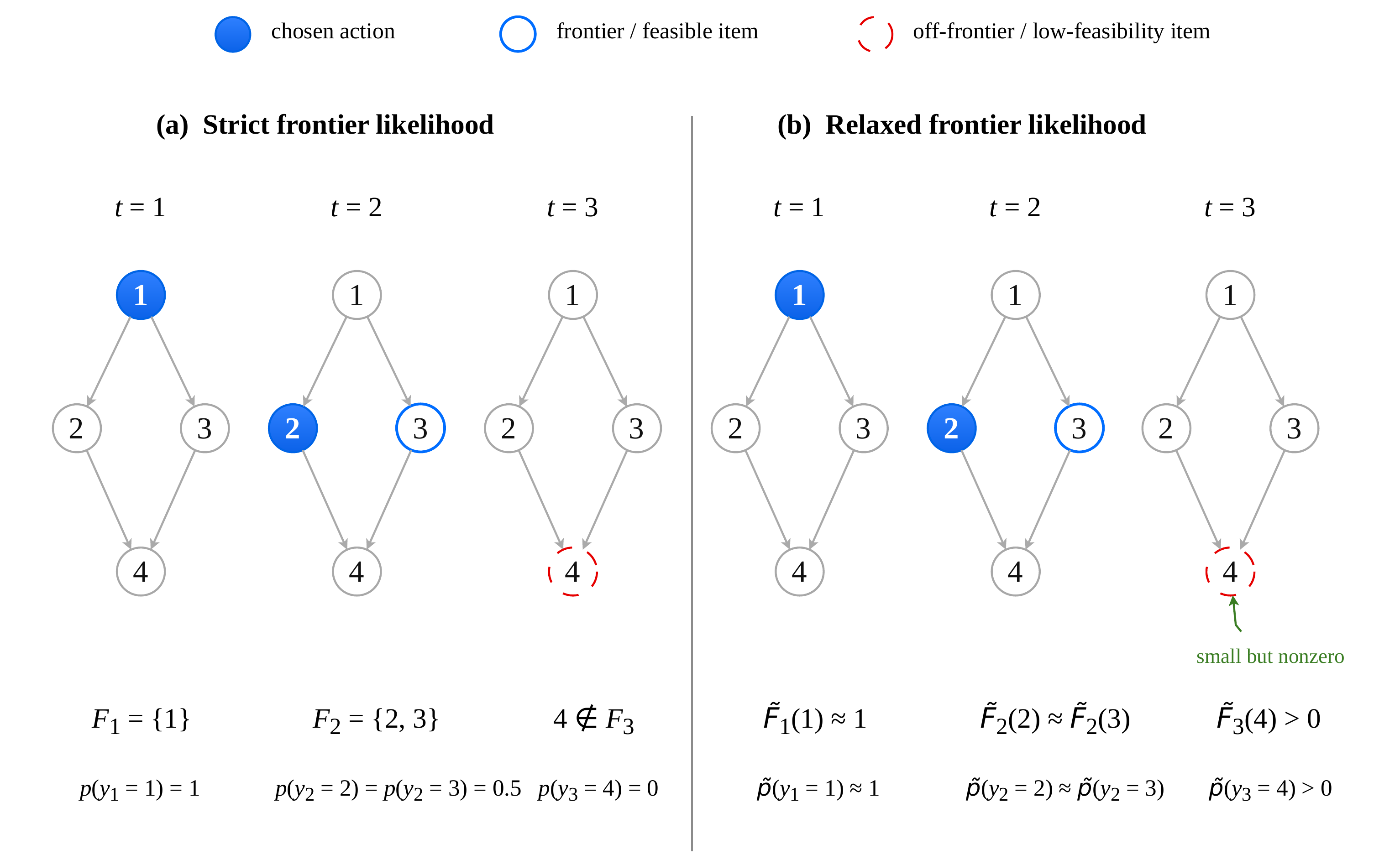}
    \caption{
    Strict versus relaxed frontier likelihood on a four-node partial order. 
    In the hard model, probability mass is assigned only to items on the current frontier; an off-frontier action such as choosing item \(4\) before its predecessors receives zero probability. 
    In the relaxed model, the soft frontier weight \(\widetilde F_t\) strongly down-weights such actions but does not make their probability exactly zero. 
    }
    \label{fig:frontier_likelihood}
\end{figure}

\begin{figure}[h]
  \centering
  \includegraphics[width=\linewidth]{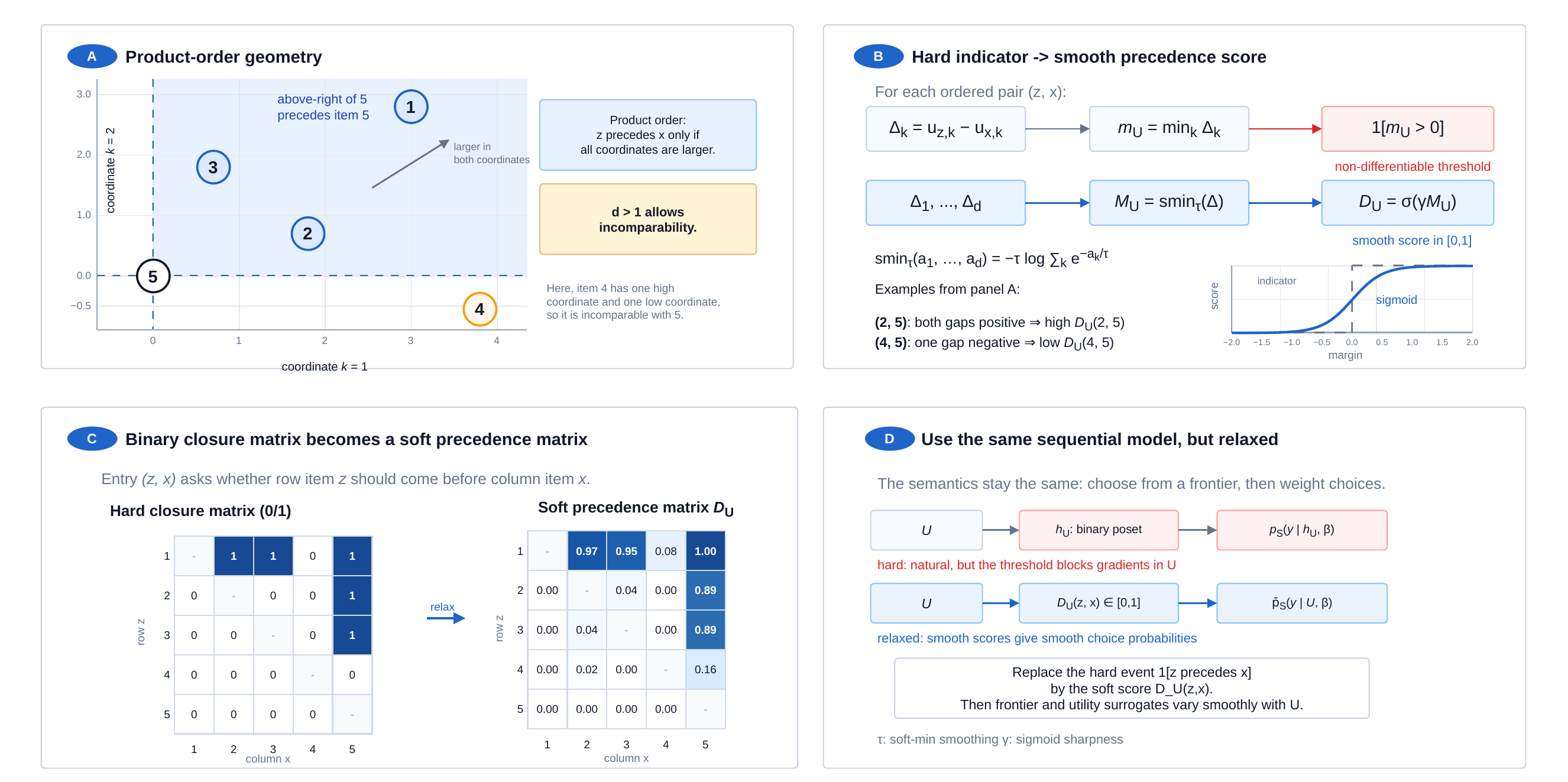}
  \caption{\textbf{Differentiable relaxation of the latent product order.}
  Items have embeddings $u_x\in\mathbb{R}^d$. Coordinatewise dominance induces the hard relation
  $z\succ_U x$ (with $d>1$ allowing incomparability). We replace the hard min-margin test with a
  soft-min and the binary indicator with a sigmoid, yielding the smooth precedence score
  $D_U(z,x)\in[0,1]$. The resulting soft precedence matrix replaces the binary closure and enables
  relaxed frontier/utility computations with a likelihood that is smooth in $U$.}
  \label{fig:soft_product_order}
\end{figure}

\section{Auxiliary proofs for the relaxed partial-order model}
\label{app:proofs}

This appendix collects the proofs of the theoretical guarantees stated in Section~\ref{subsec:relaxation_theory}. We begin with two auxiliary lemmas on the hard product order and the nonemptiness of finite frontiers, then prove the soft-min approximation, soft transitivity, frontier recovery, and likelihood convergence in turn. The final subsection gives a finite-temperature strengthening of Theorem~\ref{thm:likelihood_convergence}.

\subsection{Auxiliary lemmas}
\label{app:proof_preliminaries}

We first record two elementary facts used repeatedly below.

\begin{lemma}[The hard product order is a strict partial order]
\label{lem:hard_order_poset}
The relation \(\succ_U\) defined by
\[
z \succ_U x
\quad\Longleftrightarrow\quad
\Delta_k(z,x)>0
\quad\text{for all }k=1,\dots,d
\]
is irreflexive and transitive.
\end{lemma}

\begin{proof}
Irreflexivity is immediate, since \(\Delta_k(x,x)=0\) for every \(k\), so \(x\not\succ_U x\).

For transitivity, suppose \(z\succ_U y\) and \(y\succ_U x\). Then for every coordinate \(k\),
\[
\Delta_k(z,x)
=
\Delta_k(z,y)+\Delta_k(y,x)
>
0
\]
Hence \(z\succ_U x\). Therefore \(\succ_U\) is a strict partial order.
\end{proof}
\begin{lemma}[Finite frontiers are nonempty]
\label{lem:finite_frontier_nonempty}
Let \(S\in\mathcal{B}_{\mathcal M}\), and let \(\succ_h\) be a strict partial order on \(S\).
If \(R\subseteq S\) is non-empty, then the set of maximal elements
\[
\max(R,\succ_h)
:=
\{x\in R:\nexists\, z\in R \text{ such that } z\succ_h x\}
\]
is nonempty.
\end{lemma}

\begin{proof}
Assume for contradiction that \(\max(R,\succ_h)=\varnothing\). Then every \(x_1\in R\)
has a predecessor \(x_2\in R\) such that \(x_2\succ_h x_1\). Repeating inductively yields
a sequence \(x_1,x_2,x_3,\dots\) of distinct elements of \(R\) such that
\(x_{j+1}\succ_h x_j\) for all \(j\). Since \(R\) is finite, some state must repeat,
producing a cycle. By transitivity, this implies \(x\succ_h x\) for some \(x\in R\),
contradicting irreflexivity of a strict partial order. Hence
\(\max(R,\succ_h)\neq\varnothing\).
\end{proof}
\subsection{Proof of Proposition~\ref{prop:softmin_properties}}
\label{app:proof_softmin}

\begin{proof}[Proof of Proposition~\ref{prop:softmin_properties}]
Let
\[
m:=\min_{k=1,\dots,d} a_k.
\]
Then
\[
\sum_{k=1}^d e^{-a_k/\tau}
=
e^{-m/\tau}\sum_{k=1}^d e^{-(a_k-m)/\tau}.
\]
Since \(a_k-m\ge 0\) for all \(k\),
\[
1
\le
\sum_{k=1}^d e^{-(a_k-m)/\tau}
\le
d.
\]
Taking logarithms and multiplying by \(-\tau\) yields
\[
m-\tau\log d
\le
-\tau\log\sum_{k=1}^d e^{-a_k/\tau}
\le
m,
\]
which proves \eqref{eq:softmin_uniform_bounds} and then
% Next, let
% \[
% S(a):=\sum_{j=1}^d e^{-a_j/\tau}.
% \]
% Then
% \[
% \smin_\tau(a)=-\tau\log S(a),
% \]
% and
% \[
% \frac{\partial}{\partial a_k}\smin_\tau(a)
% =
% -\tau \frac{1}{S(a)}\frac{\partial S(a)}{\partial a_k}
% =
% -\tau \frac{1}{S(a)}\left(-\frac{1}{\tau}e^{-a_k/\tau}\right)
% =
% \frac{e^{-a_k/\tau}}{\sum_{j=1}^d e^{-a_j/\tau}},
% \]
% .
\(\smin_\tau(a)\to \min_k a_k\) as \(\tau\downarrow 0\) follows since \(\tau\log d\to 0\).
\end{proof}

\begin{lemma}[Margin approximation]
\label{lem:margin_approx_appendix}
For every pair \(z,x\in S\),
\begin{equation}
|M_U(z,x)-m_U(z,x)|\le \tau\log d.
\label{eq:margin_approx_appendix}
\end{equation}
\end{lemma}

\begin{proof}
Apply Proposition~\ref{prop:softmin_properties} to the vector
\[
a=\bigl(\Delta_1(z,x),\dots,\Delta_d(z,x)\bigr).
\]
Since
\[
m_U(z,x)=\min_k \Delta_k(z,x)
\qquad\text{and}\qquad
M_U(z,x)=\smin_\tau\!\bigl(\Delta_1(z,x),\dots,\Delta_d(z,x)\bigr),
\]
the bound follows immediately.
\end{proof}

\subsection{Proof of Theorem~\ref{thm:soft_transitivity}}
\label{app:proof_soft_transitivity}

\begin{lemma}[Superadditivity of the soft minimum]
\label{lem:superadditivity_appendix}
For any \(a,b\in\R^d\) and any \(\tau>0\),
\begin{equation}
\smin_\tau(a+b)\ge \smin_\tau(a)+\smin_\tau(b),
\label{eq:softmin_superadditivity_appendix}
\end{equation}
where \(a+b\) is taken coordinatewise.
\end{lemma}

\begin{proof}
Let
\[
u_k:=e^{-a_k/\tau},
\qquad
v_k:=e^{-b_k/\tau},
\]
so that
\[
\sum_{k=1}^d e^{-(a_k+b_k)/\tau}
=
\sum_{k=1}^d u_kv_k.
\]
Since \(u_k,v_k\ge 0\),
\[
\sum_{k=1}^d u_kv_k
\le
\left(\sum_{k=1}^d u_k\right)\left(\sum_{k=1}^d v_k\right)
=
\left(\sum_{k=1}^d e^{-a_k/\tau}\right)
\left(\sum_{k=1}^d e^{-b_k/\tau}\right).
\]
Applying \(-\tau\log(\cdot)\) to both sides gives
\[
\smin_\tau(a+b)
=
-\tau\log\sum_{k=1}^d e^{-(a_k+b_k)/\tau}
\ge
-\tau\log\sum_{k=1}^d e^{-a_k/\tau}
-
\tau\log\sum_{k=1}^d e^{-b_k/\tau},
\]
which proves \eqref{eq:softmin_superadditivity_appendix}.
\end{proof}

\begin{proof}[Proof of Theorem~\ref{thm:soft_transitivity}]
For any \(x,y,z\in S\), the coordinate-wise differences satisfy
\[
\Delta_k(z,x)=\Delta_k(z,y)+\Delta_k(y,x),
\qquad
k=1,\dots,d.
\]
Hence
\[
\bigl(\Delta_1(z,x),\dots,\Delta_d(z,x)\bigr)
=
\bigl(\Delta_1(z,y),\dots,\Delta_d(z,y)\bigr)
+
\bigl(\Delta_1(y,x),\dots,\Delta_d(y,x)\bigr),
\]
with addition taken coordinatewise. Applying Lemma~\ref{lem:superadditivity_appendix},
\begin{equation*}
M_U(z,x)
=
\smin_\tau\!\bigl(\Delta(z,x)\bigr)
\ge
\smin_\tau\!\bigl(\Delta(z,y)\bigr)
+
\smin_\tau\!\bigl(\Delta(y,x)\bigr)
=
M_U(z,y)+M_U(y,x),\label{}
\end{equation*}
which proves \eqref{eq:soft_transitivity_margin_main}.

Since
\[
D_U(a,b)=\sigma\!\bigl(\gamma M_U(a,b)\bigr)
\qquad\text{and}\qquad
\logit(\sigma(t))=t,
\]
we have
\[
\logit D_U(a,b)=\gamma M_U(a,b).
\]
Multiplying \eqref{eq:soft_transitivity_margin_main} by \(\gamma>0\) gives
\[
\logit D_U(z,x)\ge \logit D_U(z,y)+\logit D_U(y,x),
\]
which proves \eqref{eq:soft_transitivity_logit_main}.
\end{proof}

\subsection{Proof of Theorem~\ref{thm:frontier_quantitative}}
\label{app:proof_frontier}

We first prove an auxiliary result supporting Theorem~\ref{thm:frontier_quantitative}.
For each step \(1\le t\le m\), define
\begin{equation}
\delta_t(U):=
\min_{\substack{x,z\in \yt\\x\neq z}}
|m_U(z,x)|.
\label{eq:local_separation_margin}
\end{equation}
This is the magnitude of the margin for the closest pair at step $t$.

\begin{proposition}[Quantitative frontier bounds]
\label{prop:frontier_quantitative_bounds}
Assume Assumption~\ref{ass:margin_separation}. Fix a step \(t\) and an item \(x\in \yt\).

If \(x\in F(h_U;\yt)\), then
\begin{equation}
\widetilde F_t(x;U)
\ge
\prod_{z\in \yt\setminus\{x\}}
\Bigl(1-\sigma\!\bigl(-\gamma|m_U(z,x)|\bigr)\Bigr)
\ge
\Bigl(1-\sigma(-\gamma\delta_t(U))\Bigr)^{|\yt|-1}.
\label{eq:frontier_correctness_bound_appendix}
\end{equation}

If \(x\notin F(h_U;\yt)\), then for every \(\tau>0\),
\begin{equation}
\widetilde F_t(x;U)
\le
1-\sigma\!\bigl(\gamma(\delta_t(U)-\tau\log d)\bigr).
\label{eq:frontier_completeness_bound_appendix}
\end{equation}
In particular, whenever \(\tau<\delta_t(U)/\log d\),
\begin{equation}
\widetilde F_t(x;U)\to 0
\qquad\text{as }\gamma\to\infty.
\label{eq:frontier_completeness_limit_appendix}
\end{equation}
\end{proposition}

\begin{proof}
Fix a step \(t\) and an item \(x\in \yt\).

\paragraph{Case 1: \(x\in F(h_U;\yt)\).}
By definition of the hard frontier, for every \(z\in \yt\setminus\{x\}\) we have \(z\not\succ_U x\). Under Assumption~\ref{ass:margin_separation}, this implies
\[
m_U(z,x)<0.
\]
Hence \(|m_U(z,x)|=-m_U(z,x)\), and Lemma~\ref{lem:margin_approx_appendix} gives
\[
M_U(z,x)\le m_U(z,x)=-|m_U(z,x)|.
\]
Since the sigmoid is monotone increasing,
\[
D_U(z,x)=\sigma\!\bigl(\gamma M_U(z,x)\bigr)
\le
\sigma\!\bigl(-\gamma|m_U(z,x)|\bigr).
\]
Therefore
\[
1-D_U(z,x)
\ge
1-\sigma\!\bigl(-\gamma|m_U(z,x)|\bigr).
\]
Multiplying over all \(z\in \yt\setminus\{x\}\) yields
\[
\widetilde F_t(x;U)
=
\prod_{z\in \yt\setminus\{x\}}
\bigl(1-D_U(z,x)\bigr)
\ge
\prod_{z\in \yt\setminus\{x\}}
\Bigl(1-\sigma\!\bigl(-\gamma|m_U(z,x)|\bigr)\Bigr).
\]
Since \(|m_U(z,x)|\ge\delta_t(U)\) by definition of \(\delta_t(U)\), factors are at least \(1-\sigma(-\gamma\delta_t(U))\), proving \eqref{eq:frontier_correctness_bound_appendix}.

\paragraph{Case 2: \(x\notin F(h_U;\yt)\).}
Then there exists some \(z^\star\in \yt\setminus\{x\}\) such that \(z^\star\succ_U x\), equivalently \(m_U(z^\star,x)>0\). By definition of \(\delta_t(U)\),
\[
m_U(z^\star,x)\ge \delta_t(U).
\]
Lemma~\ref{lem:margin_approx_appendix} gives
\[
M_U(z^\star,x)
\ge
m_U(z^\star,x)-\tau\log d
\ge
\delta_t(U)-\tau\log d.
\]
Applying monotonicity of \(\sigma\),
\[
D_U(z^\star,x)
=
\sigma\!\bigl(\gamma M_U(z^\star,x)\bigr)
\ge
\sigma\!\bigl(\gamma(\delta_t(U)-\tau\log d)\bigr).
\]
Since every factor in the product defining \(\widetilde F_t(x;U)\) is at most \(1\),
\[
\widetilde F_t(x;U)
\le
1-D_U(z^\star,x)
\le
1-\sigma\!\bigl(\gamma(\delta_t(U)-\tau\log d)\bigr),
\]
which proves \eqref{eq:frontier_completeness_bound_appendix}. The limit \eqref{eq:frontier_completeness_limit_appendix} follows immediately when \(\tau<\delta_t(U)/\log d\).
\end{proof}

We now prove Theorem~\ref{thm:frontier_quantitative}.

\begin{proof}[Proof of Theorem~\ref{thm:frontier_quantitative}]
Fix \(t\) and \(x\in \yt\).

If \(x\in F(h_U;\yt)\), then Proposition~\ref{prop:frontier_quantitative_bounds} gives
\[
\widetilde F_t(x;U)
\ge
\Bigl(1-\sigma(-\gamma\delta_t(U))\Bigr)^{|\yt|-1}.
\]
Because \(\delta_t(U)>0\), the right-hand side converges to \(1\) as \(\gamma\to\infty\). Since \(\widetilde F_t(x;U)\le 1\), the squeeze theorem gives
\[
\lim_{\gamma\to\infty}\widetilde F_t(x;U)=1,
\]
and therefore also
\[
\lim_{\gamma\to\infty}\lim_{\tau\downarrow 0}\widetilde F_t(x;U)=1.
\]

If \(x\notin F(h_U;\yt)\), choose any \(\tau_0\in(0,\delta_t(U)/\log d)\). For all \(0<\tau\le\tau_0\), Proposition~\ref{prop:frontier_quantitative_bounds} gives
\[
\widetilde F_t(x;U)
\le
1-\sigma\!\bigl(\gamma(\delta_t(U)-\tau\log d)\bigr),
\]
and since \(\delta_t(U)-\tau\log d>0\), the right-hand side tends to \(0\) as \(\gamma\to\infty\). Hence
\[
\lim_{\gamma\to\infty}\lim_{\tau\downarrow 0}\widetilde F_t(x;U)=0.
\]

Combining the two cases proves
\[
\lim_{\gamma\to\infty}\lim_{\tau\downarrow 0}
\widetilde F_t(x;U)
=
\1\!\left[x\in F(h_U;\yt)\right],
\qedhere
\]
which is the result claimed in Theorem~\ref{thm:frontier_quantitative}.
\end{proof}

\subsection{Proof of %Equation~\ref{eq:soft_successor_terms} and 
Theorem~\ref{thm:likelihood_convergence}}
\label{app:proof_likelihood}

\iffalse
\begin{proof}[Proof of Equation~\ref{eq:soft_successor_terms}]
Fix \(\tau>0\), \(\gamma>0\), \(\beta>0\). For each pair \((z,x)\), the map
\[
U\mapsto \Delta_k(z,x)
\]
is affine, hence \(C^\infty\). The soft minimum \(\smin_\tau\) is \(C^\infty\) by Proposition~\ref{prop:softmin_properties}, and the sigmoid \(\sigma\) is \(C^\infty\), so
\[
U\mapsto D_U(z,x)
\]
is \(C^\infty\). Therefore the finite product
\[
U\mapsto \widetilde F_t(x;U)
\]
and the finite sum
\[
U\mapsto \widetilde S_t(x;U)
\]
are \(C^\infty\), hence so is
\[
U\mapsto \widetilde Q_t(x;U)=\log(1+\widetilde S_t(x;U)).
\]
Since exponentials, finite sums, and quotients by strictly positive denominators preserve smoothness, the step likelihood
\[
U\mapsto \widetilde p(y_t\mid \yt,U,\beta,\gamma)
\]
is \(C^\infty\). The trace likelihood is a finite product of step likelihoods, so it is also \(C^\infty\).
\end{proof}
\fi

We first show that the soft successor counts in \eqref{eq:soft_successor_terms} converge to the hard descendant counts in \eqref{eq:frontier_softmax_descendants}.

\begin{proposition}[Sharp-limit recovery of descendant counts]
\label{prop:descendant_limit_appendix}
Under Assumption~\ref{ass:margin_separation}, for every step \(t\) and every \(x\in \yt\),
\begin{equation}
\lim_{\gamma\to\infty}\lim_{\tau\downarrow 0}
\widetilde S_t(x;U)
=
S_t(x;h_U).
\label{eq:descendant_limit_appendix}
\end{equation}
\end{proposition}

\begin{proof}
Fix \(t\) and \(x\in \yt\). For each \(z\in \yt\setminus\{x\}\), Lemma~\ref{lem:margin_approx_appendix} gives
\[
M_U(x,z)\to m_U(x,z)
\qquad\text{as }\tau\downarrow 0.
\]
Under Assumption~\ref{ass:margin_separation}, \(m_U(x,z)\neq 0\). Therefore
\[
\lim_{\gamma\to\infty}\lim_{\tau\downarrow 0}
D_U(x,z)
=
\lim_{\gamma\to\infty}\sigma\!\bigl(\gamma m_U(x,z)\bigr)
=
\1[m_U(x,z)>0]
=
\1[x\succ_U z].
\]
Since the sum in \eqref{eq:soft_successor_terms} is finite,
\[
\lim_{\gamma\to\infty}\lim_{\tau\downarrow 0}
\widetilde S_t(x;U)
=
\sum_{z\in \yt\setminus\{x\}}\1[x\succ_U z]
=
S_t(x;h_U),
\]
which proves \eqref{eq:descendant_limit_appendix}.
\end{proof}

We now prove the convergence of the full soft likelihood in Equation~\eqref{eq:trace_likelihood_limit_main} of Theorem~\ref{thm:likelihood_convergence}.

\begin{proof}[Proof of Theorem~\ref{thm:likelihood_convergence}]
Fix a step \(t\). For each \(a\in \yt\), define the relaxed rational score
\begin{equation}
\widetilde r_t(a;U)
:=
\widetilde F_t(a;U)\exp\{\beta \, \widetilde Q_t(a;U)\},
\label{eq:appendix_relaxed_rational_score}
\end{equation}
and the corresponding hard rational score
\begin{equation}
r_t(a;h_U)
:=
\1[a\in F(h_U;\yt)]\exp\{\beta\, Q_t(a;h_U)\}.
\label{eq:appendix_hard_rational_score}
\end{equation}

By Theorem~\ref{thm:frontier_quantitative},
\[
\widetilde F_t(a;U)\to \1[a\in F(h_U;\yt)]
\]
in the sharp limit. By Proposition~\ref{prop:descendant_limit_appendix},
\[
\widetilde S_t(a;U)\to S_t(a;h_U).
\]
Since
\[
\widetilde Q_t(a;U)=\log\bigl(1+\widetilde S_t(a;U)\bigr)
\qquad\text{and}\qquad
Q_t(a;h_U)=\log\bigl(1+S_t(a;h_U)\bigr),
\]
continuity of \(\log(1+s)\) implies
\[
\widetilde Q_t(a;U)\to Q_t(a;h_U).
\]
Therefore
\[
\widetilde r_t(a;U)\to r_t(a;h_U)
\qquad
\text{for every }a\in \yt.
\]

Because \(\yt\) is finite,
\[
\sum_{a\in \yt}\widetilde r_t(a;U)
\to
\sum_{a\in \yt}r_t(a;h_U).
\]
By Lemma~\ref{lem:hard_order_poset}, \(h_U=(S,\succ_U)\) is a strict partial order, so its restriction to \(\yt\) is also a finite strict partial order. Hence by Lemma~\ref{lem:finite_frontier_nonempty},
\[
F(h_U;\yt)\neq\varnothing.
\]
Therefore
\[
\sum_{a\in \yt}r_t(a;h_U)
=
\sum_{a\in F(h_U;\yt)}\exp\{\beta\, Q_t(a;h_U)\}
>
0.
\]
It follows that
\[
\frac{\widetilde r_t(y_t;U)}
{\sum_{a\in \yt}\widetilde r_t(a;U)}
\to
\frac{r_t(y_t;h_U)}
{\sum_{a\in \yt}r_t(a;h_U)}.
\]
\[
\lim_{\gamma\to\infty}\lim_{\tau\downarrow 0}
\widetilde p(y_t\mid \yt,U,\beta,\gamma)
=
p(y_t\mid \yt,h_U,\beta),
\]
which proves \eqref{eq:step_likelihood_limit_main}.

Finally, the trace likelihood is a finite product of step likelihoods:
\[
\widetilde p(y\mid U,\beta,\gamma)
=
\prod_{t=1}^{T}
\widetilde p(y_t\mid \yt,U,\beta,\gamma).
\]
Applying \eqref{eq:step_likelihood_limit_main} stepwise and continuity of finite products gives
\[
\lim_{\gamma\to\infty}\lim_{\tau\downarrow 0}
\widetilde p(y\mid U,\beta,\gamma)
=
p(y\mid h_U,\beta),
\]
which proves \eqref{eq:trace_likelihood_limit_main} and completes the proof of Theorem~\ref{thm:likelihood_convergence}.
\end{proof}

\subsection{Effect of the relaxation temperature $\tau$}

The temperature $\tau$ controls the gap between the hard product-order margin
and its soft-min relaxation. For
\[
m_U(z,x)=\min_{k=1,\ldots,d}\{u_{z,k}-u_{x,k}\},
\qquad
M_U^\tau(z,x)
=
-\tau\log\sum_{k=1}^d
\exp\left\{-\frac{u_{z,k}-u_{x,k}}{\tau}\right\},
\]
the standard soft-min bound gives
\begin{equation}
m_U(z,x)-\tau\log d
\le
M_U^\tau(z,x)
\le
m_U(z,x).
\label{eq:tau-softmin-bound}
\end{equation}
Thus $\tau$ introduces at most $\tau\log d$ downward bias in each pairwise
margin, and $M_U^\tau(z,x)\to m_U(z,x)$ uniformly as $\tau\downarrow 0$.

Let
\[
D_U^\tau(z,x)=\sigma(\gamma M_U^\tau(z,x))
\]
be the soft precedence score. At step $t$, let $R_t$ be the remaining set and
define the local margin separation
\[
\delta_t(U)
=
\min_{\substack{a,b\in R_t\\a\neq b}}
|m_U(a,b)|.
\]
If $\tau<\delta_t(U)/\log d$, then the effective margin is positive. Defining
\[
\varepsilon_t(\tau,\gamma)
=
\exp\{-\gamma(\delta_t(U)-\tau\log d)\},
\]
we obtain the pairwise recovery bounds
\begin{equation}
z\succ_U x
\quad\Longrightarrow\quad
1-\varepsilon_t(\tau,\gamma)
\le
D_U^\tau(z,x)
\le
1,
\label{eq:tau-edge-bound}
\end{equation}
and
\begin{equation}
z\not\succ_U x
\quad\Longrightarrow\quad
0
\le
D_U^\tau(z,x)
\le
\varepsilon_t(\tau,\gamma).
\label{eq:tau-nonedge-bound}
\end{equation}
Hence the relaxation separates true precedences from non-precedences whenever
\[
\delta_t(U)-\tau\log d>0,
\qquad\text{equivalently}\qquad
\tau<\frac{\delta_t(U)}{\log d}.
\label{eq:tau-condition}
\]
More generally, to achieve pairwise error at most $\eta$, it is sufficient that
\begin{equation}
\tau
\le
\frac{\delta_t(U)-\gamma^{-1}\log(1/\eta)}{\log d}.
\label{eq:tau-target-error}
\end{equation}

This shows the statistical role of $\tau$: larger $\tau$ smooths the posterior
but reduces the effective margin available for graph recovery. Conversely,
smaller $\tau$ better approximates the hard model, but makes the objective more
sharply curved. Indeed, if
\[
p_k
=
\frac{\exp\{-\Delta_k/\tau\}}
{\sum_{\ell=1}^d\exp\{-\Delta_\ell/\tau\}},
\]
then
\[
\nabla_\Delta^2 M_U^\tau
=
-\frac{1}{\tau}
\left(\operatorname{diag}(p)-pp^\top\right),
\qquad
\|\nabla_\Delta^2 M_U^\tau\|_2=O(1/\tau).
\]
Thus $\tau$ trades off approximation accuracy and optimization stability:
too large a value makes the recovery bound weak or vacuous, while too small a
value can make gradient-based inference ill-conditioned.

\subsection{Proof of Theorem~\ref{thm:marginal_consistency}}
\label{app:marginal_consistency}

Prior beliefs about soft orders $D_U$ on a subset $S\in\mathcal{B}_\mathcal{M}$ of items $U=(u_x)_{x\in S}$ are unchanged by the absence of omitted items (items in $\mathcal{M}\setminus S$). If we construct the $D_U$ prior for orders over a choice set $S$ with $m$ items from the start (so $U\in \R^{m\times d}$), or construct the $D_U$ prior for orders over $\mathcal{M}$ (so $U\in \R^{M\times d}$) then throw out items in $\mathcal{M}\setminus S$ by dropping those rows and columns of $D_U$ then we get the same prior distribution on $D_U$.

{\bf Theorem~\ref{thm:marginal_consistency}} (Marginal Consistency) If $D_U\sim \pi_{\mathcal{M}}(\cdot)$ then $D_U[S,S]\sim \pi_S(\cdot)$.
\begin{proof} For $U\in \mathbb{R}^{M\times d}$ let $U[S,:]$ be the matrix we get by deleting rows of $U$ in $\mathcal{M}\setminus S$. Since $D_U[S,S]=D_{U[S,:]}$ and the rows $u_x$ of $U$ are independent, omitting rows give us a random matrix $U[S,:]$ with the same distribution it has if formed on $S$ from the start, and hence $D_{U[S,:]}\sim \pi_S(\cdot)$.
\end{proof}

\section{Model inference}
\subsection{Posterior Derivation}
\label{app:posterior_derivations}
This section gives the unconstrained posterior target, Jacobian corrections, variational objectives, and decoding procedure used by the continuous inference methods in Section~\ref{subsec:inference_methods}.

\subsubsection{Relaxed posterior target}
\label{app:relaxed_posterior_target}

Let
\[
\Theta=(Z,\rho,\beta,\gamma)
\]
denote the relaxed model parameters. Under the non-centred parameterization, the joint density is
\[
p(\mathcal D,\Theta)
=
p(\rho)p(\beta)p(\gamma)
\left[\prod_{x=1}^m \mathcal N(z_x;0,I_d)\right]
\prod_{n=1}^N
\widetilde p\!\left(y^{(n)}\mid U(Z,\rho),\beta,\gamma\right).
\]

For continuous inference, we work in unconstrained coordinates
\[
w=
\bigl(\operatorname{vec}(Z),\eta_\rho,\eta_\beta,\eta_\gamma\bigr)
\in\mathbb R^q,
\qquad
q=md+3,
\]
with transforms
\[
\rho=\sigma(\eta_\rho),
\qquad
\beta=e^{\eta_\beta},
\qquad
\gamma=e^{\eta_\gamma}.
\]
Let \(T(w)=\Theta\). The transformed joint density is
\[
\bar p(\mathcal D,w)
=
p(\mathcal D,T(w))\,|\det J_T(w)|,
\]
or equivalently
\[
\log \bar p(\mathcal D,w)
=
\log p(\mathcal D,T(w))
+
\log |\det J_T(w)|.
\]
For the above transforms,
\[
\log |\det J_T(w)|
=
\log\rho+\log(1-\rho)
+\eta_\beta+\eta_\gamma .
\]
If \(\beta\) is fixed in a particular experiment, the \(\eta_\beta\) coordinate, its prior, and its Jacobian term are omitted.

\subsection{Continuous MCMC}
\label{app:relaxed_mcmc}

\relaxedmcmc{} targets the transformed relaxed posterior
\[
\bar p(w\mid\mathcal D)
\propto
\bar p(\mathcal D,w)
\]
using \relaxedmcmc. This sampler targets the relaxed posterior directly and does not impose a variational family. In contrast, \hardmcmc{} targets the original discrete hard partial-order posterior and is used as a reference when long-run discrete inference is feasible.

\subsection{Stan full-rank Gaussian variational inference}
\label{app:fullrank_vi}

We also consider a simpler Gaussian variational approximation using Stan's full-rank automatic differentiation variational inference (ADVI). This serves as a baseline against the normalizing-flow posterior while keeping the target posterior model fixed.

\paragraph{Posterior target.}
The full-rank variational approximation is fit to the same relaxed local-poset posterior as the normalizing-flow method. In particular, the transformed joint density in unconstrained coordinates is
\begin{equation}
\log \bar p(\cD,w)
=
\log p\bigl(\cD,T(w)\bigr)
+
\log\left|\det J_T(w)\right|,
\label{eq:app_fullrank_transformed_joint}
\end{equation}
where
\[
w=
\bigl(\operatorname{vec}(Z),\eta_\rho,\eta_\beta,\eta_\gamma \bigr)\in\R^q
\]
is the unconstrained parameter vector and \(T(w)\) maps to the constrained parameters
\((Z,\rho,\beta,\gamma)\).

\paragraph{Variational family.}
Stan full-rank ADVI uses a single Gaussian approximation on the unconstrained coordinates:
\begin{equation}
q_\lambda(w)=\cN(w;\mu,\Sigma),
\label{eq:app_fullrank_q}
\end{equation}
where \(\mu\in\R^q\) is the variational mean and \(\Sigma\in\R^{q\times q}\) is a dense covariance matrix. This is more expressive than a mean-field Gaussian because it captures linear posterior dependence, but it is still limited to a single elliptical approximation.

\paragraph{ELBO objective.}
The variational parameters \(\lambda\) are obtained by maximizing
\begin{equation}
\mathcal L_{\mathrm{FR}}(\lambda)
=
\E_{q_\lambda(w)}
\Big[
\log \bar p(\cD,w)-\log q_\lambda(w)
\Big].
\label{eq:app_fullrank_elbo}
\end{equation}
In practice, Stan uses a Monte Carlo estimator of the ELBO,
\begin{equation}
\widehat{\mathcal L}_{\mathrm{FR}}(\lambda)
=
\frac{1}{S}
\sum_{s=1}^{S}
\left[
\log \bar p\bigl(\cD,w^{(s)}\bigr)
-
\log q_\lambda\bigl(w^{(s)}\bigr)
\right],
\qquad
w^{(s)}\sim q_\lambda,
\label{eq:app_fullrank_mc_elbo}
\end{equation}
and performs stochastic gradient optimization with respect to \(\lambda\).

\paragraph{Constrained parameters.}
As in the normalizing-flow formulation, the scalar parameters are represented in unconstrained form via
\begin{equation}
\rho=\sigma(\eta_\rho),\qquad
\beta=e^{\eta_\beta},\qquad
\gamma=e^{\eta_\gamma},\qquad
\label{eq:app_fullrank_transforms}
\end{equation}
with the appropriate Jacobian contribution included automatically in the transformed density \(\bar p(\cD,w)\).

\paragraph{Practical configuration.}
In our implementation, full-rank ADVI is run with the same relaxed likelihood and prior specification as the normalizing-flow method, so that the comparison isolates the effect of the variational family. The main tuning parameters are the maximum number of optimization iterations, the number of Monte Carlo samples used for gradient estimates, the number of Monte Carlo samples used for ELBO evaluation, the stopping tolerance on relative objective improvement, and the length of the adaptation phase.

\subsection{Normalizing-flow variational inference details}
\label{app:nf_posterior_derivations}

We provide the detailed construction of the normalizing-flow variational posterior used for the relaxed local partial-order model.

\paragraph{Variables and objective.}
We infer
\[
\Theta = (Z,\rho,\beta,\gamma),
\]
where \(Z \in \mathbb{R}^{m\times d}\) denotes the latent Gaussian coordinates,
\(\rho\) controls the induced covariance structure of \(U(Z,\rho)\),
\(\beta\) is the frontier-softmax sharpness,
\(\gamma\) is the soft-precedence sharpness.
The variational objective is the ELBO
\begin{equation}
\mathcal L(\phi)
=
\mathbb E_{q_\phi(\Theta)}
\Big[
\log p(\cD,\Theta)-\log q_\phi(\Theta)
\Big].
\label{eq:app_elbo_theta}
\end{equation}

\paragraph{Unconstrained coordinates.}
To handle positivity and unit-interval constraints, we work in unconstrained coordinates
\begin{equation}
w=
\bigl(\operatorname{vec}(Z),\eta_\rho,\eta_\beta,\eta_\gamma\bigr)\in\R^q,
\qquad q=md+3,
\end{equation}
with transforms
\begin{equation}
\rho=(1-\rho_{\mathrm{tol}})\sigma(\eta_\rho),\qquad
\beta=e^{\eta_\beta},\qquad
\gamma=e^{\eta_\gamma}
\quad \text{(if inferred)}.
\label{eq:app_scalar_transforms}
\end{equation}
Let \(T(w)=\Theta\) denote the overall transform from unconstrained to constrained parameters. The corresponding transformed joint density is
\begin{equation}
\log \bar p(\cD,w)
=
\log p\bigl(\cD,T(w)\bigr)
+
\log\left|\det J_T(w)\right|.
\label{eq:app_transformed_joint}
\end{equation}

\paragraph{Base variational family and neural spline flows.}
We begin from a diagonal Gaussian base distribution
\begin{equation}
z_0 = \mu + \sigma \odot \xi,
\qquad
\xi \sim \cN(0,I_q),
\label{eq:app_base_sample}
\end{equation}
where \(\mu,\sigma\in\R^q\) are variational parameters.
We then apply \(K\) neural spline flow layers:
\begin{equation}
z_k = T_k(z_{k-1}),
\qquad
k=1,\dots,K,
\label{eq:app_nsf_flow}
\end{equation}
where each \(T_k\) is an invertible rational-quadratic spline transformation.

We use an autoregressive parameterisation. For each coordinate,
\begin{equation}
z_{k,i}
=
S_{k,i}\!\left(
z_{k-1,i};
\theta_{k,i}(z_{k-1,<i})
\right),
\qquad
i=1,\dots,q,
\label{eq:app_nsf_coordinate}
\end{equation}
where \(S_{k,i}\) is a monotone rational-quadratic spline and
\(\theta_{k,i}\) is produced by an autoregressive neural network.
The spline parameters define positive bin widths, positive bin heights, and positive knot derivatives, which guarantee monotonicity and invertibility.

Since the autoregressive Jacobian is triangular, the log-determinant is
\begin{equation}
\log\left|
\det
\frac{\partial z_k}{\partial z_{k-1}}
\right|
=
\sum_{i=1}^q
\log
\left|
\frac{\partial z_{k,i}}{\partial z_{k-1,i}}
\right|.
\label{eq:app_nsf_jacobian}
\end{equation}
Therefore, after \(K\) flow layers,
\begin{equation}
\log q_\phi(z_K)
=
\log q_0(z_0)
-
\sum_{k=1}^K
\sum_{i=1}^q
\log
\left|
\frac{\partial z_{k,i}}{\partial z_{k-1,i}}
\right|.
\label{eq:app_nsf_logq}
\end{equation}

We identify \(w=z_K\) as the final unconstrained sample.

\paragraph{Soft partial-order construction.}
Given \(U=U(Z,\rho)\), define the coordinate differences
\begin{equation}
\Delta_{ijr}=U_{ir}-U_{jr}.
\label{eq:app_delta}
\end{equation}
The soft-minimum margin is
\begin{equation}
M_{ij}(U;\tau)
=
-\tau \log \sum_{r=1}^d
\exp\!\left(-\frac{\Delta_{ijr}}{\tau}\right),
\label{eq:app_softmin}
\end{equation}
and the soft precedence score is
\begin{equation}
D_{ij}(U;\gamma,\tau)
=
\sigma\!\bigl(\gamma M_{ij}(U;\tau)\bigr),
\qquad
D_{ii}=0.
\label{eq:app_softprecedence}
\end{equation}

\paragraph{Trace likelihood.}
At step \(t\), let \(R_t\) be the remaining set and let
\(p_{\mathrm{rat}}(y_t\mid R_t,\Theta)\)
denote the relaxed frontier-based choice probability.
We then use the form
\begin{equation}
p(y_t\mid R_t,\Theta)
=
p_{\mathrm{rat}}(y_t\mid R_t,\Theta).
\label{eq:app_noisy_choice}
\end{equation}
The trace log-likelihood is therefore
\begin{equation}
\log p(\text{trace}\mid\Theta)
=
\sum_t \log p(y_t\mid R_t,\Theta).
\label{eq:app_trace_loglik}
\end{equation}
Across the full dataset \(\cD=\{y^{(n)}\}_{n=1}^N\), the joint log density takes the form
\begin{equation}
\log p(\cD,Z,\rho,\beta,\gamma)
=
\sum_{n=1}^{N}
\log p\bigl(y^{(n)} \mid U(Z,\rho),\beta,\gamma\bigr)
+
\log p(\rho)
+
\log p(\beta)
+
\log p(\gamma)
\sum_{x=1}^{m}
\log \cN\bigl(z_x;0,I_d\bigr).
\label{eq:app_logjoint_z}
\end{equation}

\paragraph{Monte Carlo ELBO estimator.}
Using \(S\) Monte Carlo samples \(w^{(s)}\sim q_\phi(w)\), we optimize
\begin{equation}
\widehat{\mathcal L}(\phi)
=
\frac{1}{S}\sum_{s=1}^{S}
\left[
\log \bar p(\cD,w^{(s)})
-
\log q_\phi(w^{(s)})
\right].
\label{eq:app_mc_elbo}
\end{equation}
This estimator is fully differentiable under the reparameterization induced by the Gaussian base distribution and the flow transforms.

\paragraph{Optimization details.}
We optimize the ELBO using \texttt{Adam} with gradient clipping and optional antithetic sampling.
The learning rate is scheduled as
\begin{equation}
\mathrm{lr}(t)
=
\mathrm{lr}_0\cdot \mathrm{warmup}(t)\cdot
\exp\!\bigl(\log(\alpha)\cdot \mathrm{progress}(t)\bigr),
\label{eq:app_lr_schedule}
\end{equation}
where \(\alpha\) is the final learning-rate factor.
To improve robustness, we use multiple random restarts, retain the best evaluation-ELBO checkpoint within each restart, apply early stopping, and keep the globally best restart for posterior summaries and decoded structures.

\subsection{Computational complexity and memory scaling}
\label{app:scaling_analysis}

We summarize the dominant computational and memory costs of the relaxed model and the
continuous inference procedures. Note that the costs for likelihood evaluation and storage for \hardmcmc{} are qualitatively the same as for \relaxedmcmc{}. The speedup we see in experiments is due to the more efficient MCMC methods available for differentiable targets. The speedup for the variational methods is because no likelihood evaluation is needed in simulation, once the variational approximation is constructed.

\paragraph{Per-draw relaxed likelihood cost.}
Fix one posterior draw \(\Theta=(Z,\rho,\beta,\gamma)\). Let \(M\) be the size of the global item
universe and let trace \(y=(y_1,\dots,y_T)\) have length \(T\). The relaxed likelihood is defined
through the soft-precedence scores \(D_U(z,x)\), the soft frontier weights \(\widetilde F_t(x;U)\),
the soft descendant counts \(\widetilde S_t(x;U)\), and the frontier-softmax step probabilities
\(\tilde p(y_t\mid y_{\ge t},U,\beta,\gamma)\); see Eqs.~\eqref{eq:soft_pairwise_precedence}. Given \(U=ZL_\rho^\top\), evaluating all pairwise soft-precedence scores requires
\[
O(M^2 d)
\]
work, since each ordered pair \((z,x)\) needs a soft minimum over \(d\) coordinates.
Storing the resulting matrix \(D_U\) costs
\[
O(M^2)
\]
memory.

A naive evaluation of one trace would recompute, at every step \(t\), the soft frontier
products and successor sums over the remaining set \(y_{\ge t}\), giving
\[
\sum_{t=1}^T O(|y_{\ge t}|^2)=O(T^3).
\]
However, this can be reduced to quadratic time by caching the current frontier scores and
successor counts. Define
\[
\phi_t(x)
:= \log \widetilde F_t(x;U)
= \sum_{z\in y_{\ge t}\setminus\{x\}} \log\!\bigl(1-D_U(z,x)\bigr),
\]
and
\[
\widetilde S_t(x;U)
:= \sum_{z\in y_{\ge t}\setminus\{x\}} D_U(x,z).
\]
After selecting item \(y_t\), for every remaining \(x\in y_{\ge t+1}\) we have
\[
\phi_{t+1}(x)=\phi_t(x)-\log\!\bigl(1-D_U(y_t,x)\bigr),
\qquad
\widetilde S_{t+1}(x)=\widetilde S_t(x)-D_U(x,y_t).
\]
Thus, after an \(O(T^2)\) initialization of \(\phi_1\) and \(\widetilde S_1\), each step only updates
the surviving items and normalizes the resulting frontier-softmax probabilities, costing
\(O(|y_{\ge t}|)\). Summing over \(t\) yields
\[
O(T^2)
\]
time per trace once \(D_U\) is available.

For a dataset \(D=\{y^{(n)}\}_{n=1}^N\), this gives the total per-draw likelihood cost
\[
C_{\mathrm{like}}
=
O\!\left(M^2 d + \sum_{n=1}^N T_n^2\right).
\]
In settings such as the Cloud benchmark where all traces are linear extensions of the same
global order, \(T_n=M\) and this simplifies to
\[
C_{\mathrm{like}} = O(M^2 d + N M^2).
\]

\paragraph{Memory footprint.}
The model state stores:
\begin{itemize}
    \item the embedding coordinates \(Z\in\mathbb R^{M\times d}\), costing \(O(Md)\),
    \item the soft-precedence matrix \(D_U\in[0,1]^{M\times M}\), costing \(O(M^2)\),
    \item per-trace work vectors such as \(\phi_t\), \(\widetilde S_t\), and frontier logits, costing \(O(M)\).
\end{itemize}
Hence the per-draw working memory is
\[
O(M^2 + Md),
\]
dominated by the \(M\times M\) soft-precedence matrix.

\paragraph{Inference-method complexity.}
The continuous parameter dimension is
\[
q = Md + 3,
\]
corresponding to \(\operatorname{vec}(Z)\) and the unconstrained scalar parameters
\((\eta_\rho,\eta_\beta,\eta_\gamma)\).

For \textbf{Relaxed-MCMC}, one gradient evaluation of the transformed log posterior
\(\log \bar p(D,w)\) costs \(O(C_{\mathrm{like}})\) up to a constant-factor reverse-mode autodiff
overhead. If \(L\) leapfrog steps are used per retained draw, then the per-draw cost is
\[
O(L\,C_{\mathrm{like}}).
\]

For \textbf{FullRank-VI}, the ELBO uses the same likelihood term plus a dense Gaussian
variational family \(q_\lambda(w)=\mathcal N(w;\mu,\Sigma)\). Storing the covariance costs
\[
O(q^2)=O(M^2 d^2),
\]
and each Monte Carlo ELBO sample costs
\[
O(C_{\mathrm{like}} + q^2)
\]
up to lower-order linear algebra terms.

For \textbf{Flow-VI}, the likelihood term is again \(O(C_{\mathrm{like}})\), with an additional
flow-transformation cost. For a \(K\)-layer flow with per-layer cost \(c_{\mathrm{flow}}(q)\), one
ELBO sample costs
\[
O\!\left(C_{\mathrm{like}} + K\,c_{\mathrm{flow}}(q)\right).
\]
For planar layers, \(c_{\mathrm{flow}}(q)=O(q)\); for autoregressive spline layers,
\(c_{\mathrm{flow}}(q)\) is typically \(O(qh)\), where \(h\) is the conditioner width.

\
The key scaling result is that, for fixed embedding dimension \(d\), the relaxed likelihood can
be evaluated in time quadratic in trace length after precomputing the pairwise soft-precedence
matrix. Thus the dominant model-level cost is
\[
O\!\left(M^2 d + \sum_{n=1}^N T_n^2\right),
\]
with memory
\[
O(M^2 + Md).
\]
This should be contrasted with the hard discrete model, whose bottleneck is not merely the
cost of a single likelihood evaluation but discrete posterior exploration over a combinatorial
state space of partial orders and linear extensions.

\section{Evaluation criteria}
\label{app:evaluation_criteria}

We evaluate methods using structural recovery, posterior fidelity, predictive fit, and computational cost. Structural metrics are used when a ground-truth partial order is available; posterior-fidelity metrics are used when long-run \hardmcmc{} is feasible; predictive metrics are reported for both synthetic and real datasets. All predictive quantities below are computed from retained posterior or variational draws, rather than from posterior-mean plug-in parameters.
\subsection{Evaluation metrics}
\label{app:evaluation_metrics}
\paragraph{Closure-level structural recovery.}
Let \(h^\star\) be the ground-truth partial order and \(\widehat h\) an inferred partial order. Since the closure-level precedence relation is the primary object of inference, we compare transitive closures:
\[
C^\star = TC(h^\star),
\qquad
\widehat C = TC(\widehat h).
\]
We report
\[
\mathrm{Precision}
=
\frac{|\widehat C \cap C^\star|}{|\widehat C|},
\qquad
\mathrm{Recall}
=
\frac{|\widehat C \cap C^\star|}{|C^\star|},
\]
and
\[
\mathrm{F1}
=
\frac{2\,\mathrm{Precision}\,\mathrm{Recall}}
{\mathrm{Precision}+\mathrm{Recall}}.
\]
Higher values indicate better closure recovery.

\paragraph{Posterior fidelity to \hardmcmc{}.}
On small synthetic instances where long-run hard partial-order MCMC is feasible, we use \hardmcmc{} as the reference posterior. For method \(a\), define
\[
H^{(s,a)}_{ij}
=
\mathbf{1}\{i \succ j \text{ in the closure induced by draw } s\},
\qquad i\neq j,
\]
and estimate posterior pairwise precedence probabilities by
\[
\widehat P^{(a)}_{ij}
=
\frac{1}{S}
\sum_{s=1}^{S}
H^{(s,a)}_{ij}.
\]
Let \(\widehat P^{(\mathrm{hard})}\) denote the corresponding matrix from \hardmcmc{}. We define
\begin{equation}
\mathrm{MAE}_{\mathrm{hard}}(a)
=
\frac{1}{M(M-1)}
\sum_{i\neq j}
\left|
\widehat P^{(a)}_{ij}
-
\widehat P^{(\mathrm{hard})}_{ij}
\right|,
\label{eq:mae_hard}
\end{equation}
where \(M\) is the number of items. Lower values indicate closer agreement with the hard Bayesian reference posterior. We also report Pearson and Spearman correlations between the off-diagonal entries of \(\widehat P^{(a)}\) and \(\widehat P^{(\mathrm{hard})}\).

\paragraph{Trace-level negative log-likelihood.}
For held-out traces \(\mathcal D_{\mathrm{test}}\) and retained draws \(\{\Theta^{(s)}\}_{s=1}^{S}\), we compute posterior-predictive trace-level negative log-likelihood as
\[
\mathrm{Trace\mbox{-}NLL}
=
-\frac{1}{|\mathcal D_{\mathrm{test}}|}
\sum_{y\in\mathcal D_{\mathrm{test}}}
\log
\left[
\frac{1}{S}
\sum_{s=1}^{S}
p(y\mid \Theta^{(s)})
\right].
\]
Lower values indicate better predictive fit for complete held-out traces.

\paragraph{Step-level negative log-likelihood.}
For a held-out trace \(y=(y_1,\dots,y_{T_y})\), let
\[
R_t(y)=S_y\setminus\{y_1,\dots,y_{t-1}\}
\]
be the remaining choice set after the observed prefix, where \(S_y\) is the item set for that trace. For retained posterior or variational draws
\(\{\Theta^{(b)}\}_{b=1}^B\), the posterior-predictive probability of the next observed item is
\[
\bar p_t(y)
=
\frac{1}{B}
\sum_{b=1}^{B}
p\!\left(y_t \mid R_t(y),\Theta^{(b)}\right).
\]
We define
\[
\mathrm{Step\mbox{-}NLL}
=
-
\frac{1}{N_{\mathrm{step}}}
\sum_{y\in\mathcal D_{\mathrm{test}}}
\sum_{t\in\mathcal T_y}
\log \bar p_t(y),
\qquad
N_{\mathrm{step}}
=
\sum_{y\in\mathcal D_{\mathrm{test}}}|\mathcal T_y|.
\]
In our Cloud experiments, \(\mathcal T_y=\{1,\dots,T_y\}\), so the first action is scored from the empty prefix. If one wants to score only nonempty prefixes, the same definition applies with \(\mathcal T_y=\{2,\dots,T_y\}\).

\paragraph{WAIC.}
For Bayesian methods, we report WAIC using pointwise log-likelihoods over retained draws. Let
\[
\ell_{is} = \log p(y^{(i)} \mid \Theta^{(s)})
\]
be the log-likelihood contribution of observation \(i\) under draw \(s\). We compute
\[
\mathrm{lppd}
=
\sum_i
\log
\left(
\frac{1}{S}
\sum_{s=1}^{S}
\exp(\ell_{is})
\right),
\qquad
p_{\mathrm{WAIC}}
=
\sum_i
\mathrm{Var}_{s=1,\dots,S}(\ell_{is}),
\]
and
\[
\mathrm{WAIC}
=
-2(\mathrm{lppd}-p_{\mathrm{WAIC}}).
\]
Lower WAIC indicates better estimated out-of-sample predictive performance. Large \(p_{\mathrm{WAIC}}\) indicates high per-draw log-likelihood variability and should be interpreted as a warning that the posterior approximation may be unstable for WAIC.

\paragraph{Per-draw predictive models.}
For \hardmcmc{}, predictive probabilities are evaluated using the hard frontier-softmax model on retained hard closure draws. For \relaxedmcmc{}, \fullrankvi{}, and \flowvi{}, predictive probabilities are evaluated using the relaxed frontier-softmax model on each retained soft-precedence draw. Non-Bayesian baselines, when included, are evaluated at their fitted point estimate and are not assigned posterior-fidelity or WAIC uncertainty summaries.

\paragraph{Runtime.}
Runtime is wall-clock time for fitting the method and producing the posterior samples or variational draws used for evaluation. All runtime comparisons are reported in seconds.

\subsection{Incomparable-pair coverage and trace sufficiency}
\label{app:ipcov}

We use incomparable-pair coverage (IP-Cov) to quantify trace diversity in settings where a ground-truth partial order is available.

\paragraph{Ground-truth incomparable pairs.}
Let \(h^\star=(S,\succ_{h^\star})\) be the ground-truth strict partial order. A pair \((i,j)\) is incomparable if neither item reaches the other in the transitive closure:
\begin{equation}
\mathcal P_{\parallel}(h^\star)
=
\{(i,j): 1\le i<j\le m,\; i\not\succ_{h^\star} j,\; j\not\succ_{h^\star} i\}.
\label{eq:incomparable_pairs_eval}
\end{equation}

\paragraph{Order induced by a trace.}
Each trace \(y\in\mathcal D\) induces a total order over the items appearing in that trace. We write
\[
i\succ_y j
\]
if both \(i\) and \(j\) appear in \(y\) and \(i\) appears before \(j\).

\paragraph{Incomparable-pair coverage.}
Incomparable-pair coverage measures how often true incomparable pairs are observed in both relative orders across the dataset:
\begin{equation}
\mathrm{IP\text{-}Cov}(\mathcal D;h^\star)
=
\frac{1}{|\mathcal P_{\parallel}(h^\star)|}
\sum_{(i,j)\in\mathcal P_{\parallel}(h^\star)}
\mathbf 1
\left[
\exists y,y'\in\mathcal D:
(i\succ_y j)\wedge (j\succ_{y'} i)
\right].
\label{eq:ipcov_eval}
\end{equation}
Intuitively, IP-Cov is the fraction of ground-truth incomparable pairs for which the data contain evidence in both directions. Such evidence helps distinguish genuine incomparability from strict precedence.

\paragraph{Trace sufficiency.}
We say that a trace set is \emph{incomparability-sufficient} for \(h^\star\) if
\begin{equation}
\mathrm{IP\text{-}Cov}(\mathcal D;h^\star)=1.
\label{eq:ipcov_sufficiency_eval}
\end{equation}
This condition does not by itself guarantee perfect recovery under finite data or model mismatch. Rather, it ensures that every ground-truth incomparable pair is witnessed in both relative orderings, providing the necessary evidence to distinguish concurrency from strict precedence. In our synthetic and reference-trace experiments, IP-Cov is used as a controllable measure of trace diversity.

\subsection{Non-Bayesian structural baselines}
\label{app:structural_baselines}
We compare against two non-Bayesian structural baselines: Majority and Continuous DAG Learning. These baselines produce a directed cover graph rather than a Bayesian posterior. We therefore evaluate them only using closure-level precision, recall, F1, and runtime.

\subsubsection{Majority baseline.}
\label{alg:majority_baseline_instruct}
The Majority baseline estimates pairwise precedence frequencies from the observed traces. Let
\[
\widehat r_{ij}
=
\frac{
\sum_{y\in\mathcal D}\mathbf 1[i\text{ appears before }j\text{ in }y]
}{
\sum_{y\in\mathcal D}\mathbf 1[i,j\in y]
}.
\]
We add edge \(i\to j\) when \(\widehat r_{ij}>\tau\) and \(\widehat r_{ij}>\widehat r_{ji}\), with default \(\tau=0.5\). Edge weights are set to \(|\widehat r_{ij}-0.5|\), so more decisive pairwise majorities receive larger weights. See algorithm \ref{alg:majority_baseline}

\begin{algorithm}[h]
\caption{\textsc{CycleBreakAndCover}(\(A,W\))}
\label{alg:cycle_break_cover}
\begin{algorithmic}[1]
\Require adjacency \(A\in\{0,1\}^{n\times n}\), weights \(W\in\mathbb R_{\ge 0}^{n\times n}\)
\Ensure DAG cover graph \(\widehat H\)
\While{\(\textsc{HasCycle}(A)\)}
    \State \(C \gets \textsc{FindCycle}(A)\)
    \State \((u,v) \gets \arg\min_{(i,j)\in C} W_{ij}\)
    \State \(A_{uv} \gets 0\)
\EndWhile
\State \(\widehat C \gets \mathrm{TC}(A)\)
\State \(\widehat H \gets \mathrm{TR}(\widehat C)\)
\State \Return \(\widehat H\)
\end{algorithmic}
\end{algorithm}

\begin{algorithm}[h]
\caption{Majority pairwise-precedence baseline}
\label{alg:majority_baseline}
\begin{algorithmic}[1]
\Require traces \(\mathcal D=\{y^{(t)}\}_{t=1}^N\), threshold \(\theta_{\mathrm{maj}}\)
\Ensure DAG cover graph \(\widehat H_{\mathrm{maj}}\)
\State Initialize \(C,T\in\mathbb R^{n\times n}\) to zero
\For{each trace \(y\in\mathcal D\)}
    \State compute positions \(\mathrm{pos}_y(\cdot)\)
    \For{all \(i\ne j\) co-occurring in \(y\)}
        \State \(T_{ij}\gets T_{ij}+1\)
        \If{\(\mathrm{pos}_y(i)<\mathrm{pos}_y(j)\)}
            \State \(C_{ij}\gets C_{ij}+1\)
        \EndIf
    \EndFor
\EndFor
\State Initialize \(A,W\in\mathbb R^{n\times n}\) to zero
\For{all unordered pairs \(\{i,j\}\)}
    \If{\(T_{ij}=0\) or \(T_{ji}=0\)}
        \State \textbf{continue}
    \EndIf
    \State \(\widehat r_{ij}\gets C_{ij}/T_{ij}\), \(\widehat r_{ji}\gets C_{ji}/T_{ji}\)
    \If{\(\widehat r_{ij}>\theta_{\mathrm{maj}}\) and \(\widehat r_{ij}>\widehat r_{ji}\)}
        \State \(A_{ij}\gets 1\), \(W_{ij}\gets |\widehat r_{ij}-\widehat r_{ji}|\)
    \ElsIf{\(\widehat r_{ji}>\theta_{\mathrm{maj}}\) and \(\widehat r_{ji}>\widehat r_{ij}\)}
        \State \(A_{ji}\gets 1\), \(W_{ji}\gets |\widehat r_{ji}-\widehat r_{ij}|\)
    \EndIf
\EndFor
\State \(\widehat H_{\mathrm{maj}}\gets \textsc{CycleBreakAndCover}(A,W)\)
\State \Return \(\widehat H_{\mathrm{maj}}\)
\end{algorithmic}
\end{algorithm}

\subsubsection{SoftDAG-Frontier diagnostic baseline}
\label{app:softdag}

SoftDAG-Frontier is a MAP diagnostic baseline inspired by continuous DAG
learning~\citep{zheng2018dags,lorch2021dibs}. Unlike the proposed Bayesian
product-order model, it directly parameterizes a soft adjacency matrix. Let
\(A\in\mathbb{R}^{n\times n}\) be free edge logits and define
\[
W_{ij}=\sigma(A_{ij}), \qquad W_{ii}=0.
\]
We encourage acyclicity using the NOTEARS penalty
\[
h(W)=\operatorname{tr}\{\exp(W\circ W)\}-n.
\]
To use the same trace observation model as the proposed method, we convert
\(W\) into a soft reachability matrix
\[
R_K(W)
=
1-\exp\!\left\{-\sum_{\ell=1}^{K} W^\ell\right\},
\qquad R_{ii}=0,
\]
with \(K=\min(8,n-1)\). We substitute \(R_K(W)\) for the soft precedence
matrix in the relaxed frontier likelihood:
\[
\widetilde F_t(x;R)
=
\prod_{z\in R_t\setminus\{x\}}
(1-R_{zx}),
\]
\[
\widetilde S_t(x;R)
=
\sum_{z\in R_t\setminus\{x\}} R_{xz},
\qquad
\widetilde Q_t(x;R)=\log(1+\widetilde S_t(x;R)).
\]
The model is fit by minimizing
\[
\mathcal L(A,\beta)
=
-\log \widetilde p(D\mid R_K(W),\beta)
+
\lambda_1\|W\|_1
+
\lambda_h\, h(W)^2.
\]
This baseline is not Bayesian and does not provide posterior uncertainty over
closure relations. We therefore evaluate it using closure recovery, held-out
NLL, and runtime, but not posterior-fidelity MAE or WAIC.

\textbf{Implementation details.}
We use \(K=\min(8,n-1)\), three random restarts of Adam, and validation NLL
for hyperparameter selection over \(\lambda_1\in\{10^{-4},10^{-3},10^{-2}\}\)
and \(\lambda_h\in\{1,10,100\}\). Because the trace likelihood already
identifies the scale of \(\beta\), we parameterize \(\beta=\exp(\eta_\beta)\)
and learn \(\eta_\beta\) jointly with \(A\) for both synthetic and cloud
experiments, matching the corresponding setting of the proposed method which
also infers \(\beta\) jointly with the partial-order representation. We
approximate the matrix exponential in the acyclicity penalty by its Taylor
truncation to degree \(\max(K,12)\), which is positive and equals zero iff
\(W\) is acyclic, and is autodiff-friendly. The soft adjacency logits are
initialised at \(A_{ij}\!\sim\!\mathcal{N}(-2,0.1^2)\) so that \(W\) starts
sparse, and we run Adam for 400 steps at learning rate \(0.05\).
Hyperparameter selection never touches the test set or ground-truth
structure; selection is done on validation step-NLL.

\textbf{Optimization and model selection.}
For each synthetic instance, SoftDAG-Frontier is fit by MAP optimization with Adam. Hyperparameters
are selected using an 80/20 split of the training traces into train and validation subsets, using
validation Step-NLL only; neither the test traces nor the ground-truth closure are used for model
selection. The grid search varies the sparsity and acyclicity penalties. For scalability, the
optimization budget is reduced with \(n\): we use \(600\) steps and \(3\) restarts for
\(n\leq 30\), \(400\) steps and \(2\) restarts for \(n=50\), and \(250\) steps and \(1\) restart
for \(n=100\). All 36 synthetic configurations completed.

\textbf{Decoding.}
Unlike the Bayesian methods, SoftDAG-Frontier returns a single MAP soft-reachability matrix rather
than posterior precedence probabilities. The primary reported closure F1 uses
\(\zeta=0.5\) for \(n<100\) and
\(\zeta=1/3\) for \(n=100\). We use the lower threshold at
\(n=100\) because SoftDAG produces MAP soft-reachability scores rather than
posterior closure probabilities, and its large-\(n\) scores are not calibrated
on the same scale as the Bayesian posterior means. 

\textbf{Observed degeneracies.}
The SoftDAG grid reveals two characteristic instabilities of this generic MAP baseline. First, at
large \(n\), some runs converge to a symmetric uninformative solution. In the two degenerate
\(n=100\) cases, the decoded closure contains zero directed edges and the test Step-NLL is
approximately \(3.6374\). This value matches the uniform frontier prediction
\[
\frac{1}{100}\sum_{r=1}^{100}\log r
=
\frac{\log(100!)}{100}
\approx 3.637 ,
\]
which is the per-step NLL obtained when all remaining actions receive equal score. Thus the zero
F1 cases are not caused by an evaluation bug: the optimizer has found an uninformative symmetric
solution, so the decoder has no strict directed preference to threshold.

Second, for \(n=5\), the train/validation split is extremely small: after the 80/20 split, the
baseline is selected from only two training traces and one validation trace. Several selected runs
therefore have very large learned inverse-temperature parameters \(\beta\), indicating a highly
unstable MAP objective under very limited data. This explains the high variance of SoftDAG at
small \(n\). By contrast, the Bayesian methods average over posterior uncertainty and are less
sensitive to these small-sample validation artifacts.

\section{Experiment}

\subsection{Synthetic Dataset}
For each \((n,\rho)\) setting, we generated three random instances and summarized the resulting numbers of training and test traces, the target trace budget 2n, and the realized incomparable-pair coverage. See Table \ref{tab:synthetic_generation_summary}. 
\begin{table}[t]
\centering
\small
\renewcommand{\arraystretch}{1.12}
\caption{Synthetic trace-generation summary across three seeds for each \((n,\rho)\) setting.
\texttt{train} and \texttt{test} are the mean numbers of training and test traces, respectively, with ranges across seeds shown in parentheses.
\(\mathrm{IP\text{-}Cov}\) is the mean incomparable-pair coverage, with the min--max range across seeds in parentheses.}
\label{tab:synthetic_generation_summary}
\begin{tabular}{rrrrr}
\toprule
\(n\) & \(\rho\) & train & test & mean IP-Cov \\
\midrule
5   & 0.5 & 5.0   \,(5--5)      & 1.0  \,(1--1)    & 1.00 \,(1.00--1.00) \\
5   & 0.9 & 5.0   \,(5--5)      & 1.0  \,(1--1)    & 1.00 \,(1.00--1.00) \\
10  & 0.5 & 10.0  \,(10--10)    & 2.0  \,(2--2)    & 1.00 \,(1.00--1.00) \\
10  & 0.9 & 10.0  \,(10--10)    & 2.0  \,(2--2)    & 1.00 \,(1.00--1.00) \\
20  & 0.5 & 20.0  \,(20--20)    & 4.0  \,(4--4)    & 1.00 \,(1.00--1.00) \\
20  & 0.9 & 20.0  \,(20--20)    & 4.0  \,(4--4)    & 1.00 \,(1.00--1.00) \\
30  & 0.5 & 30.0  \,(30--30)    & 6.0  \,(6--6)    & 1.00 \,(1.00--1.00) \\
30  & 0.9 & 30.0  \,(30--30)    & 6.0  \,(6--6)    & 1.00 \,(1.00--1.00) \\
50  & 0.5 & 50.0  \,(50--50)    & 10.0 \,(10--10)  & 0.99 \,(0.99--1.00) \\
50  & 0.9 & 50.0  \,(50--50)    & 10.0 \,(10--10)  & 1.00 \,(1.00--1.00) \\
100 & 0.5 & 124.0 \,(109--132)  & 24.7 \,(22--26)  & 0.98 \,(0.97--0.98) \\
100 & 0.9 & 100.0 \,(100--100)  & 20.0 \,(20--20)  & 1.00 \,(1.00--1.00) \\
\bottomrule
\end{tabular}
\end{table}

\subsubsection{Full synthetic grid results}
\label{app:synthetic_full_grid}

Table~\ref{tab:full-synthetic-grid-unified} reports the full synthetic grid results by graph size \(n\), latent correlation \(\rho\), and inference method. Results are averaged over the completed random seeds for each configuration. The hard sampler is used as the Bayesian reference whenever it is feasible; therefore its MAE to hard MCMC is zero by definition. For relaxed methods, MAE measures the average absolute difference between posterior pairwise precedence probabilities and the corresponding hard-MCMC posterior probabilities. Structural recovery is reported as closure-level F1 after thresholding posterior precedence probabilities at two decoding thresholds, \(0.5\) and \(1/3\). Predictive fit is measured by held-out trace-level NLL, step-level NLL, and WAIC. Runtime is wall-clock fitting time in seconds.

\begin{table}[h]
\centering
\scriptsize
\setlength{\tabcolsep}{3pt}
\caption{Full synthetic grid under the posterior-predictive evaluator. Trace NLL, Step NLL, and WAIC are computed from each method's retained posterior or variational draws on fixed train/test splits. Lower is better for MAE, NLLs, WAIC, and runtime; higher is better for $F_1$.}
\label{tab:full-synthetic-grid-unified}

\resizebox{\textwidth}{!}{%
\begin{tabular}{@{}lrrrrrrrrrr@{}}
\toprule
Method & $n$ & $\rho$ & Seeds
& MAE $\downarrow$
& $F_1@0.5 \uparrow$
& $F_1@1/3 \uparrow$
& Trace NLL $\downarrow$
& Step NLL $\downarrow$
& WAIC $\downarrow$
& Time (s) $\downarrow$ \\
\midrule
\hardmcmc{} & 5 & 0.5 & 3 & 0.000 & 1.000 & 1.000 & 3.338 & 0.668 & 35.96 & 87.58 \\
\relaxedmcmc{} & 5 & 0.5 & 3 & 0.082 & 0.889 & 0.917 & 3.903 & 0.772 & 41.92 & 0.60 \\
\fullrankvi{} & 5 & 0.5 & 3 & 0.094 & 0.889 & 0.886 & 3.952 & 0.783 & 45.67 & 3.93 \\
\flowvi{} & 5 & 0.5 & 3 & 0.089 & 0.889 & 0.827 & 3.939 & 0.768 & -- & 17.21 \\
\addlinespace

\hardmcmc{} & 5 & 0.9 & 3 & 0.000 & 1.000 & 1.000 & 2.003 & 0.400 & 21.33 & 89.73 \\
\relaxedmcmc{} & 5 & 0.9 & 3 & 0.055 & 1.000 & 0.952 & 2.738 & 0.547 & 25.90 & 0.71 \\
\fullrankvi{} & 5 & 0.9 & 3 & 0.060 & 0.974 & 0.952 & 2.749 & 0.548 & 28.91 & 3.61 \\
\flowvi{} & 5 & 0.9 & 3 & 0.065 & 0.883 & 0.923 & 2.513 & 0.507 &-- & 18.79 \\
\addlinespace

\hardmcmc{} & 10 & 0.5 & 3 & 0.000 & 1.000 & 1.000 & 9.000 & 0.897 & 183.70 & 587.10 \\
\relaxedmcmc{} & 10 & 0.5 & 3 & 0.078 & 0.944 & 0.923 & 10.164 & 1.015 & 211.92 & 6.31 \\
\fullrankvi{} & 10 & 0.5 & 3 & 0.096 & 0.927 & 0.848 & 10.441 & 1.040 & 222.89 & 21.32 \\
\flowvi{} & 10 & 0.5 & 3 & 0.128 & 0.844 & 0.763 & 10.222 & 1.029 &-- & 21.79 \\
\addlinespace

\hardmcmc{} & 10 & 0.9 & 3 & 0.000 & 1.000 & 1.000 & 5.464 & 0.552 & 112.65 & 546.65 \\
\relaxedmcmc{} & 10 & 0.9 & 3 & 0.106 & 0.856 & 0.936 & 6.265 & 0.623 & 133.69 & 13.43 \\
\fullrankvi{} & 10 & 0.9 & 3 & 0.049 & 0.974 & 0.965 & 6.242 & 0.622 & 142.86 & 27.42 \\
\flowvi{} & 10 & 0.9 & 3 & 0.058 & 0.974 & 0.956 & 6.119 & 0.595 &-- & 27.82 \\
\addlinespace

\hardmcmc{} & 20 & 0.5 & 3 & 0.000 & 0.993 & 1.000 & 29.337 & 1.465 & 1170.41 & 3209.21 \\
\relaxedmcmc{} & 20 & 0.5 & 3 & 0.036 & 0.988 & 0.951 & 30.052 & 1.502 & 1228.63 & 215.20 \\
\fullrankvi{} & 20 & 0.5 & 3 & 0.062 & 0.899 & 0.892 & 31.420 & 1.568 & 1332.57 & 104.39 \\
\flowvi{} & 20 & 0.5 & 3 & 0.142 & 0.765 & 0.727 & 39.460 & 1.887 &-- & 46.65 \\
\addlinespace

\hardmcmc{} & 20 & 0.9 & 3 & 0.000 & 1.000 & 1.000 & 19.632 & 0.982 & 784.06 & 3461.29 \\
\relaxedmcmc{} & 20 & 0.9 & 3 & 0.033 & 0.984 & 0.963 & 20.328 & 1.018 & 840.82 & 211.48 \\
\fullrankvi{} & 20 & 0.9 & 3 & 0.051 & 0.959 & 0.928 & 21.115 & 1.056 & 896.88 & 122.40 \\
\flowvi{} & 20 & 0.9 & 3 & 0.042 & 0.974 & 0.958 & 21.633 & 1.052 &-- & 41.49 \\
\addlinespace

\hardmcmc{} & 30 & 0.5 & 2 & 0.000 & 0.976 & 0.987 & 47.045 & 1.567 & 2825.70 & 10792.52 \\
\relaxedmcmc{} & 30 & 0.5 & 2 & 0.027 & 0.983 & 0.983 & 47.661 & 1.588 & 2897.07 & 2422.53 \\
\fullrankvi{} & 30 & 0.5 & 2 & 0.070 & 0.895 & 0.888 & 51.038 & 1.698 & 3302.83 & 488.10 \\
\flowvi{} & 30 & 0.5 & 2 & 0.050 & 0.912 & 0.903 & 48.903 & 1.631 &-- & 45.88 \\
\addlinespace

\hardmcmc{} & 30 & 0.9 & 3 & 0.000 & 0.998 & 0.999 & 32.838 & 1.095 & 1984.08 & 11383.83 \\
\relaxedmcmc{} & 30 & 0.9 & 3 & 0.023 & 0.990 & 0.975 & 33.405 & 1.113 & 2045.19 & 1980.36 \\
\fullrankvi{} & 30 & 0.9 & 3 & 0.042 & 0.958 & 0.945 & 35.726 & 1.192 & 2266.04 & 477.45 \\
\flowvi{} & 30 & 0.9 & 3 & 0.017 & 0.991 & 0.985 & 33.218 & 1.107 &-- & 46.04 \\
\addlinespace

\hardmcmc{} & 50 & 0.5 & 3 & 0.000 & 0.917 & 0.932 & 96.793 & 1.936 & 9683.92 & 53353.65 \\
\relaxedmcmc{} & 50 & 0.5 & 2 & 0.174 & 0.492 & 0.489 & 88.778 & 1.776 & 8909.54 & 30937.44 \\
\fullrankvi{} & 50 & 0.5 & 3 & 0.056 & 0.934 & 0.923 & 98.634 & 1.971 & 10335.71 & 7251.24 \\
\flowvi{} & 50 & 0.5 & 3 & 0.057 & 0.906 & 0.898 & 99.151 & 1.970 & -- & 45.92 \\
\addlinespace
\hardmcmc{} & 50 & 0.9 & 2 & 0.000 & 0.995 & 0.997 & 69.440 & 1.389 & 6951.46 & 108577.87 \\
\relaxedmcmc{} & 50 & 0.9 & 1 & 0.019 & 0.995 & 0.984 & 68.270 & 1.366 & 6809.87 & 66433.96 \\
\fullrankvi{} & 50 & 0.9 & 3 & 0.022 & 0.986 & 0.976 & 67.850 & 1.358 & 7005.71 & 5265.21 \\
\flowvi{} & 50 & 0.9 & 3 & 0.022 & 0.983 & 0.981 & 67.789 & 1.345 & -- & 45.86 \\
\bottomrule
\end{tabular}%
}
\begin{minipage}{0.98\textwidth}
\scriptsize
\textit{Note.}
The evaluator is described in Appendix~E. \hardmcmc{} evaluates the hard frontier-softmax step model on retained binary closure draws after burn-in; \relaxedmcmc{}, \fullrankvi{}, and \flowvi{} evaluate the relaxed frontier-softmax step model on per-draw soft precedence matrices. Trace NLL uses the posterior-predictive held-out trace likelihood; Step NLL averages per-draw next-step probabilities before taking $-\log$ and pooling over held-out trace steps. WAIC uses $-2(\mathrm{lppd}-p_{\mathrm{WAIC}})$ on the training split. Closure MAE is the off-diagonal mean absolute deviation of each method's posterior mean closure from the \hardmcmc{} posterior mean closure, so it is zero by definition for \hardmcmc{}. WAIC is omitted for \flowvi{} because its large $p_{\mathrm{WAIC}}$ indicates high per-draw log-likelihood variance.
\end{minipage}
\end{table}

\begin{figure}[t]
    \centering
    \includegraphics[width=0.95\textwidth]{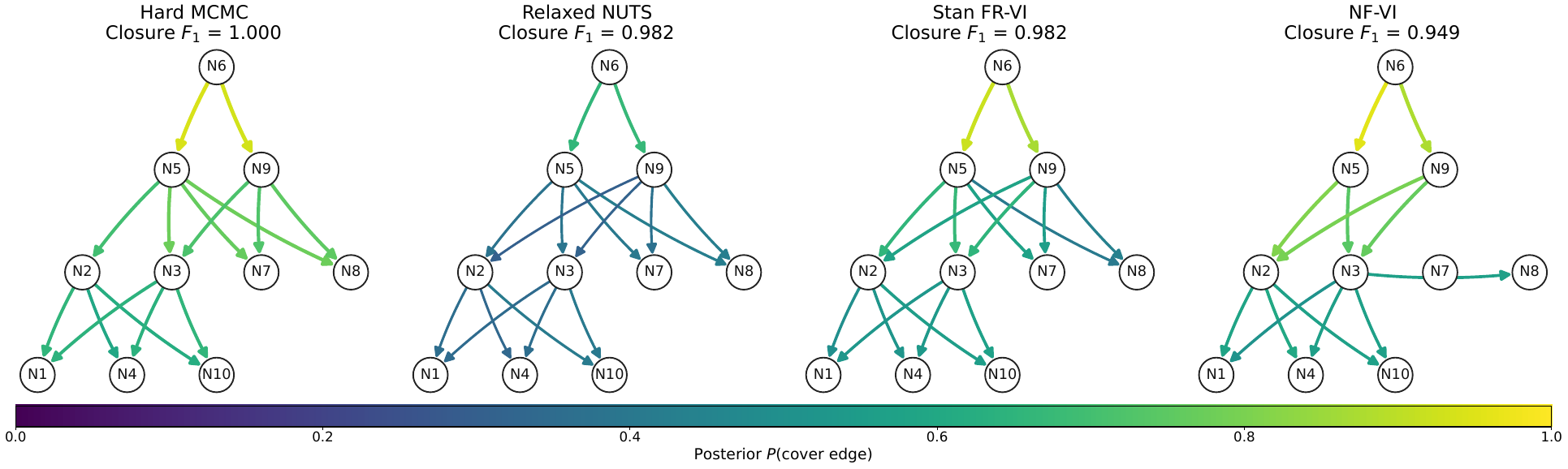}
    \caption{
    Representative synthetic Hasse diagrams for \(n=10\), \(\rho=0.9\). 
    \hardmcmc{} is the reference posterior; \relaxedmcmc{} and \fullrankvi{} use the relaxed model with \(\tau=0.5\). 
    Edge color encodes posterior cover-edge probability. Both relaxed methods recover structures closely aligned with \hardmcmc{}, though with more diffuse posterior confidence.
    }
    \label{fig:synthetic_n10_hasse}
\end{figure}

\subsubsection{Ablation study: normalizing-flow sensitivity to \texorpdfstring{$\tau$}{tau} across graph sizes}
\label{sec:nf_tau_sensitivity}

We study how the performance of the normalizing-flow variational posterior depends on the smoothing temperature \(\tau\) across synthetic partial-order datasets of different sizes. Averaged over two representative random seeds, the best closure-level F1 is obtained at \(\tau=0.1\) for \(n=10\) and \(n=20\), with mean F1 \(=0.982\) in both cases. For larger graphs, the preferred temperature shifts upward: the best mean F1 is \(0.897\) at \((n=50,\tau=0.5)\) and \(0.828\) at \((n=100,\tau=0.7)\). This pattern suggests that smaller problems favor sharper relaxations, whereas larger graphs benefit from additional smoothing.

\subsubsection{Ablation: trace coverage and partial-order identifiability.}
\label{sec:ipcov-ablation}
The incomparable-pair coverage (IP-Cov) measures the fraction of truly
incomparable pairs that are observed in \emph{both} relative orderings across
the training traces. When IP-Cov is below one, some reorderable pairs appear in
only a single orientation, and the partial order becomes only partially
identifiable from the observations: such pairs are observationally
indistinguishable from genuine precedence constraints, regardless of which
inference algorithm is applied.

Figure~\ref{fig:ipcov-ablation} quantifies this effect on synthetic instances
with $n=30$, $\rho=0.5$, and $\tau=0.5$. As IP-Cov rises from $0.7$ to $1.0$,
closure-level recovery improves sharply for every method that integrates over
the latent graph. Relaxed NUTS and \hardmcmc{}, which both produce posterior samples over decoded partial order, recover the true closure almost perfectly at IP-Cov
$=1.0$, while their held-out NLL drops to the level of the oracle hard model.
The same trend holds for Stan~fullrank and NF-VI, but with a smaller absolute
gain. These results indicate that accurate latent graph recovery is
fundamentally trace-bound: when the trace set witnesses every incomparability
in both orientations, all four inference families converge to high-quality
posteriors; below that threshold, the limiting factor is the information
content of the traces, not the choice of inference algorithm.

\begin{figure*}[t]
    \centering
    \includegraphics[width=\textwidth]{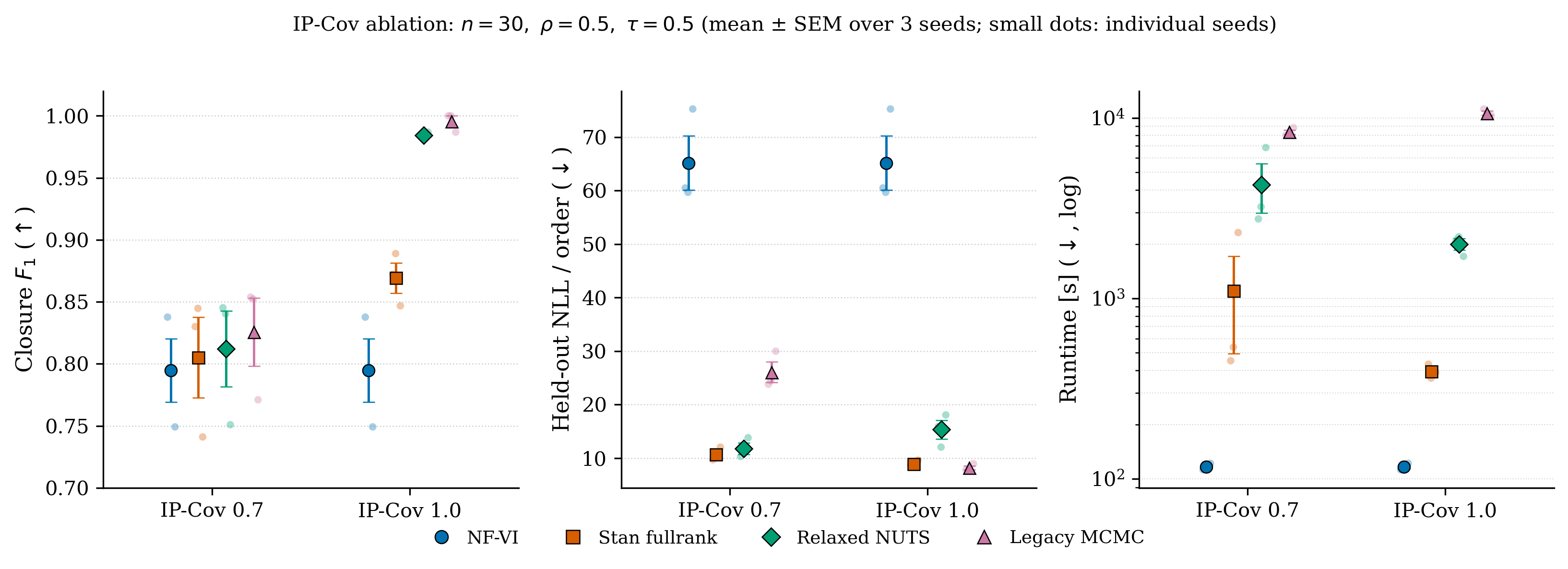}
    \caption{%
    \textbf{Trace-coverage ablation on the latent partial order.}
    Synthetic instances with $n{=}30$, $\rho{=}0.5$, $\tau{=}0.5$, three seeds.
    Each panel shows mean $\pm$ SEM (markers) and individual seeds (small
    dots) for four inference methods at two coverage levels.
    \textit{Left:} closure $F_1$; \textit{middle:} held-out negative log
    likelihood per test order under the hard frontier-softmax model
    (a single, common evaluation likelihood is used for all methods to ensure
    apples-to-apples comparison); \textit{right:} wall-clock runtime
    (log scale).
    Increasing IP-Cov from $0.7$ to $1.0$ resolves the partial-identifiability
    of incomparable pairs and yields near-perfect closure recovery for the
posterior samplers (Relaxed~NUTS, Legacy~MCMC); held-out NLL drops by
    a factor of roughly three across all methods, while runtime is largely
    unchanged.
    }
    \label{fig:ipcov-ablation}
\end{figure*}

\subsubsection{Relaxation-temperature ablation}
\label{sec:tau-ablation}

We vary the relaxation temperature
$\tau\in\{0.1,0.3,0.5,1.0\}$ on synthetic instances with
$n=30$, $\rho=0.5$, and IP-Cov $\approx 1.0$, so that the traces contain
sufficient information to identify the true graph. Figure~\ref{fig:tau-ablation}
shows the expected sharpness--smoothness trade-off. Smaller $\tau$ gives better
hard-model fidelity and closure recovery: Relaxed NUTS at
$\tau\in\{0.1,0.3\}$ nearly matches the \hardmcmc{} reference
($F_1>0.99$ and hard-frontier NLL $\approx 49.5$ versus $49.3$), while using
roughly $4$--$5\times$ less runtime. FullRank-VI also improves as $\tau$
decreases, but remains near $F_1\approx 0.9$, suggesting a structural-recovery
limit from the Gaussian variational family. In contrast, $\tau=1.0$ clearly
over-smooths the relaxation: hard-frontier NLL increases to roughly $90$ and
closure recovery degrades. Thus $\tau\approx 0.1$--$0.3$ gives the best balance
between approximating the hard partial-order target and maintaining tractable
continuous inference.

\begin{figure}[t]
    \centering
    \includegraphics[width=\linewidth]{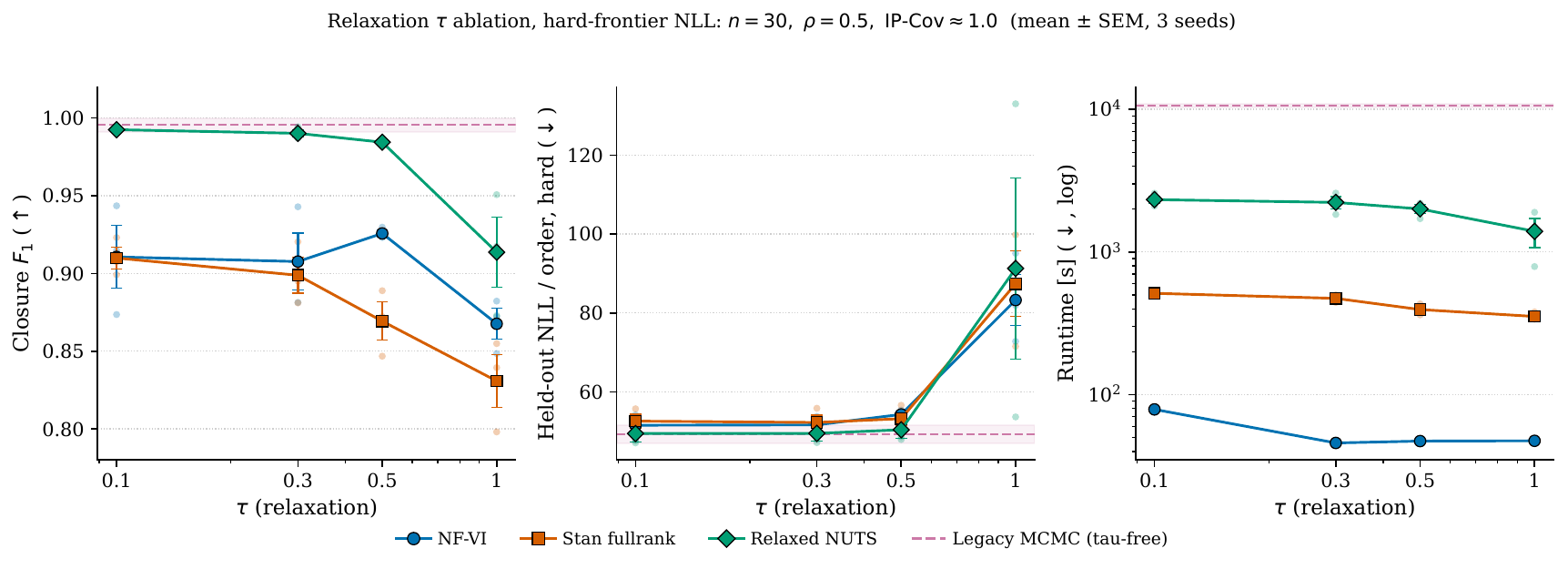}
    \caption{
    \textbf{Relaxation-temperature ablation.}
    We vary $\tau$ on synthetic instances with $n=30$, $\rho=0.5$, and
    IP-Cov $\approx 1.0$. Smaller $\tau$ gives a sharper approximation to the
    hard partial-order model and improves closure recovery. Relaxed NUTS with
    $\tau\in\{0.1,0.3\}$ nearly matches the \hardmcmc{} reference in closure
    $F_1$ and hard-frontier NLL while using substantially less runtime.
    Conversely, $\tau=1.0$ over-smooths the relaxation, causing worse NLL and
    degraded graph recovery. Error bars show SEM over seeds; faint points show
    individual seeds.
    }
    \label{fig:tau-ablation}
\end{figure}

\begin{figure}[t]
    \centering
    \includegraphics[width=\textwidth]{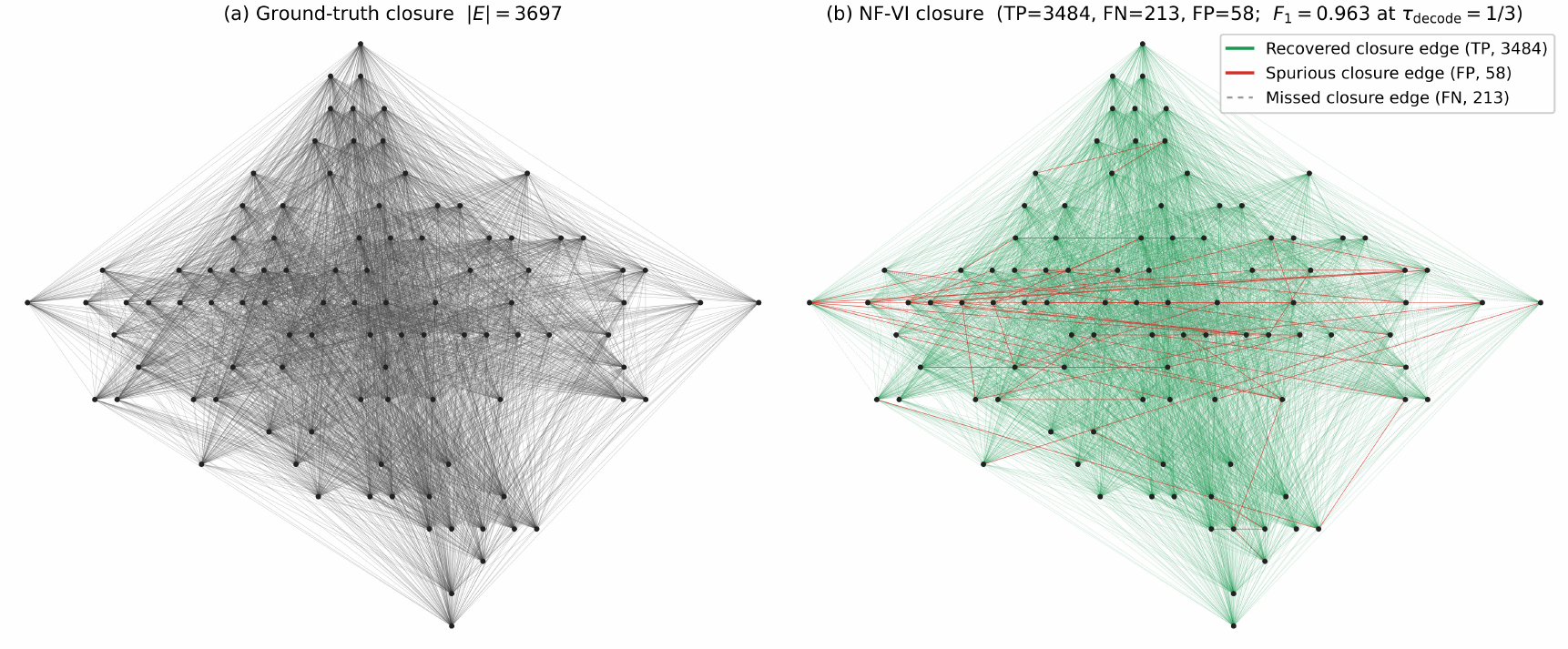}
    \caption{
    \textbf{NF-VI recovers the latent partial order at \(n=100\).}
    Best NF-VI run (\(\rho=0.9\), seed \(19\), \(\tau=0.3\), training IP-Cov \(\geq 0.97\)).
    Both panels share the same node positions, computed by Graphviz \texttt{dot} on the ground-truth cover graph.
    \emph{(a)} All true closure (transitive-closure) edges (\(|E|=3697\)).
    \emph{(b)} NF-VI closure decoded from the posterior pairwise-precedence probabilities at threshold \(1/3\), 
    coloured by edge type: 
    \emph{recovered (TP)} \(=3484\) (green); 
    \emph{spurious (FP)} \(=58\) (red); 
    \emph{missed (FN)} \(=213\) (light dashed).
    Closure-level \(F_1=0.963\) (precision \(0.984\), recall \(0.942\)).
    NF-VI is the only method that completes inference at this scale (mean wall-clock \(\approx 67\,\mathrm{s}\)).
    }
    \label{fig:flow_n100_closure}
\end{figure}

\subsection{Bishops data details}
\label{app:bishops_details}
We use witness lists from the ``Royal Acta'' database created for \emph{The Charters of William II and Henry I} project by Richard Sharpe and Nicholas Karn \citep{sharpe2014charters,nicholls2025royalacta}. The database records dated witness lists from legal documents issued in England and Wales during the eleventh and twelfth centuries. Each list is associated with a date, although the dating is sometimes uncertain; in such cases the data include lower and upper bounds on the plausible date range. We take all lists dated 1134-1138 CE and involving all 45 witnesses who appeared in 5 or more lists. The partial order has $M=45$ elements, there are $N=68$ lists and the longest list has $32$ items/witnesses.

Each individual in the database is associated with a profession or title, such as Queen, Archbishop, and so on; individuals without a recorded title are grouped into the category \emph{other}. Witness names are written down in order of social status. The witness lists reflect a strong but not necessarily total social hierarchy: some individuals or title groups reliably precede others, while many relations remain uncertain, indirect, or context-dependent. This makes the dataset a natural real-world example of latent partial-order inference rather than ordinary total-order ranking.

The dataset is also computationally relevant for our purposes. Prior Bayesian work on partial orders from random linear extensions reports that counting linear extensions is typically feasible only up to about \(20\) actors (in an MCMC-setting, where it has to be done thousands or even millions of times, bigger posets are countable as a one-off), whereas larger random partial orders can become especially difficult to analyze exactly \citep{nicholls2025royalacta}. Our instance contains 45 actors, placing it beyond the regime where hard discrete
partial-order inference is routinely comfortable. We therefore use this dataset to assess whether the differentiable relaxation extends Bayesian partial-order modeling to larger and more challenging real ranking problems. Figure~\ref{fig:bishop_data_sample} illustrates the form of the raw data used in the Bishops experiment. Each document contributes an ordered witness list, such as the one shown here, where higher-status individuals typically appear earlier. Our model does not assume that the full social hierarchy is totally ordered; instead, each observed list is treated as one linear extension of an underlying latent partial order.
\begin{figure}[t]
    \centering
    \includegraphics[width=0.5\linewidth]{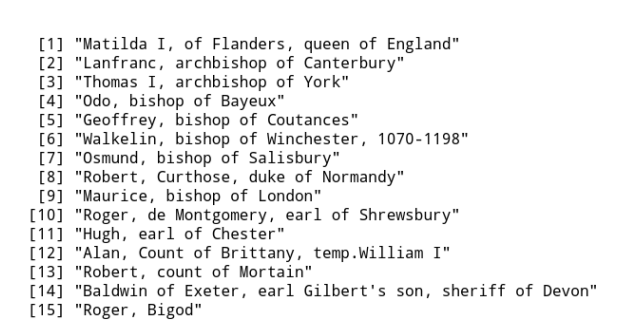}
    \caption{Example of an observed witness list from the Bishops / Royal Acta corpus. Each entry is a named witness recorded in the order in which the document presents them. In the model, such a list is treated as an observed linearization of an underlying latent social hierarchy, rather than as evidence for a single total order.}
    \label{fig:bishop_data_sample}
\end{figure}
\paragraph{Common settings.}
The relaxed methods are fit under the same relaxed partial-order likelihood
with $\tau=0.10$, $\rho\sim\mathrm{Beta}(2,2)$, and
$\gamma\sim\mathrm{Gamma}(2,1)$. Hard-MCMC instead targets the corresponding
hard frontier-softmax model. All partial-order methods fix $\beta=0$.

\begin{table*}[h]
\centering
\small
\renewcommand{\arraystretch}{1.15}
\caption{Inference settings for the Bishops experiment. The relaxed methods use \(d=3\), \(\tau=0.10\), \(\rho\sim\mathrm{Beta}(2,2)\), and \(\gamma\sim\mathrm{Gamma}(2,1)\). All partial-order methods fix \(\beta=0\), so frontier selection is uniform.}
\label{tab:bishops_inference_settings}
\begin{tabular}{p{0.18\linewidth}p{0.76\linewidth}}
\toprule
\textbf{Method} & \textbf{Settings} \\
\midrule
\hardmcmc{} &
\(10^{6}\) iterations; \(50\%\) burn-in; \(5000\) posterior draws after thinning; latent-dimension prior \(K\sim\mathrm{Poisson}(n/2)\) truncated to \([1,30]\); reversible-jump moves on \(K\) every 500 steps; likelihood variant \texttt{FRONTIER\_SOFTMAX\_UNI\_UTILITY}. \\

\relaxedmcmc{} (Stan NUTS) &
Assume \(d=3\); 4 chains; 2000 post-warmup draws per chain (8000 total); 1000 warmup iterations; \(\delta_{\mathrm{adapt}}=0.96\); maximum tree depth \(=11\). \\

\fullrankvi{} & Assume \(d=3\); \texttt{iter}=60{,}000; \texttt{tol\_rel\_obj}\(=5\times 10^{-4}\); 2000 output samples; three random seeds \(\{7,11,19\}\); reported result corresponds to the seed with the lowest WAIC (seed 19).\\

\flowvi{} &
\(d=3\); neural spline flow (NSF) posterior on unconstrained coordinates;
4 flow layers; 8 bins; hidden size 32; Adam optimizer with learning rate
\(0.002\); 1500 optimization steps; 4 Monte Carlo samples per gradient step
and 16 samples for ELBO evaluation; 2000 posterior draws used for summaries \\
\bottomrule
\end{tabular}
\end{table*}

\begin{figure}[t]
  \centering
  \includegraphics[width=\textwidth]{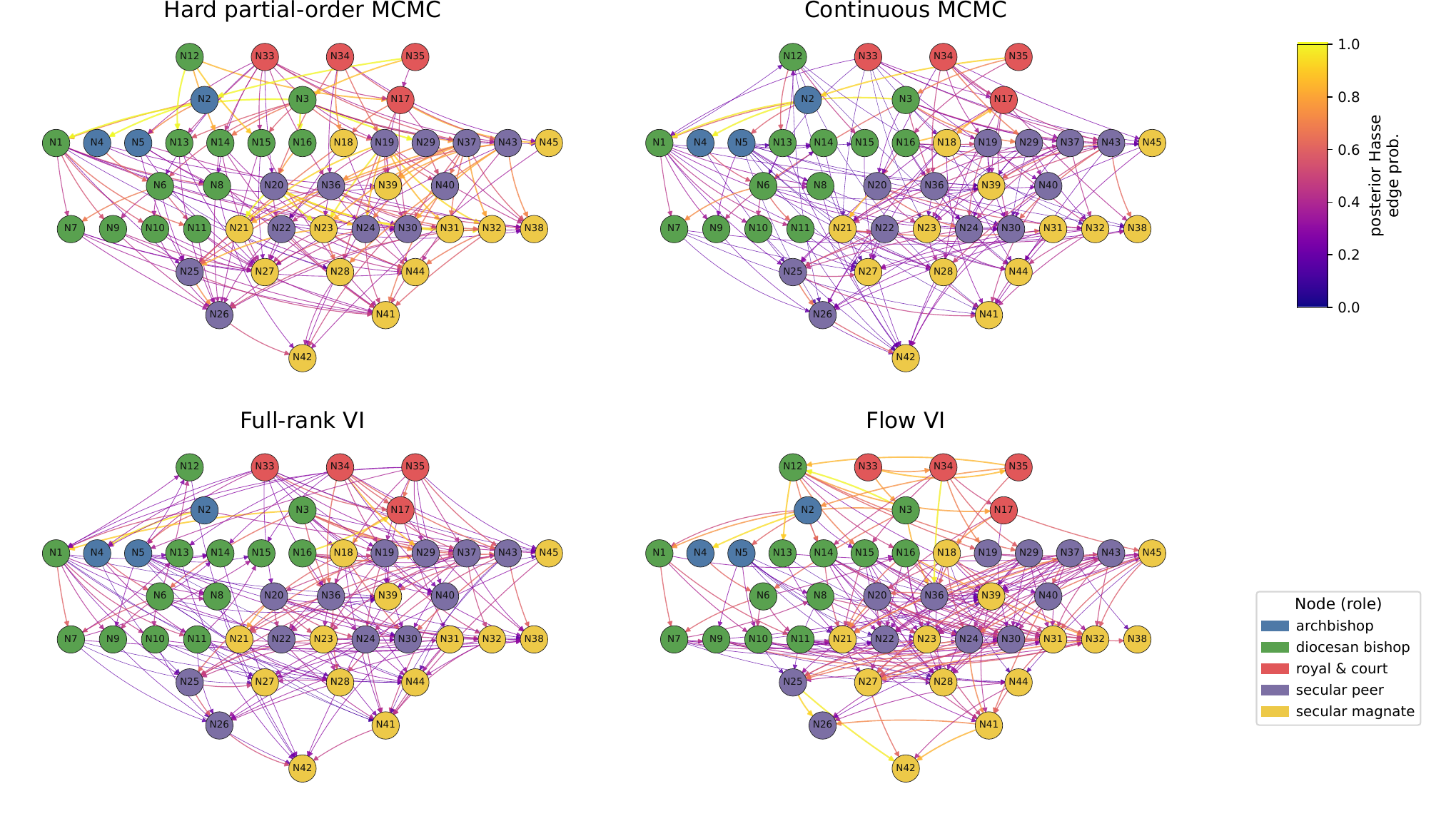}
\caption{Posterior Hasse diagrams for the Bishops corpus under the four inference methods. Edge colour and width encode posterior cover-edge probability. All four recover similar large-scale hierarchy structure, with \relaxedmcmc{} attaining the lowest WAIC.}
\label{fig:bishop_posterior_hasse}
\end{figure}

% Figure~\ref{fig:bishop_scalar_posterior} contrasts the scalar-parameter
% posteriors. \hardmcmc{} prefers noticeably larger \(\rho\) (posterior mode
% \(\sim 0.70\)) because it can also vary the latent dimension \(K\)
% , effectively moving
% part of the model's explanatory power from \(\gamma\) into higher-\(K\)
% components. Both Stan methods keep \(K=3\) fixed and consequently
% concentrate \(\rho\) at \(\sim 0.36\)--\(0.42\) and use a larger pairwise
% softness \(\gamma\) to compensate.
% \begin{figure}[t]
%   \centering
%   \includegraphics[width=0.95\textwidth]{figures/bishop_main_posterior_rho_gamma.pdf}
% \caption{Posterior densities of \(\rho\) and \(\gamma\) for the Bishops corpus. \hardmcmc{} favors larger \(\rho\), while \relaxedmcmc{} and \fullrankvi{} place posterior mass at smaller \(\rho\) and larger effective soft-precedence strength. This is consistent with the fact that \hardmcmc{} also varies the latent dimension \(K\), whereas the Stan-based methods keep \(K=3\) fixed. Vertical dotted lines indicate posterior means.}
% \label{fig:bishop_scalar_posterior}
% \end{figure}

\begin{figure}[h]
  \centering
  \includegraphics[width=\textwidth]{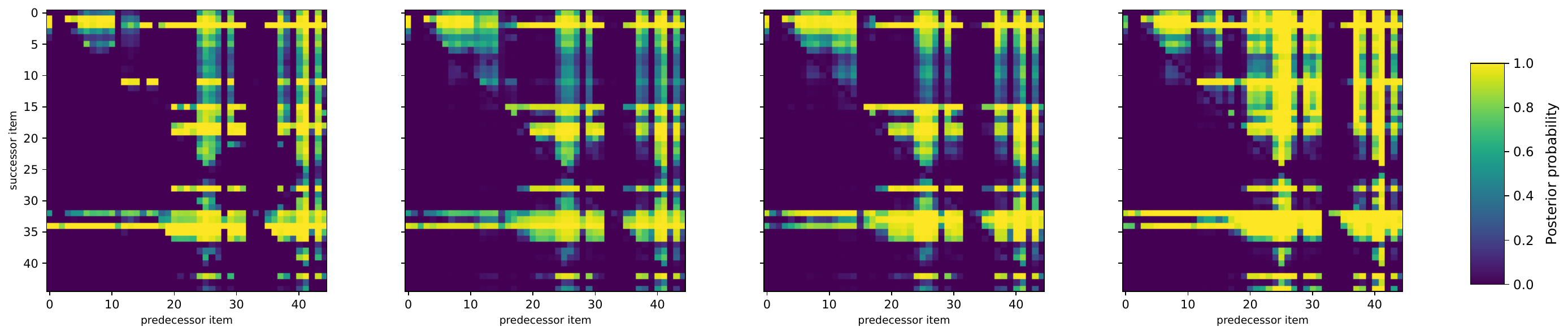}
\caption{Posterior pairwise precedence probabilities \(\mathbb{P}(i \succ j \mid \mathcal D)\) for the Bishops corpus under the four inference methods. Each heatmap shows, on a common \(0\)--\(1\) scale, the posterior probability that predecessor item \(i\) precedes successor item \(j\). The broadly similar banded patterns across panels indicate that the methods recover aligned closure-level structure.}
  \label{fig:bishop_closure_posterior}
\end{figure}

\paragraph{Plackett--Luce baselines.}
\label{app:pl_base}

As external total-order benchmarks, we fit finite-mixture Plackett--Luce models to the 68 Bishops witness lists using the Bayesian \texttt{PLMIX} sampler \citep{luce1959,plackett1975,mollica2017bayesian}. We consider \(G\in\{1,2,3\}\) mixture components, where \(G=1\) recovers the standard single-component Plackett--Luce model. We use weakly informative priors for the item support parameters and run Gibbs sampling to obtain posterior draws.

To make the predictive criterion directly comparable to the LLPO fits, we compute WAIC at the same pointwise resolution: one contribution per witness list rather than one contribution per comparable pair. For posterior draw \(s\), with item-worth vector \(w^{(s)}\), the log-likelihood contribution for witness list
\(y_i=(y_{i1},\ldots,y_{iL_i})\) is
\[
\ell_i^{(s)}
=
\sum_{t=1}^{L_i}
\left[
\log w^{(s)}_{y_{it}}
-
\log \sum_{k\in R_{it}} w^{(s)}_k
\right],
\qquad
R_{it}=\{y_{it},\ldots,y_{iL_i}\}.
\]
This is the usual Luce chain likelihood conditional on the actors appearing in that witness list. WAIC is then computed from the resulting posterior log-likelihood matrix with 68 columns \citep{watanabe2010}.

All PLMIX runs use \(15{,}000\) Gibbs iterations, \(3{,}000\) burn-in iterations, and \(2{,}000\) posterior draws for WAIC estimation. Table~\ref{tab:bishop_plmix_baselines} reports the results. The best PLMIX baseline is the single-component model, with WAIC \(=1588.55\). Adding mixture components does not improve predictive fit on this corpus: \(G=2\) and \(G=3\) have higher WAIC, partly due to larger effective complexity. The best LLPO fit, \relaxedmcmc{}, achieves WAIC \(=956.53\), a reduction of about \(632\) WAIC units relative to the best total-order PLMIX baseline.

\begin{table}[h]
\centering
\small
\caption{Plackett--Luce mixture baselines on the Bishops witness-list corpus. WAIC is computed at the witness-list level using the listwise Luce chain likelihood, matching the pointwise resolution of the LLPO likelihood. Lower WAIC is better.}
\label{tab:bishop_plmix_baselines}
\begin{tabular}{lrrrr}
\toprule
Model & WAIC & elpd & \(p_{\mathrm{WAIC}}\) & Time (s) \\
\midrule
PLMIX, \(G=1\) & 1588.55 & \(-794.28\) & 27.61  & 6.16 \\
PLMIX, \(G=2\) & 1820.51 & \(-910.26\) & 121.88 & 8.24 \\
PLMIX, \(G=3\) & 1733.55 & \(-866.77\) & 89.94  & 11.82 \\
\bottomrule
\end{tabular}
\end{table}

\subsection{Cloud agent-trace benchmark details}
\label{subsec:cloud_agent_trace}
\paragraph{Cloud-IaC-6 dataset.}
\label{app:cloud_iac_6_dataset}
We use \textbf{Cloud-IaC-6}, a cloud-provisioning benchmark from \citet{li2026delinearizingagenttracesbayesian}. The benchmark contains six Cloud Agent infrastructure-as-code scenarios, ranging from simple instance creation to high-availability cloud deployments.

The dataset contains 54 successful execution traces generated by multiple LLM agents, including Qwen-Plus, DeepSeek, and Kimi, to encourage diversity in valid execution orders. Expert-provided partial orders serve as reference structures for evaluation, with cover graphs shown in Figure~\ref{fig:aliyun-gt-covers}. The implementation and benchmark files are available in the anonymized repository.\footnote{\url{https://anonymous.4open.science/r/Cloud-IaC-6-5367/README.md}}

\paragraph{Agent Trace Recovery}
To complement the aggregate metrics, Figure~\ref{fig:aliyun_qualitative_example} shows a representative dual-zone ECS + SLB scenario at the level of the recovered Hasse diagram. In this case, \relaxedmcmc{}, \fullrankvi{}, and \flowvi{} all recover the reference structure exactly, but their posterior edge confidence differs substantially: \relaxedmcmc{} is the most confident, \fullrankvi{} is slightly less certain, and \flowvi{} is noticeably more diffuse. This highlights that exact decoded recovery does not imply the same posterior quality.

\begin{figure}[t]
  \centering
  \includegraphics[width=\textwidth]{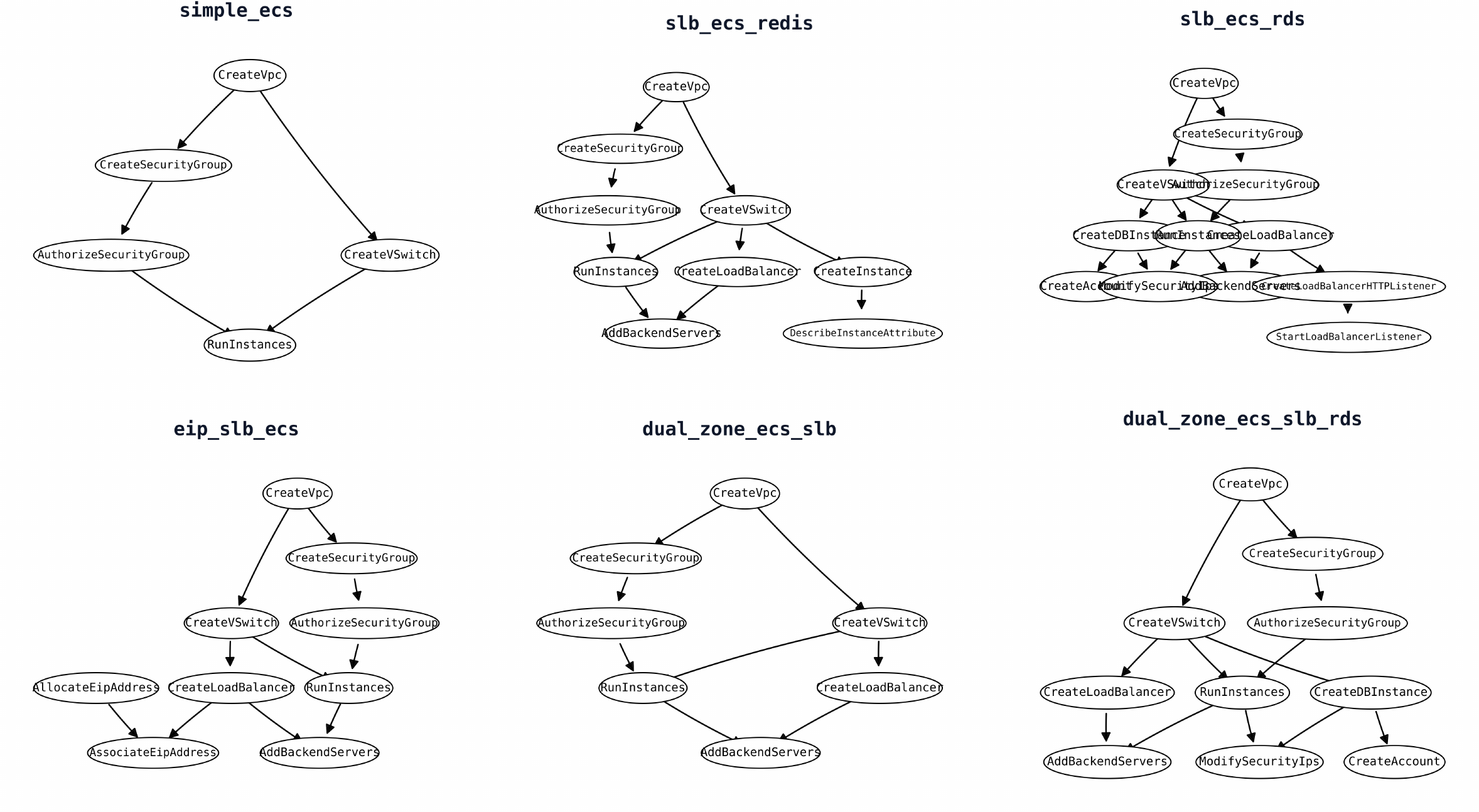}
\caption{Ground-truth action-precedence graphs (covers).
Nodes are cloud API actions; edges denote mandatory precedence constraints. from \citep{li2026delinearizingagenttracesbayesian}}
  \label{fig:aliyun-gt-covers}
\end{figure}

\begin{figure}[h]
    \centering
    \includegraphics[width=0.92\textwidth]{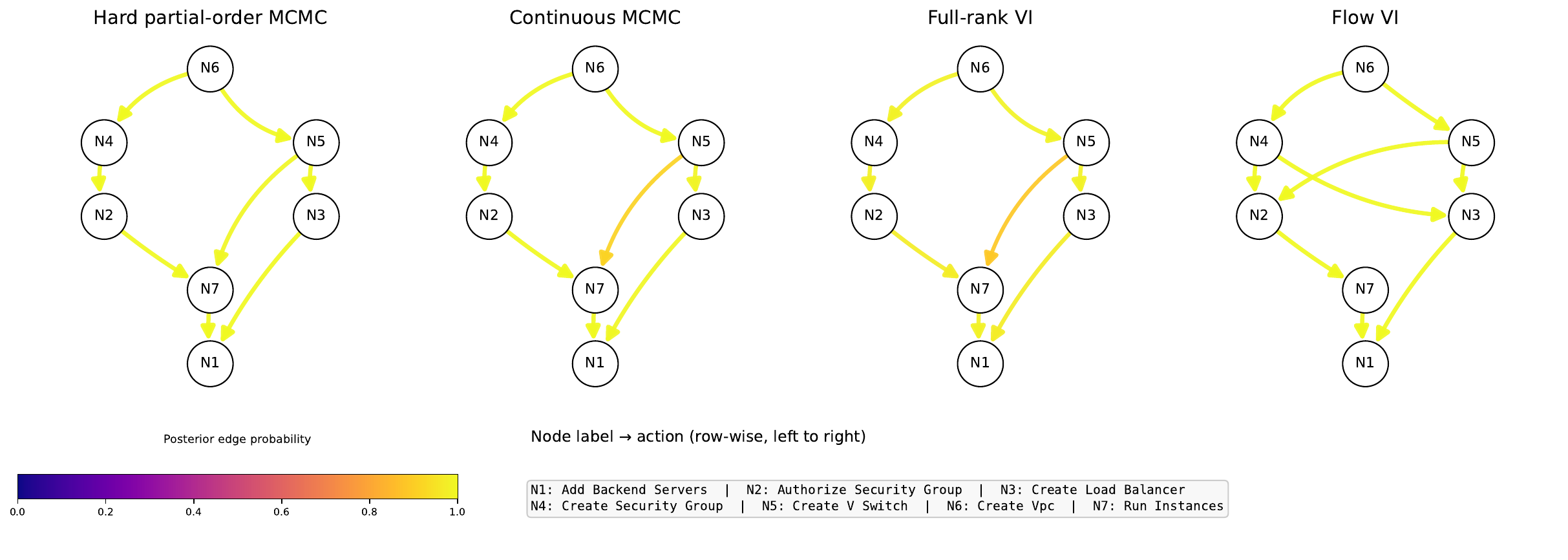}
\caption{Qualitative Cloud Agent Trace example: posterior Hasse-diagram summaries for one representative scenario dual-zone-ecs-slb. Edge color indicates posterior edge probability, and node labels correspond to cloud actions listed below the figure. The methods recover closely related skill-dependency structures, but differ in posterior confidence.}
    \label{fig:aliyun_qualitative_example}
\end{figure}

\paragraph{Convergence and optimization diagnostics}
\label{app:aliyun_prefix_anomaly}
To complement the aggregate cloud agent benchmark results, we provide representative convergence and optimization traces for one scenario, \texttt{simple\_ecs\_seed7}. Figure~\ref{fig:aliyun_convergence_traces} compares the trajectory of the objective or log-posterior quantity recorded by each inference method. These diagnostics are not intended as a formal convergence proof, but as a qualitative illustration of the different optimization and sampling behaviors underlying the runtime--accuracy trade-offs reported in the main text. In particular, they show that the four methods reach stable regimes on very different timescales and with markedly different variability.

\begin{figure*}[t]
    \centering
    \includegraphics[width=0.95\textwidth]{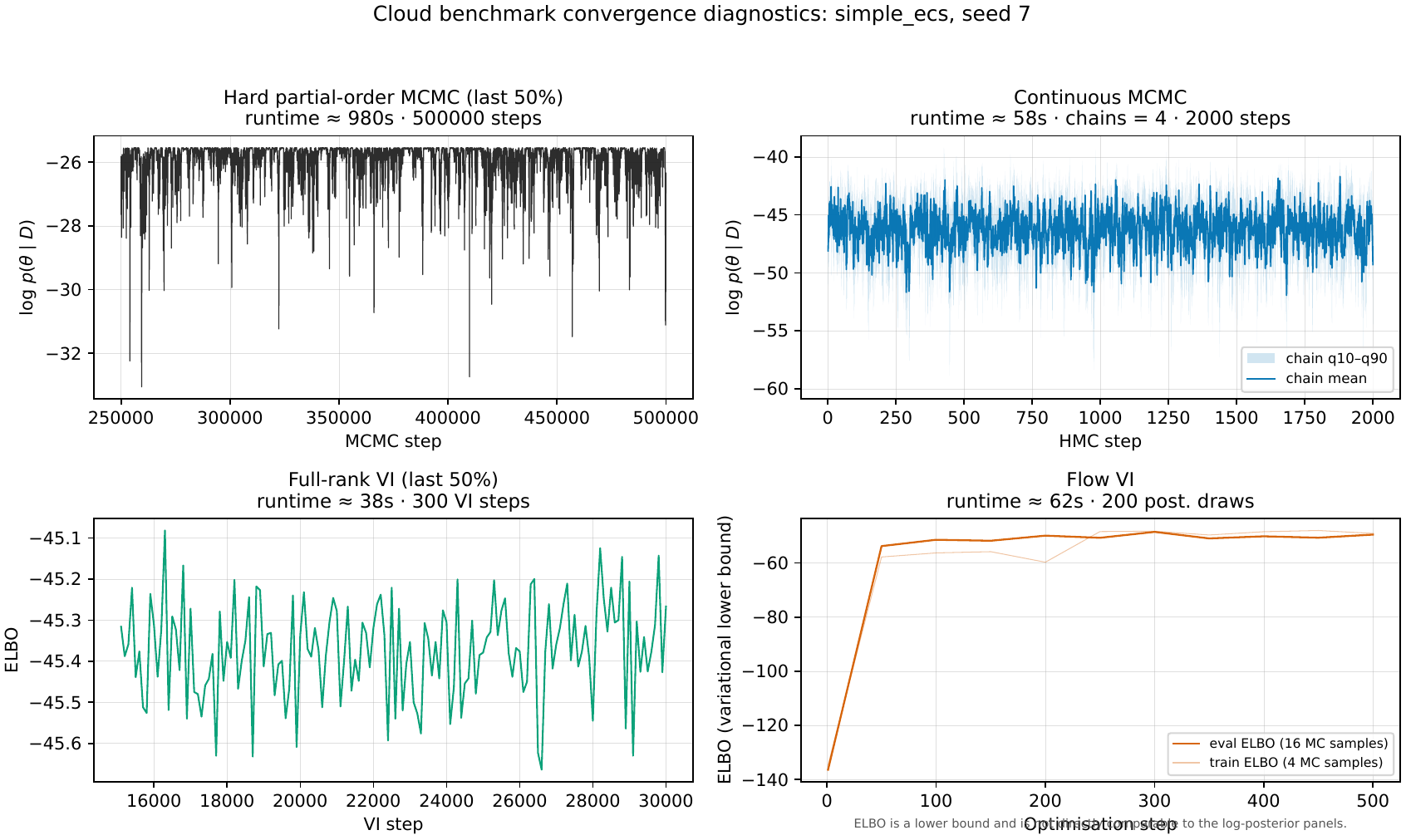}
    \caption{Representative convergence and optimization traces for the cloud agent scenario \texttt{simple\_ecs\_seed7}. The panels show the recorded objective or log-posterior trajectory for \hardmcmc{}, \relaxedmcmc{}, \fullrankvi{}, and \flowvi{}. \hardmcmc{} exhibits long-run stochastic exploration with occasional deeper excursions, \relaxedmcmc{} stabilizes around a lower objective level with persistent Monte Carlo variability, \fullrankvi{} shows a relatively stable variational objective near convergence, and \flowvi{} displays rapid early improvement followed by a slower plateau. These traces help explain the differing runtime and accuracy profiles observed in the main cloud agent experiment.}
    \label{fig:aliyun_convergence_traces}
\end{figure*}

\paragraph{Held-out next-action prediction}
\label{app:aliyun_nll}

To complement the structural-recovery results, we also evaluate predictive performance on
the Cloud benchmark through a held-out next-action prediction task. For each held-out trace,
we reveal a prefix \(y_{<t}\) and evaluate the posterior predictive probability assigned to the
next action \(y_t\), averaging over retained posterior or variational draws. We report the
corresponding next-action negative log-likelihood (Step NLL), so lower values indicate better
predictive performance.

Table~\ref{tab:aliyun-prefix-nll-mae} reports the per-scenario Step NLL and closure MAE. Under the posterior-predictive evaluator, all three relaxed methods improve next-action prediction relative to \hardmcmc{}. Averaged over the 12 Cloud cases, Step NLL is \(0.701\) for \hardmcmc{}, \(0.491\) for \relaxedmcmc{}, \(0.513\) for \fullrankvi{}, and \(0.544\) for \flowvi{}. Thus, the relaxed posterior gives better prefix-conditioned behavioral prediction, even when the hard sampler remains strongest for closure recovery.

\begin{table}[t]
\centering
\scriptsize
\setlength{\tabcolsep}{2.5pt}
\caption{Per-case next-step Step NLL and closure MAE for the 12 Cloud agent-trace cases. Step NLL is posterior-predictive over retained draws; MAE is the off-diagonal mean absolute deviation from the binary \texttt{true\_closure}.}
\label{tab:aliyun-prefix-nll-mae}
\begin{tabular}{ll|rr|rr|rr|rr}
\toprule
Scenario & Seed
& \multicolumn{2}{c|}{\hardmcmc{}}
& \multicolumn{2}{c|}{\relaxedmcmc{}}
& \multicolumn{2}{c|}{\fullrankvi{}}
& \multicolumn{2}{c}{\flowvi{}} \\
& & Step NLL & MAE & Step NLL & MAE & Step NLL & MAE & Step NLL & MAE \\
\midrule
\texttt{dual\_zone\_ecs\_slb}      &  7 & 0.292 & 0.000 & 0.302 & 0.006 & 0.309 & 0.010 & 0.252 & 0.069 \\
\texttt{dual\_zone\_ecs\_slb}      & 11 & 0.264 & 0.000 & 0.280 & 0.022 & 0.287 & 0.026 & 0.307 & 0.070 \\
\texttt{dual\_zone\_ecs\_slb\_rds} &  7 & 3.569 & 0.002 & 1.123 & 0.010 & 1.115 & 0.014 & 0.905 & 0.128 \\
\texttt{dual\_zone\_ecs\_slb\_rds} & 11 & 0.557 & 0.072 & 0.549 & 0.051 & 0.643 & 0.078 & 0.895 & 0.125 \\
\texttt{eip\_slb\_ecs}             &  7 & 0.568 & 0.001 & 0.571 & 0.025 & 0.578 & 0.030 & 0.724 & 0.144 \\
\texttt{eip\_slb\_ecs}             & 11 & 0.574 & 0.002 & 0.598 & 0.026 & 0.603 & 0.030 & 0.641 & 0.098 \\
\texttt{simple\_ecs}               &  7 & 0.173 & 0.000 & 0.189 & 0.005 & 0.195 & 0.009 & 0.302 & 0.048 \\
\texttt{simple\_ecs}               & 11 & 0.173 & 0.000 & 0.173 & 0.020 & 0.182 & 0.028 & 0.223 & 0.051 \\
\texttt{slb\_ecs\_rds}             &  7 & 0.666 & 0.003 & 0.677 & 0.018 & 0.688 & 0.033 & 0.625 & 0.063 \\
\texttt{slb\_ecs\_rds}             & 11 & 0.504 & 0.003 & 0.509 & 0.013 & 0.611 & 0.132 & 0.678 & 0.155 \\
\texttt{slb\_ecs\_redis}           &  7 & 0.531 & 0.029 & 0.434 & 0.035 & 0.445 & 0.046 & 0.469 & 0.129 \\
\texttt{slb\_ecs\_redis}           & 11 & 0.541 & 0.029 & 0.486 & 0.024 & 0.498 & 0.027 & 0.508 & 0.067 \\
\midrule
\textbf{MEAN} & -- & 0.701 & 0.012 & 0.491 & 0.021 & 0.513 & 0.039 & 0.544 & 0.096 \\
\bottomrule
\end{tabular}
\end{table}

\paragraph{Why closure fidelity and predictive NLL can disagree}
\label{sec:aliyun-metrics-calibration}

The Cloud benchmark reports both \emph{structural} metrics (closure F1 and MAE to the reference closure) and \emph{predictive} metrics (Trace NLL, Step NLL). These metrics need not rank methods in the same order. Closure metrics evaluate recovery of the latent prerequisite structure, whereas Step NLL evaluates posterior predictive next-action probabilities from observed prefixes under the same frontier-softmax evaluator.

The difference arises from the multiplicative geometry of frontier-based prediction. Closure metrics average pairwise errors linearly, so a small number of incorrectly oriented edges may have limited effect on the overall structural score. By contrast, Step NLL is a log score: in the frontier-softmax model, an overconfident wrong edge can strongly suppress the probability of the correct next action through the frontier weight, producing a large log penalty. Thus sharper posterior structure can improve closure recovery while hurting predictive calibration, whereas smoother posterior uncertainty can worsen decoded closure but improve next-action prediction.

For this reason, we report both families of metrics. When the goal is latent graph recovery or workflow governance, closure fidelity is the primary criterion. When the goal is sequential prediction or agent-trace continuation, Step NLL and WAIC are more informative.

\paragraph{Scenario-wise scalar posterior densities on Cloud-IaC-6}
\label{app:aliyun_scalar_posteriors}

To complement the aggregate Cloud-IaC-6 metrics in the main text, Figure~\ref{fig:aliyun_scalar_posteriors_grid} shows scenario-wise posterior densities for the scalar parameters \(\rho\), \(\beta\), and \(\gamma\) across the six benchmark tasks. These plots make clear that the relative behavior of the inference methods is not uniform across scenarios. In particular, the \hardmcmc{}, \relaxedmcmc{}, \fullrankvi{}, and \flowvi{} can place posterior mass in noticeably different regions of parameter space even when they recover similar closure structure.

This comparison is useful for two reasons. First, it shows how the inferred embedding correlation \(\rho\), frontier-softmax sharpness \(\beta\), and soft-precedence steepness \(\gamma\) vary with workflow complexity. Second, it helps explain why exact decoded recovery alone does not fully characterize posterior quality: methods that recover the same partial order may still differ substantially in posterior concentration and uncertainty.

\begin{figure}[!ht]
    \centering
    \includegraphics[width=\textwidth]{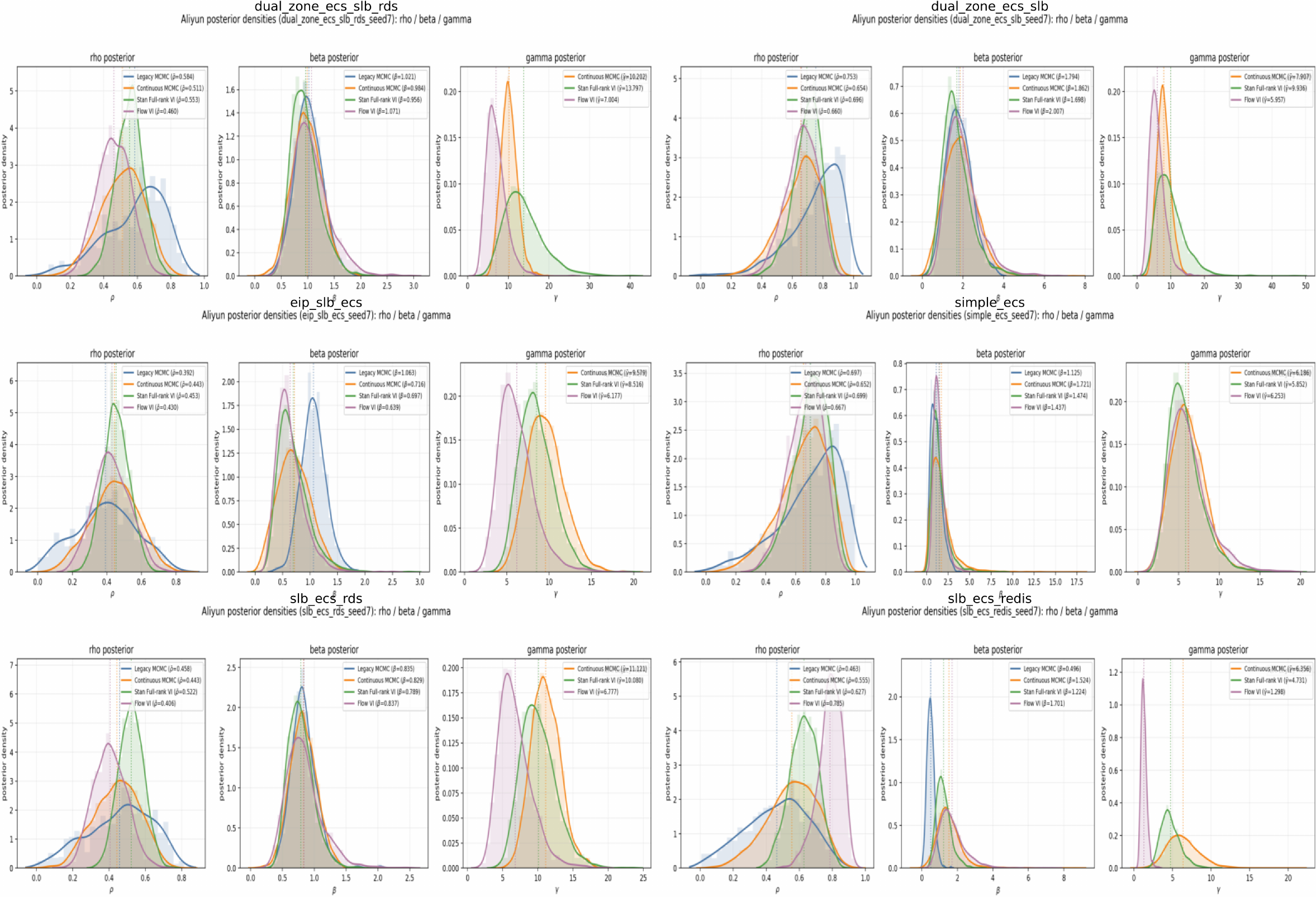}
    \caption{Scenario-wise posterior densities of the scalar parameters \(\rho\), \(\beta\), and \(\gamma\) on the six Cloud-IaC-6 tasks. Each row corresponds to a pair of benchmark scenarios, and each triplet of panels shows posterior densities for one scenario under the compared inference methods. The figure highlights substantial between-method differences in posterior concentration, even in cases where the recovered closure structure is similar.}
    \label{fig:aliyun_scalar_posteriors_grid}
\end{figure}
\section{Reproducibility}
\label{app:repro}

This appendix summarizes the code paths, datasets, inference settings, evaluation rules, and compute budget used to reproduce the main results. Exact run-level settings are stored in the released artefacts, including \texttt{run\_summary.json}, \texttt{fit\_summary.json}, \texttt{flow\_config.json}, posterior draw files, and the suite/case JSON files used by \texttt{experiments/scripts/run\_experiment\_suite.py}. When a suite-level configuration differs from an executed run summary, the run-level summary is authoritative.

\subsection{Codebase and environment}
\label{app:repro_env}

The anonymized repository contains the full experimental pipeline. The main entry point is
\[
\texttt{experiments/scripts/run\_experiment\_suite.py},
\]
which dispatches \texttt{legacy\_mcmc} (\hardmcmc{}), \texttt{relaxed\_mcmc} (\relaxedmcmc{} / Stan NUTS), \texttt{fullrank} (\fullrankvi{} / Stan full-rank ADVI), and \texttt{flow} (\flowvi{}). Suite configurations are stored in \texttt{experiments/configs/}. Stan models are stored in \texttt{stan/}, Flow-VI is implemented in \texttt{bpop\_vi\_frontier\_experiment/normalizing\_flow\_vi.py}, and the shared likelihood, decoding, and evaluation utilities are in \texttt{src/utils/}.

All posterior-predictive metrics in the paper are computed by the evaluator
\[
\texttt{src/utils/posterior\_predictive\_eval.py},
\]
which evaluates Trace NLL, Step NLL, and WAIC over retained posterior or variational draws rather than posterior-mean plug-in parameters.

The environment is Python 3.9 in the Conda environment specified by \texttt{environment.yml}. We use \texttt{numpy}, \texttt{pandas}, \texttt{scipy}, \texttt{matplotlib}, and \texttt{networkx}; \texttt{cmdstanpy} with CmdStan 2.34 for Stan-based methods; and \texttt{torch} for Flow-VI.

A typical reproduction command is:
\begin{verbatim}
python experiments/scripts/run_experiment_suite.py \
    --config experiments/configs/<suite_name>.json
\end{verbatim}

\subsection{Datasets and fixed splits}
\label{app:repro_data}

\paragraph{Synthetic data.}
Synthetic datasets are generated by
\[
\texttt{experiments/scripts/generate\_synthetic\_minimal\_ipcov.py}.
\]
We use \(n\in\{5,10,20,30,50,100\}\), \(\rho\in\{0.5,0.9\}\), seeds \(\{7,11,19\}\), latent dimension \(d=4\), generation \(\beta=1.0\). The ground-truth closure is the product order induced by Gaussian embeddings with covariance
\[
\Sigma_\rho=(1-\rho)I_d+\rho\mathbf 1\mathbf 1^\top.
\]
Training traces are selected to maximize incomparable-pair coverage under a budget of \(2n\) traces, with at least \(n\) traces retained. Test traces are independent hard frontier-softmax draws from the same ground-truth partial order. Each case stores its realized train/test traces and IP-Cov in \texttt{dataset.json}. The low-coverage ablation additionally uses \(n=30\), \(\rho=0.5\), seeds \(\{7,11,19\}\), and target IP-Cov \(=0.7\).

\paragraph{External data sources.}
The Bishops witness-list corpus is derived from the public
\emph{Charters of William II and Henry I} project website:
\url{https://actswilliam2henry1.wordpress.com/the-charters/}.
The project page states that the charter materials are released under a
Creative Commons Attribution--NonCommercial--NoDerivs 4.0 license. We use the
processed witness-list data described in \citet{jiang2023bayesian} and
\citet{nicholls2025royalacta}.

The Cloud agent-trace benchmark is adapted from the anonymized Cloud-IaC-6
repository:
\url{https://anonymous.4open.science/r/Cloud-IaC-6-5367/README.md}.
The benchmark consists of six infrastructure-as-code scenarios with fixed
train/test seeds; our scripts use the processed traces and ground-truth covers
specified in Appendix~\ref{app:repro_profiles}.

\paragraph{Bishops witness lists.}
The Bishops corpus is stored as
\[
\texttt{data/bishop/bishop\_po\_cla\_o\_lists\_only.json},
\]
containing 68 ordered witness lists over 45 actors. Actor names and title groups are stored in the accompanying mapping files
\[
\texttt{bishop\_o\_to\_node\_name\_mapping.csv}
\quad\text{and}\quad
\texttt{bishop\_item\_hasse\_color\_group.csv}.
\]
All 68 witness lists are used for posterior fitting and list-level predictive summaries.

\paragraph{Cloud agent traces.}
The Cloud-IaC-6 benchmark contains six cloud-provisioning scenarios crossed with seeds \(\{7,11\}\), giving 12 cases. Raw traces, expert traces, and expert cover graphs are stored under \texttt{data/}. The exact train/test order indices used  for Tables~\ref{tab:aliyun_metrics_updated} and~\ref{tab:aliyun-prefix-nll-mae} 
are stored in the
experiment configuration files. All methods are fit and evaluated on identical

\subsection{Inference settings}
\label{app:repro_profiles}

\paragraph{\hardmcmc{}.}
\hardmcmc{} is the legacy discrete MCMC sampler over hard partial-order states. On synthetic data, it uses \(2\times10^5\) to \(1.5\times10^6\) iterations depending on \(n\), with 50\% burn-in and fixed latent dimension \(K=4\). On Cloud, it uses \(5\times10^5\) iterations, 50\% burn-in, and fixed \(K=3\). On Bishops, it uses \(10^6\) iterations, 50\% burn-in, 5000 retained posterior draws after thinning, and reversible-jump moves over \(K\), with \(K\sim\mathrm{Poisson}(N//2)\).

\paragraph{\relaxedmcmc{}.}
\relaxedmcmc{} samples the relaxed posterior using Stan NUTS on unconstrained coordinates. Synthetic runs use one chain with 500 warmup and 500 sampling iterations by default, with slightly longer runs for larger \(n\). Cloud uses 4 chains, 1500 warmup iterations, 2000 post-warmup draws per chain, \texttt{adapt\_delta}=0.98, and maximum tree depth 13. Bishops uses 4 chains, 1000 warmup iterations, 2000 post-warmup draws per chain, \texttt{adapt\_delta}=0.96, and maximum tree depth 11.

\paragraph{\fullrankvi{}.}
\fullrankvi{} uses Stan full-rank ADVI on the same relaxed posterior target as \relaxedmcmc{}. Synthetic runs use 500 posterior draws and 8000--12000 optimization iterations depending on \(n\). Cloud uses 2000 posterior draws and 30000 optimization iterations. Bishops uses 2000 posterior draws and 60000 optimization iterations. On Bishops, we run seeds \(\{7,11,19\}\).

\paragraph{\textsc{Flow-VI}.}
\textsc{Flow-VI} parameterises the variational posterior with a neural spline
flow on the unconstrained coordinates and shares one optimiser configuration
across all benchmarks: $4$ NSF coupling layers, learning rate $0.002$ with a
$25$-step warm-up decaying linearly to $0.2\times$ over $1500$ steps,
Adam ($\beta_1=0.9$, $\beta_2=0.999$, gradient-norm clip $10$), $4$ antithetic
Monte~Carlo samples per ELBO gradient, $16$ samples for ELBO evaluation,
early-stopping patience $200$, and $200$ retained posterior draws.
Synthetic runs use $d=4$, $\tau\!\in\!\{0.1,\,0.3,\,0.5\}$, $16$ spline bins
and hidden dimension $64$.
The cloud benchmark uses $d=4$, $\tau=0.5$ with a narrower conditioner
($8$ bins, hidden dimension $32$).
Bishops uses $d=3$, $\tau=0.10$, $8$ bins and hidden dimension $32$, with the
noise scale $\beta$ held fixed rather than inferred.
Exact settings are stored in each case's \texttt{flow\_config.json}.

\paragraph{Baselines.}
For Bishops, we fit PLMIX baselines with \(G\in\{1,2,3\}\) mixture components using 15000 Gibbs iterations, 3000 burn-in iterations, and 2000 retained posterior draws for list-level WAIC. For the SoftDAG-Frontier diagnostic baseline, we use three Adam restarts, 400 optimization steps, learning rate \(0.05\), \(K=\min(8,n-1)\) for soft reachability, and validation-based selection over \(\lambda_1\in\{10^{-4},10^{-3},10^{-2}\}\) and \(\lambda_h\in\{1,10,100\}\).

\subsection{Decoding and posterior-predictive evaluation}
\label{app:repro_eval}

\paragraph{Posterior decoding.}
All structural metrics are computed from posterior pairwise closure probabilities. For method \(a\), let \(\widehat P^{(a)}_{ij}\) be the posterior mean probability that \(i\succ j\). We decode a closure by thresholding \(\widehat P^{(a)}\) at \(\tau_{\mathrm{decode}}=0.5\), setting diagonal entries to zero, removing any threshold-induced cycles, and taking the transitive closure. We use \(\tau_{\mathrm{decode}}=1/3\) for n=100.

\paragraph{Posterior-predictive metrics.}
Trace NLL, Step NLL, and WAIC are computed from each method's retained posterior or variational draws. For \hardmcmc{}, the per-draw predictive distribution is the hard frontier-softmax evaluated on the retained binary closure draw. For \relaxedmcmc{}, \fullrankvi{}, and \flowvi{}, it is the relaxed frontier-softmax evaluated on each retained soft-precedence draw. Thus the reported predictive metrics are posterior-predictive quantities, not plug-in evaluations at posterior mean parameters.

For held-out traces \(\mathcal D_{\mathrm{test}}\),
\[
\mathrm{Trace\mbox{-}NLL}
=
-\frac{1}{|\mathcal D_{\mathrm{test}}|}
\sum_{y\in\mathcal D_{\mathrm{test}}}
\log
\left[
\frac{1}{S}\sum_{s=1}^{S}
p(y\mid\Theta^{(s)})
\right].
\]
For next-action prediction,
\[
\mathrm{Step\mbox{-}NLL}
=
-\frac{1}{\sum_{y\in\mathcal D_{\mathrm{test}}}T_y}
\sum_{y\in\mathcal D_{\mathrm{test}}}
\sum_{t=1}^{T_y}
\log
\left[
\frac{1}{S}\sum_{s=1}^{S}
p(y_t\mid y_{<t},\Theta^{(s)})
\right].
\]
Because \(\log\mathbb E_s[p]\neq \mathbb E_s[\log p]\), Trace NLL is not generally equal to mean trace length times Step NLL. For WAIC, let
\[
\ell_{is}=\log p(y^{(i)}\mid\Theta^{(s)})
\]
be the pointwise training log-likelihood. We compute
\[
\mathrm{lppd}
=
\sum_i
\log\left[
\frac{1}{S}\sum_{s=1}^{S}\exp(\ell_{is})
\right],
\qquad
p_{\mathrm{WAIC}}
=
\sum_i \mathrm{Var}_s(\ell_{is}),
\]
and
\[
\mathrm{WAIC}
=
-2(\mathrm{lppd}-p_{\mathrm{WAIC}}).
\]
Large \(p_{\mathrm{WAIC}}\) indicates high per-draw log-likelihood variability and is interpreted as a warning that the posterior approximation is unstable for WAIC.

\paragraph{Posterior fidelity.}
On synthetic and Bishops settings where \hardmcmc{} is available as a reference, posterior fidelity is measured by
\[
\mathrm{MAE}_{\mathrm{hard}}(a)
=
\frac{1}{n(n-1)}
\sum_{i\neq j}
\left|
\widehat P^{(a)}_{ij}
-
\widehat P^{(\mathrm{hard})}_{ij}
\right|.
\]
For \hardmcmc{}, this value is zero by definition.

\subsection{Compute budget}
\label{app:repro_compute}

Experiments were run on a CPU workstation; Flow-VI used a single GPU for
normalizing-flow optimization. Exact hardware varied across runs, but all
reported runtimes are wall-clock times measured on the executing machine. The dominant cost is still
\hardmcmc{} on larger synthetic instances; for example, \(n=50,\rho=0.9\) requires
roughly 30 hours for one seed. The continuous methods are substantially cheaper in
wall-clock time on most settings. Representative costs are:

\begin{center}
\small
\begin{tabular}{lrrr}
\toprule
Setting & \hardmcmc{} & \relaxedmcmc{} & \flowvi{} \\
\midrule
Synthetic \(n=10,\rho=0.5\) & \(\sim 10\) min & \(\sim 6\) s & \(\sim 22\) s \\
Synthetic \(n=30,\rho=0.9\) & \(\sim 3\) h & \(\sim 33\) min & \(\sim 46\) s \\
Synthetic \(n=50,\rho=0.9\) & \(\sim 30\) h & one completed seed & \(\sim 46\) s \\
Bishops, per seed & \(\sim 70\) min & \(\sim 27\) min & \(\sim 9\) min \\
Cloud, mean over cases & \(\sim 24\) min & \(\sim 5\) min & \(\sim 1.7\) min \\
\bottomrule
\end{tabular}
\end{center}

The reported experiments, ablations, and diagnostic baselines used approximately
\(\le 10\) CPU-days, plus GPU time for the normalizing-flow VI runs.

\subsection{Known caveats}
\label{app:repro_caveats}

\paragraph{Stan version sensitivity.}
Stan NUTS and full-rank ADVI were run through CmdStan 2.34. Minor CmdStan version changes can alter exact seed behavior and ELBO trajectories, although the reported closure F1 and posterior-fidelity metrics were stable in spot checks.

\paragraph{Flow-VI posterior variability.}
Flow-VI uses fewer retained draws than the MCMC methods and can produce more diffuse posterior samples.
\paragraph{Assets and licenses.}
\label{app:assets_licenses}
All datasets and software assets are cited in the paper. The Bishops / Royal Acta data are drawn from the published academic database associated with \emph{The Charters of William II and Henry I}. The Cloud-IaC-6 benchmark, synthetic generators, code, and experiment configurations are released through the anonymized repository with the license stated there. Software dependencies follow their respective open-source licenses.

\paragraph{Human subjects.}
The paper does not involve crowdsourcing or prospective human-subject experiments. The empirical datasets are synthetic, historical witness lists, and cloud-agent execution traces.
\clearpage

\section{Broader impacts.}
This work may improve workflow auditing and uncertainty-aware modeling of dependency structure. Risks arise if inferred partial orders are treated as definitive causal or social hierarchies. We therefore emphasize posterior uncertainty, domain validation, and use as decision support rather than deterministic evidence.

\end{document}